\documentclass[11pt]{article}
\usepackage[top=1in, bottom=1in, left=1in, right=1in]{geometry}

\usepackage{amsmath,amssymb,amsbsy,amsfonts,amscd,bm,url,color,latexsym}
\usepackage[hidelinks]{hyperref}
\usepackage{paralist}
\usepackage[dvipsnames]{xcolor}
\usepackage{color}
\usepackage{graphicx}
\graphicspath{{./figs/}}
\usepackage{algorithm}
\usepackage{algorithmic}
\usepackage{comment}
\usepackage{multirow}
\usepackage{enumitem}
\usepackage{fancyhdr}
\usepackage{cite}
\usepackage{cleveref}
\usepackage{subcaption}
\usepackage{authblk}
\newtheorem{thm}{Theorem}%[section] %(If you want theorem numbered
\newcommand{\R}{\mathbb{R}}

\newcommand{\e}{\begin{equation}}
\newcommand{\ee}{\end{equation}}
\newcommand{\en}{\begin{equation*}}
\newcommand{\een}{\end{equation*}}
\newcommand{\eqn}{\begin{eqnarray}}
\newcommand{\eeqn}{\end{eqnarray}}
\newcommand{\bmat}{\begin{bmatrix}}
\newcommand{\emat}{\end{bmatrix}}
\DeclareMathAlphabet\mathbfcal{OMS}{cmsy}{b}{n}
\newcommand{\mb}{\mathbf}
\newcommand{\mc}{\mathcal}

\newcommand{\bb}{\mathbb}

\newcommand{\vct}[1]{\boldsymbol{#1}}
\newcommand{\mtx}[1]{\boldsymbol{#1}}
\newcommand{\trace}{\operatorname{trace}}

\newcommand{\wt}{\widetilde}
\newcommand{\ol}{\overline}
\newcommand{\NC}{$\mc {NC}$}

\newcommand{\norm}[2]{\left\| #1 \right\|_{#2}}

\newcommand{\abs}[1]{\left| #1 \right|}
\newcommand{\bracket}[1]{\left( #1 \right)}
\newcommand{\parans}[1]{\left(#1\right)}
\newcommand{\innerprod}[2]{\left\langle #1,  #2 \right\rangle}

\newcommand*\samethanks[1][\value{footnote}]{\footnotemark[#1]}

\newcommand{\calC}{\mathcal{C}}

\newcommand{\calL}{\mathcal{L}}

\newcommand{\va}{\vct{a}}
\newcommand{\vb}{\vct{b}}

\newcommand{\vh}{\vct{h}}

\newcommand{\vu}{\vct{u}}
\newcommand{\vv}{\vct{v}}
\newcommand{\vw}{\vct{w}}
\newcommand{\vx}{\vct{x}}
\newcommand{\vy}{\vct{y}}
\newcommand{\vz}{\vct{z}}
\newcommand{\valpha}{\vct{\alpha}}

\newcommand{\vtheta}{\vct{\theta}}

\newcommand{\vzero}{\vct{0}}
\newcommand{\vone}{\vct{1}}

\newcommand{\mA}{\mtx{A}}
\newcommand{\mB}{\mtx{B}}

\newcommand{\mH}{\mtx{H}}

\newcommand{\mP}{\mtx{P}}

\newcommand{\mU}{\mtx{U}}
\newcommand{\mV}{\mtx{V}}
\newcommand{\mW}{\mtx{W}}

\newcommand{\mZ}{\mtx{Z}}

\newcommand{\mDelta}{\mtx{\Delta}}

\newcommand{\mSigma}{\mtx{\Sigma}}

\graphicspath{{./figs/}}

\newlength{\imgwidth}
\newboolean{twoColVersion}
\setboolean{twoColVersion}{false}
\newcommand{\twoCol}[2]{\ifthenelse{\boolean{twoColVersion}} {#1} {#2} }

\usepackage{tikz}
\usepackage{tikz-3dplot}
\usepackage{pgfplots}
\usepgfplotslibrary{patchplots}
\usetikzlibrary{patterns, positioning, arrows}
\pgfplotsset{compat=1.15}

\usepackage{tgpagella}
\usepackage[utf8]{inputenc} % allow utf-8 input
\usepackage[T1]{fontenc}    % use 8-bit T1 fonts
\usepackage{url}
\hypersetup{
    colorlinks=true,%
    citecolor=blue,%
    filecolor=blue,%
    linkcolor=blue,%
    urlcolor=blue
}
\usepackage[toc, page]{appendix}
\newtheorem{theorem}{Theorem}[section]
\newtheorem{lemma}[theorem]{Lemma}
\newtheorem{corollary}[theorem]{Corollary}

\newtheorem{definition}[theorem]{Definition}

\renewcommand{\mathbf}{\boldsymbol}

\newcommand{ \Brac }[1]{\left\lbrace #1 \right\rbrace}
\newcommand{ \brac }[1]{\left[ #1 \right]}
\newcommand{ \paren }[1]{ \left( #1 \right) }

\newcommand{\Lce}{\mc L_{\mathrm{CE}}}

\newcommand{\zz}[1]{{\color{blue}{\bf Zhihui: #1}}}

\newcommand{\cy}[1]{{\color{blue}{\bf Chong: #1}}}

\def \endprf{\hfill {\vrule height6pt width6pt depth0pt}\medskip}
\newenvironment{proof}{\noindent {\bf Proof} }{\endprf\par}

\newcommand{\lambdaW}{\lambda_{\mW}}
\newcommand{\lambdaH}{\lambda_{\mH}}
\newcommand{\lambdab}{\lambda_{\vb}}

\title{A Geometric Analysis of Neural Collapse\\ with Unconstrained Features}

\begin{document}

% make the title area
\author[$\sharp$]{Zhihui Zhu\thanks{The first two authors contributed to this work equally.}$^,$}
\author[$\Diamond$]{Tianyu Ding\samethanks$^,$}
\author[$\sharp$,$\dagger$]{Jinxin Zhou}
\author[$\dagger$,$\ddagger$]{Xiao Li}
\author[$\clubsuit$]{Chong You}
\author[$\spadesuit$]{\\Jeremias Sulam}
\author[$\dagger$]{Qing Qu}

\affil[$\sharp$]{Department of Electrical \& Computer Engineering, University of Denver}
\affil[$\Diamond$]{Department of Applied Mathematics \& Statistics, Johns Hopkins University}
\affil[$\ddagger$]{Center for Data Science, New York University}
\affil[$\clubsuit$]{Google Research, New York City}
\affil[$\spadesuit$]{Department of Biomedical Engineering \& MINDS, Johns Hopkins University}
\affil[$\dagger$]{Department of Electrical Engineering \& Computer Science, University of Michigan}
%\date{\today}
\maketitle
%\IEEEpeerreviewmaketitle

\begin{abstract}
{\normalsize

We provide the first global optimization landscape analysis of \emph{Neural Collapse}
-- an intriguing empirical phenomenon that arises in the last-layer classifiers and features of neural networks during the terminal phase of training. %\js{ i just modified the previous sentence a little. See if OK.}
%\cy{a bit weird to have ``landscape analysis'' of a ``phenomenon''? Maybe ``analysis of neural networks for understanding Neural Collapse''}
As recently reported in \cite{papyan2020prevalence}, this phenomenon implies that \emph{(i)} the class means and the last-layer classifiers all collapse to the vertices of a Simplex Equiangular Tight Frame (ETF) up to scaling, and \emph{(ii)} cross-example within-class variability of last-layer activations collapses to zero. We study the problem based on a simplified \emph{unconstrained feature model}, which isolates the topmost layers from the classifier of the neural network. In this context, we show that the classical cross-entropy loss with weight decay has a benign global landscape, in the sense that the only global minimizers are the Simplex ETFs while all other critical points are strict saddles whose Hessian exhibit negative curvature directions. 
In contrast to existing landscape analysis for deep neural networks which is often disconnected from practice, 
%cy{``mostly focuses on the optimization perspective'' may not be precise, maybe simply remove it. }
our analysis of the simplified model not only does it explain what kind of features are learned in the last layer, but it also shows why they can be efficiently optimized in the simplified settings, matching the empirical observations in practical deep network architectures. These findings could have profound implications for optimization, generalization, and robustness of broad interests. For example, our experiments demonstrate that one may set the feature dimension equal to the number of classes and fix the last-layer classifier to be a Simplex ETF for network training, which reduces memory cost by over 20\% on ResNet18 without sacrificing the generalization performance.%\xiao{around $19\%$?}
}
\end{abstract}

%\keywords{deep neural network, layer-peeled model, neural collapse, nonconvex optimization, global optimality, strict saddle function, second-order geometry.}

%\newpage 
%\js{Why the new page?}
\section{Introduction}\label{sec:intro}

%first paragraph: motivations for studying Neural Collapse in deep neural network

In the past decade, the revival of deep neural networks has led to dramatic success in numerous applications ranging from computer vision, to natural language processing, to scientific discovery and beyond \cite{krizhevsky2012imagenet,lecun2015deep,goodfellow2016deep,senior2020improved}. Nevertheless, the practice of deep networks has been shrouded with mystery as our theoretical understanding for the success of deep learning remains elusive. There are many intriguing phenomena, such as implicit algorithmic bias in training \cite{neyshabur2014search,soudry2018implicit,arora2019implicit,moroshko2020implicit,razin2020implicit}, and good generalization of highly-overparameterized networks \cite{zhang2016understanding,soudry2018implicit,belkin2019reconciling,mei2019generalization,nakkiran2019deep,yang2020rethinking}, that  %\js{maybe "seem", instead of "are" is better}  
seem often contradictory to or cannot be explained by classical optimization and learning theory.

\begin{figure}[t]
% 	\centering
%     \includegraphics[width = 0.7\textwidth]{figs/neural_collapse.pdf}
    \begin{center}
    \tdplotsetmaincoords{60}{110}
    \begin{tikzpicture}[scale=2.4]
      \coordinate (o1) at (0,0);
      \coordinate (o2) at (5,1.1);
      
      % figure transitions
      \draw[very thick,->] (3.2, 1.0) -- (3.8,1.0);
      \draw (3.5,1.0) node[anchor=north]{$\phi_{\bm \theta}(\cdot)$};
      
      % left figure
      \tdplotsetrotatedcoords{100}{0}{0}
      \tdplotsetrotatedcoordsorigin{(o1)}
      \begin{scope}[tdplot_rotated_coords]
    
      % data manifold
      \draw[thick] (0.6,0.5,0.5) .. controls (1.3,1.4,0.55) and (1.8,0.1,0.55) .. (2.7,0.5,0.6);
      \draw[thick] (2.7,0.5,0.6) .. controls (2.6,1.5,0.55) .. (2.5,2.7,0.5);
      \draw[thick] (2.5,2.7,0.5) .. controls (1.8,1.9,0.65) and (1.2,3.4,0.65) .. (0.5,2.3,0.8);
      \draw[thick] (0.5,2.3,0.8) .. controls (0.55,1.5,0.65) .. (0.6,0.5,0.5);
      \draw (2.7,0.5,0.6) node[anchor=south east]{\small $\mathcal{M}$};
      
      % data subspaces
      \draw[red] (0.85,0.85,0.75) .. controls (1.6,1.5,0.8) .. (2.6,1.5,0.8);
      \draw[black!30!green] (0.9,1.8,0.75) .. controls (1.6,1.5,0.8) .. (2.4,0.8,0.8);
      \draw[black!10!blue] (1.6,0.6,0.75) .. controls (1.7,1.5,0.8) .. (1.7,2.1,0.8);
      
      % data points on subspaces
      \def\points{(0.95,0.95,0.75), (1.05,1.05,0.75), (1.4,1.35,0.75), (1.75,1.53,0.75), (2.0,1.56,0.75), (2.3,1.58,0.75)}
      \foreach \p in \points {
        \draw plot [mark=*, mark size=0.8, mark options={draw=red, fill=red}] coordinates{\p}; 
      }
      \def\points{(0.95,1.78,0.75), (1.15,1.72,0.75), (1.4,1.62,0.75), (1.75,1.41,0.75), (2.0,1.21,0.75), (2.3,0.95,0.75)}
      \foreach \p in \points {
        \draw plot [mark=*, mark size=0.8, mark options={draw=black!30!green, fill=black!30!green}] coordinates{\p};
      }
      \def\points{(1.61,0.7,0.75), (1.62,0.8,0.75), (1.64,1.1,0.75), (1.67,1.4,0.75), (1.68,1.8,0.75), (1.68,2.1,0.75)}
      \foreach \p in \points {
        \draw plot [mark=*, mark size=0.8, mark options={draw=black!10!blue, fill=black!10!blue}] coordinates{\p};
      }
      \end{scope}
      
      % right figure
      \tdplotsetrotatedcoords{-25}{0}{30}
      \tdplotsetrotatedcoordsorigin{(o2)}
      \begin{scope}[tdplot_rotated_coords]
        % one-hots
        \draw [thick,red] (0,0,0) -- (2/3,-1/3,-1/3) node[anchor=north east]{$\bm\mu_1$};
        \draw plot [mark=*, mark size=1.2, mark options={draw=red, fill=red}] coordinates{(2/3,-1/3,-1/3)};
        \draw [thick,black!30!green] (0,0,0) -- (-1/3,2/3,-1/3) node[anchor=north west]{$\bm\mu_2$};
        \draw plot [mark=*, mark size=1.2, mark options={draw=black!30!green, fill=black!30!green}] coordinates{(-1/3,2/3,-1/3)};
        \draw [thick,black!10!blue] (0,0,0) -- (-1/3,-1/3,2/3) node[anchor=east]{$\bm\mu_3$};
        \draw plot [mark=*, mark size=1.2, mark options={draw=black!10!blue, fill=black!10!blue}] coordinates{(-1/3,-1/3,2/3)};
        % \draw [black!30!green, ->] (0,-1,0) -- (0,1,0) node[anchor=north west]{$\mathcal{S}^2$};
        % \draw [black!10!blue, ->] (0,0,-1) -- (0,0,1) node[anchor=east]{$\mathcal{S}^j$};
        
        % coordinate axis
        \draw [black!30!white,line width=0.3mm,dashed,->] (-1.2,0,0) -- (1.2,0,0);
        \draw [black!30!white,line width=0.3mm,dashed,->] (0,-0.9,0) -- (0,1,0);
        \draw [black!30!white,line width=0.3mm,dashed,->] (0,0,-1) -- (0,0,1);
        \draw [black!30!white] (0,0,0.25) -- (-0.25,0,0.25) -- (-0.25,0,0);
      \end{scope}
    \end{tikzpicture}
\end{center}
	\caption{Illustration of Neural Collapse. Here $\phi_{\vtheta}(\cdot)$ denotes the feature mapping of the network, i.e. the output of the penultimate layer; see \eqref{eq:func-NN} for the formal definition. %\cy{The right figure is not an ETF since it has $d = 2 < K = 3$. Maybe replace with a $d = K = 3$ figure?} \zz{What I want to plot is $d=3$, but seems right now it looks in 2D space.} \cy{Does this new figure look better or is it making it too complicated?} \zz{Great, thanks.}
	}
	\label{fig:nc}
\end{figure}
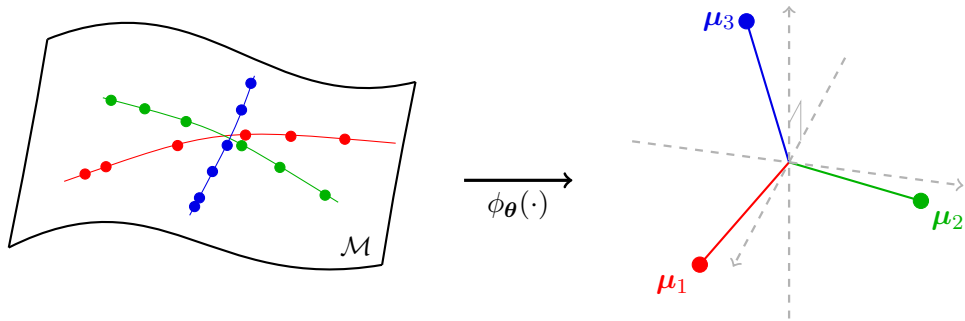

\paragraph{\emph{An Intriguing Phenomenon in Deep Network Training: Neural Collapse.}} Towards demystifying deep neural networks (DNN), recent seminal work \cite{papyan2020prevalence,papyan2020traces} empirically discovered an intriguing phenomenon that persists across a range of canonical classification problems during the terminal phase of training. As illustrated in \Cref{fig:nc}, it has been widely observed that last-layer features and classifiers of a trained DNN exhibit simple but elegant mathematical structures:

\begin{itemize}
      \item[(i)] \textbf{\emph{Variability Collapse}:} cross-example within-class variability of last-layer features collapses to zero, as the individual features of each class themselves concentrate to their isolated class-means.
    \item[(ii)] \textbf{\emph{Convergence to Simplex ETF}:} the class-means centered at their global mean are not only linearly separable, but are actually maximally distant and located on a sphere centered at the origin up to scaling (i.e., they form a Simplex Equiangular Tight Frame -- or Simplex ETF). 
    \item[(iii)] \textbf{\emph{Convergence to Self-duality}:} the last-layer linear classifiers, living in the dual vector space to that of the class-means, are 
    %almost \js{I'm confused by this "almost": either we explain these phenomena idealized (as done for the points (i and ii) above, or we say that all of these hold "almost" in all cases} 
    perfectly matched with their class-means.
    \item[(iv)] \textbf{\emph{Simple Decision Rule}:} the last-layer classifier is behaviorally equivalent to a Nearest Class-Center decision rule.
\end{itemize}
These results suggest that deep networks are learning maximally separable features between classes, and a max-margin classifier in the last layer upon these learned features, touching the ceiling in terms of the performance. This phenomenon is referred to as \emph{Neural Collapse} (\NC) \cite{papyan2020prevalence}, and it persists across the range of canonical classification problems, on different neural network architectures (e.g., VGG~\cite{simonyan2014very}, ResNet \cite{he2016deep}, and DenseNet \cite{huang2017densely}) and on a variety of standard datasets (such as MNIST \cite{lecun2010mnist}, CIFAR \cite{krizhevsky2009learning}, and ImageNet \cite{deng2009imagenet}). 

%The result implies that the deep network is learning a seemingly perfect classifier in terms of its generalization ability, touching the ceiling in terms of the performance. 
%\cy{I'd be careful with the phrase ``max-margin classifier'' since it may be confused as being max-margin in the input space rather than learned feature space. ``touching the ceiling in terms of the performance'' seems inappropriate since the result is not about generalization.  Maybe something like: The results suggests that deep networks are learning maximally separable features between classes before the last layer and a max-margin classifier in the last layer upon the learned features. }

As demonstrated in \cite{papyan2020prevalence,zarka2020separation,mixon2020neural,lu2020neural,weinan2020emergence,fang2021layer}, the symmetric and simple mathematical structure of the last-layer classifiers could potentially lead to a profound understanding of deep networks in terms of training, generalization, and robustness. 
For example, it has been conjectured that the benefits of the
interpolation regime of overparameterized networks might be directly related to \NC\ \cite{papyan2020prevalence}, as this behavior occurs simultaneously with the ``benign overfitting'' phenomenon (see \cite{ma2018power, belkin2019reconciling, liang2020just, bartlett2020benign, li2020benign}); that is, when the model perfectly interpolates the training data. 
In addition, the variability collapse of \NC\ aligns with \emph{information bottleneck} theory \cite{tishby2015deep} which hypothesizes that neural networks seek to preserve only the minimal set of information in the learned feature representations for predicting the label hence discourage any additional variabilities.
%\cy{In addition, the variability collapse of \NC\ aligns with \emph{information bottleneck} theory \cite{tishby2015deep} which hypothesizes that neural networks seek to preserve only the minimal set of information in the learned feature representations for predicting the label hence discourage any additional variabilities. }
On the other hand, a recent line of work \cite{yu2020learning,chan2020deep} raises controversial questions on whether \NC\;improves robustness against data corruptions by showing that diverse features that preserve the intrinsic structure of data can better handle label corruptions. %\js{I don't understand the previous sentence. What do you want to say?} \cy{Rewrote a bit. }
% On the other hand, a recent line of work \cite{yu2020learning,chan2020deep} raise controversial questions whether \NC\;improves robustness against noise corruptions, and in representation learning abundant empirical evidences showed that preventing \NC\;seems to be critical for the recent advances on self-supervised learning [].\zz{Not sure for the last argument about self-supervised learning.} \cy{This seems a bit far-fetched: self-supervised learning with CE loss (hence \NC) is actually not a problem, but without CE loss can cause a problem, so I'd say the issue is not really with neural collapse}
Therefore, an in-depth theoretical study of the \NC\;phenomenon could provide further insights for addressing all these fundamental questions (see \Cref{sec:conclusion} for a detailed discussion).
%[MBB18, BHMM19, LR20, BLLT20, LSS20]

%\qq{working here...}

% during the terminal phase of training 
\begin{figure}[t]
	\centering
    \includegraphics[width = 0.65\textwidth]{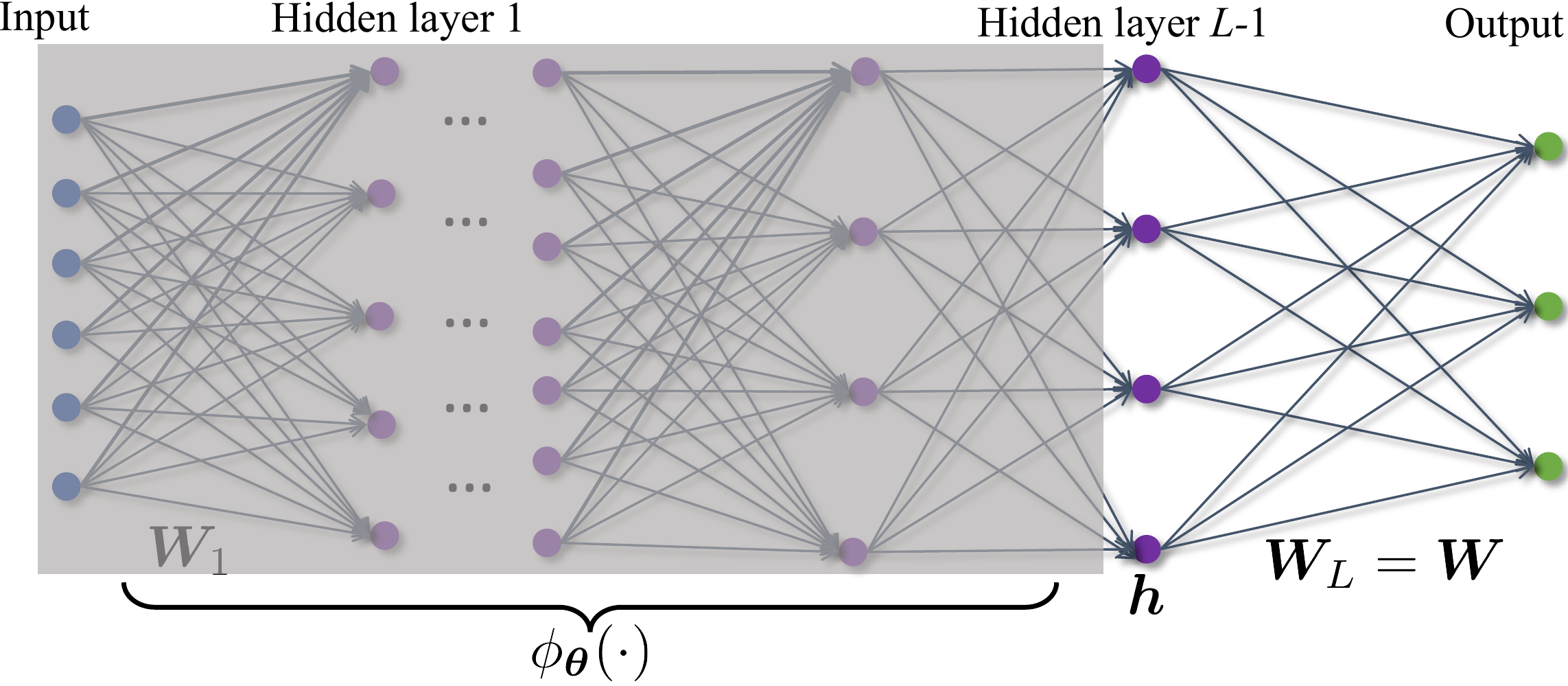}
	\caption{Illustration of the unconstrained feature model, where the gray box is peeled off so that the representation $\vh$ is modeled by a simple decision variable for every training sample.}
	\label{fig:layer-peeled}
\end{figure}

\paragraph{\emph{Towards Mathematical Analysis: Simplification via Unconstrained Feature Models.}} Fully demystifying the \NC\ phenomenon in theory can be very challenging. Perhaps the most difficult hurdle lies in the nonconvexity of the optimization problem for training neural networks, which, loosely speaking, stems from the nonlinear interaction across many different layers of neural networks. Towards this goal, a recent line of work~\cite{mixon2020neural,lu2020neural,weinan2020emergence,zarka2020separation,fang2021layer} studied the properties of last-layer classifiers and features based on the assumption of the so-called {\it unconstrained feature model} \cite{mixon2020neural} or {\it layer-peeled model} \cite{fang2021layer}. At a high level, the unconstrained feature model takes a \emph{top-down} approach to the analysis of deep neural networks \cite{mixon2020neural,lu2020neural,weinan2020emergence,zarka2020separation,fang2021layer,yu2020learning,chan2020deep}, wherein the last-layer features are modeled as \emph{free} optimization variables (hence we call them {\it unconstrained features}) along with the last-layer classifiers (see \Cref{fig:layer-peeled} for an illustration); this is in contrast to the conventional \emph{bottom-up} approach that studies the problem starting from the input \cite{haeffele2015global,baldi1989neural,kawaguchi2016deep,safran2018spurious,yun2018small,laurent2018deep,nouiehed2018learning,liang2018adding,zhu2019global,yun2018global,zhou2021local}.\footnote{Here, top-down means that we study the network starting from the last-layer down to the input layer, whereas bottom-up refers to an approach from the input up to the last layer.} The underlying reasoning is that modern deep networks are often highly overparameterized with the capacity of learning any representations \cite{cybenko1989approximation, hornik1991approximation,lu2017expressive,shaham2018provable}, so that the last-layer features can approximate, or interpolate, any point in the 
feature space. %\js{did u mean in the input space, or what is the optimization space?} \cy{Maybe ``feature space'' instead of ``optimization space''}. 
In this way, the model simplifies the study of last-layer features, enabling us to analyze the interaction between them and the last-layer classifiers.
%\xiao{Please check whether the citation is appropriate}

%\qq{talk about the benefit of top-down vs bottom-up?} \cy{Not entirely getting the wording \emph{top-down} and \emph{bottom-up}. e.g. why are the analysis of two-layer net, deep linear net, very wide net etc. a \emph{bottom-up} approach?}

Nonetheless, the simplified unconstrained feature model still leaves us a highly nonconvex training loss to be dealt with. Despite the nonconvexity, recent work \cite{mixon2020neural,fang2021layer,lu2020neural,weinan2020emergence} studied the global minimizers, proving that Simplex ETFs (i.e., \NC) are indeed global solutions to the unconstrained feature model. In particular, the work \cite{mixon2020neural} studied the training problem with the squared loss, proving that the gradient flow converges to \NC\;solutions with extra assumptions. Another line of work \cite{fang2021layer,lu2020neural,weinan2020emergence} considered the commonly used cross-entropy loss for classification, showing that the only global minimizers of the loss function are Simplex ETFs with different constraints on the weights and features.\footnote{The constraints on the features are mainly used to prevent it from approaching infinity since the objective function, with the cross-entropy loss, is not coercive. Note that we still refer to this model as an \emph{unconstrained feature model} even if they include norm constraints or regularization.} However, these results still suffer from several limitations: \emph{(i)} due to the nonconvex nature, only characterizing optimality conditions is not enough to explain the empirical global convergence of iterative algorithms to \NC, such as stochastic gradient descent (SGD); \emph{(ii)} the problem formulations differ from those typically used by practitioners, which deploy norm regularization on the network weights (i.e., weight decay), rather than enforcing constraints, for the ease of optimization.
 
% []), given spurious local minima or flat saddle points \cite{safran2018spurious}[]
 
\paragraph{Contributions of This Work.}
Inspired by these pioneering results \cite{papyan2020prevalence,zarka2020separation,fang2021layer,lu2020neural,weinan2020emergence,mixon2020neural}, in this work we take a step further by characterizing the global optimization landscape of the network training loss based on the unconstrained feature model. At a high level, our contributions can be summarized as follows.
\begin{itemize}[leftmargin=*]
    \item \textbf{\emph{Benign Global Landscape.}} For the unconstrained feature model, we provide the first result showing that a commonly used, regularized cross-entropy loss is a \emph{strict saddle function} \cite{ge2015escaping,sun2015nonconvex,zhang2020symmetry}. In other words, every critical point is either a global solution (corresponding to Simplex ETFs) or a \emph{strict saddle point}\footnote{Throughout the paper, for a minimization problem, we will not distinguish between local maxima and saddle points. We call a critical point {\it strict saddle} if the Hessian at this critical point has at least one negative eigenvalue.} with negative curvature, so that there is \emph{no} spurious local minimizer on the optimization landscape. As summarized in \Cref{table:overview}, this is in contrast to previous work \cite{fang2021layer,lu2020neural,weinan2020emergence,mixon2020neural} that only characterizes global minimizers.  
    \item \textbf{\emph{Efficient, Algorithmic Independent, Global Optimization.}} %For the simplified problem \cy{``simplified problem'' is a bit vague. Maybe remove ``For the simplified problem '' altogether}, 
    The benign global landscape implies that any method that can escape strict saddle points (e.g. stochastic gradient descent) converges to a global solution \cite{lee2016gradient} that exhibits \NC. This result supports our empirical observation, as shown in \Cref{sec:exp-NC}, that practical overparameterized networks always converge to ETF solutions with a diverse choice of optimization algorithms.
  %  This result supports the observation that, in practice, the training of deep overparameterized networks converge to ETF solutions regardless of the specific optimization algorithm used. 
  %  \cy{Make it sounds like sth already know. Perhaps ``This result supports our empirical observation, as shown in \Cref{sec:exp-NC}, that practical overparameterized networks always converge to ETF solutions with a diverse choice of optimization algorithms.''}
    %, it provides a clear theoretical explanation of \NC\;phenomenon on training modern overparameterized networks, that iterates of different type of optimization methods, ranging from first-order to second-order, \emph{all} collapse to Simplex ETFs on standard image datasets (see \Cref{sec:experiment}). %\cy{Again, ``max-margin classifier'' may be misleading. Maybe ``Efficient and Algorithm-Independent Optimization'' or simply ``Efficient Optimization''? }

 \item \textbf{\emph{Cost Reduction for Practical Network Training.}} Moreover, the universality of \NC\ implies that there is no need of training the last-layer classifiers since the weights can be simply fixed as a Simplex ETF throughout the training process. On the other hand, since \NC\ happens whenever $d\geq K$, it implies that we can choose the feature dimension $d$ comparable to the number of classes $K$, reducing the feature dimension for further computational benefits.
% there is no need of using a feature space with dimension larger than the number of classes, . 
 In \Cref{sec:exp-fix-classifier}, our experiments demonstrate that such a strategy achieves on par performance with classical training methods, leading to substantial 
cost reductions on memory and computation.
 
 %\zz{I feel like this point can be highlighted here, the abstract, and introduction.}
\end{itemize}
Our results shed new light on the question raised in the recent paper \cite{elad2020another} on the role of the optimization strategy (e.g., stochastic gradient descent) for achieving \NC\;in training practical deep networks. This question also relates to the recent highly influential work \cite{soudry2018implicit} on the implicit algorithmic bias. For multi-class classification problems with linearly separable data, this work \cite{soudry2018implicit} showed that linear predictors optimized by gradient descent converge to the max-margin classifiers even without adding any explicit regularization on the cross-entropy loss. Based on this result, a sequence of works \cite{gunasekar2018implicit,gidel2019implicit,nacson2019convergence,ji2020gradient,lyu2019gradient,shamir2021gradient,jagadeesan2021inductive,NEURIPS2020_cd42c963} laid great emphasis on the notion of ``inductive bias'' of particular optimization algorithms as a reason for the surprising success in training deep learning models.\footnote{While (stochastic) gradient descent and generic steepest descent methods converge to max-margin classifiers, it should be noted that other commonly used optimization algorithms in deep learning, such as AdaGrad \cite{duchi2011adaptive} and Adam \cite{kingma2014adam}, do not have max-margin properties in general and their solutions depend upon initialization, step-size and other algorithm hyper-parameters \cite{nacson2019stochastic,qian2019implicit,gunasekar2018characterizing}.}
In contrast, both our theoretical result on the global landscape for the unconstrained feature model and the empirical evidence on practical deep models demonstrate that \NC\;in network training is facilitated \emph{not only} by the choice of the optimization methods, but more importantly, by the choice of loss functions and the power of overparameterization in the network architecture.

\begin{table*}[t]
%\begin{wraptable}{r}{0.59\linewidth}
\centering
%\vspace{-9pt}
    \begin{footnotesize}
        \caption{\footnotesize \label{table:overview} Comparison of the setup and results under the unconstrained feature model with cross-entropy loss. }
        \begin{tabular}{@{\,}c|ccccc@{\,}}
        \hline
  &  \multicolumn{2}{c}{Regularizer}  & Bias term & \multicolumn{2}{c}{Results} \\ \hline
  & Constraint & Weight decay & &Global minimizer & Landscape \\ \hline
  \cite{lu2020neural} & $\checkmark$ & & & $\checkmark$ & \\ \hline
  \cite{weinan2020emergence} & $\checkmark$ & & & $\checkmark$ & \\ \hline
  \cite{fang2021layer}  & $\checkmark$ & & & $\checkmark$ &  \\ \hline
  This paper & & $\checkmark$ & $\checkmark$ & $\checkmark$ & $\checkmark$ \\ \hline
        \end{tabular}
    \end{footnotesize}
%\vspace{-8pt}
%\end{wraptable} 
\end{table*}

%\qq{make sure the citation comprehensive below, e.g. Rong Ge's recent paper on mildly overparameterized network}

\paragraph{Relationship to the Prior Arts.} 
Our work highly relates to recent advances on studying the optimization landscape in neural network training; see \cite{sun2020optimization,sun2020global} for a contemporary survey. Most of the existing work \cite{haeffele2015global,baldi1989neural,kawaguchi2016deep,safran2018spurious,yun2018small,laurent2018deep,nouiehed2018learning,liang2018adding,zhu2019global,yun2018global,zhou2021local} analyzes the problem based on a \emph{bottom-up} approach -- from the input to the output of neural networks -- ranging from two-layer linear network \cite{baldi1989neural,nouiehed2018learning,kunin2019loss,zhu2019global}, deep linear network \cite{kawaguchi2016deep,nouiehed2018learning,yun2018global,laurent2018deep}, to nonlinear network \cite{safran2018spurious,nouiehed2018learning,yun2018small,liang2018adding}. More specifically, the line of work \cite{baldi1989neural,kawaguchi2016deep,nouiehed2018learning,zhu2019global,yun2018global} studied the optimization landscape for linear two-layer networks and proved that the associated training loss is a strict saddle function. For deeper linear networks, it can be shown that flat saddle points exist at the origin, but there are no spurious local minima \cite{kawaguchi2016deep,laurent2018deep}. For nonlinear neural networks, it has been proved that there do exist spurious local minima \cite{safran2018spurious,yun2018small}, but such local minima may be eliminated, or the number can be significantly reduced, in the over-parameterization regime \cite{safran2018spurious,liang2018adding}. Additionally, the work \cite{haeffele2015global} proved that certain local minima (having an all-zero ``slice'') are also global solutions, but the analysis is crucially dependent on the sufficient condition of an all-zero slice in the weights, which is insufficient to characterize the landscape properties. At a high level, the differences between these results and ours can be summarized as follows.
\begin{itemize}[leftmargin=*]
 %   \item \textbf{\emph{Optimization vs. Learning Perspectives.}} Most of these results based on the bottom-up approach only focus on the optimization perspective of deep learning, explaining why the training loss can be efficiently minimized to optimality. However, they provided limited insights into the empirical success of learning in terms of learned representations, generalization, and robustness. In contrast, we take a \emph{top-down} approach to look at the network starting from the very last layer. The slight difference in the models can lead to a dramatic difference in the interpretability for deep learning. As we will show in this work, by starting from the last layer, our results not only provide valid reasons on why the training loss can be efficiently optimized, but also provide new insights on the learning perspective through the last-layer classifiers and features learned from the network. \js{I'm not super sure about this point on ``optimization vs learning''. There are works that look at optimization and can show max-min solutions, which naturally lead to learning (generalization) guarantees (see eg [Wei, Lee, Liu \& Ma, 2018]).} \zz{Agree with JS. We need to be careful about ``learning". I thinking ``learning perspective" is mainly referred to generalization and other learning issues, but this paper focuses more on the optimization part. What this paragraph expresses is that here our optimization problem looks at the representation or feature, while most of the others look at the loss function.}\cy{Agreed \& how about the following:}

    \item \textbf{\emph{A Feature Learning Perspective.}} While most of these results based on the bottom-up approach explain optimization and generalization of certain types of deep neural networks, they provided limited insights into the practice of deep learning. In contrast, we take a \emph{top-down} approach to look at the network starting from the very last layer. The slight difference in the models can lead to a dramatic difference in the interpretability for deep learning. By starting from the last layer, our results not only provide valid reasons on why the training loss can be efficiently optimized, but also provide a precise characterization of the last-layer features as well as the classifiers learned from the network. As we will show, such a feature learning perspective not only helps with network design (see \Cref{sec:exp-fix-classifier}), but may bear broadly on generalization and robustness of deep learning as well as the recent development of contrastive learning (see \Cref{sec:conclusion}).% \cy{Best I can come up with to replace the previous one. Hopefully this is better}

    \item \textbf{\emph{Connections to Empirical Phenomena.}} Moreover, most existing theoretical results on landscape analysis \cite{allen2019convergence,du2019gradient} are somewhat disconnected from practice due to unrealistic assumptions, providing limited guiding principles for modern network training or design. In contrast, our assumption on the last-layer features as optimization variables is naturally based on model overparameterization. Moreover, our results provide explanations for \NC, an empirical phenomenon that has been widely observed on convincing numerical evidence across many different practical network architectures and a variety of standard image datasets.
\end{itemize}
%\qq{cite Sam's work and discuss relationship to NTK?}

Our work also broadly relates to the recent theoretical study for deep learning based on Neural Tangent Kernels (NTKs) \cite{jacot2018neural}, where a neural network behaves like a \emph{linear} model on random features hence has a benign optimization landscape. 
However, the ``kernel regime'' that NTKs work in requires neural networks that are infinitely wide -- or at least so wide that is beyond the regime that practical neural networks work in \cite{zou2020gradient,montanari2020interpolation,buchanan2020deep,chizat2019lazy}. 
% Consequently, there has been controversy on whether NTKs provides insights for understanding the practice of deep learning at all. 
In contrast, we adopt the unconstrained feature model which does not directly impose requirements on the width of the neural network and, as shown in our experiments, well captures the behavior of practical neural networks. 
Hence, our result can provide a more practical understanding for deep learning. 

Moreover, from a boarder perspective our work is rooted in recent advances on global nonconvex optimization theory for signal processing and machine learning problems \cite{chi2019nonconvex,danilova2020recent,zhang2020symmetry,qu2020finding,10.1145/3418526}. In a sequence of works \cite{zhao2015nonconvex,sun2016complete1,sun2015complete,li2016,qu2017convolutional,sun2018geometric,zhu2018global,li2019symmetry,zhu2018distributed,chi2019nonconvex,qu2020geometric,qu2020finding,qu2020exact,lau2020short,kuo2019geometry,zhang2019structured}, the authors showed that many problems exhibit ``equivalently good'' global minimizers due to symmetries and intrinsic low-dimensional structures, and the loss functions are usually strict saddles \cite{ge2015escaping,sun2015nonconvex,zhang2020symmetry}. These problems include, but are not limited to, phase retrieval \cite{sun2018geometric,qu2017convolutional}, low-rank matrix recovery \cite{zhao2015nonconvex,li2016,zhu2018global,chi2019nonconvex,zhu2018distributed}, dictionary learning \cite{qu2014finding,sun2016complete1,sun2015complete,qu2020geometric,qu2020finding}, and sparse blind deconvolution \cite{kuo2019geometry,qu2020exact,lau2020short,zhang2019structured}. As we shall see, the global minimizers (i.e., simplex ETFs) of our problem here also exhibit a similar rotational symmetry, compared to low-rank matrix recovery. In fact, our proof techniques are inspired by recent results on low-rank matrix recovery \cite{zhu2018global,chi2019nonconvex}. Thus, our work establishes a new connection between recent advances on nonconvex optimization theory and deep learning.

%As we can see, our work establishes interconnections among communities of nonconvex optimization, theoretical deep learning, and practice. 

%\qq{the last sentence might need some rephrasing} \cy{Don't forget to cite Lerman.}\zz{Haha, I am thinking of any Learman's work on the global landscape analysis.}

% while  In comparison, 

%More specifically, 

%more about nonconvex optimization in general

%===============

%In seeking to better understand the associated optimization problems in training neural networks, theoretical efforts have been dedicated to analyze their geometric landscape. The work \cite{baldi1989neural,kawaguchi2016deep,laurent2018deep,nouiehed2018learning,zhu2019global,yun2018global} studied optimization landscape for linear neural networks and proved that the associated training problem  is a strict saddle function for shallow network, but only has no spurious local minima for deeper network. However, most of these results are about regression problems with least squares loss function, without using the weight decay and bias term, and thus cannot be directly applied to the problem considered in this paper.  See \cite{sun2020optimization} for a recent survey.

%\zz{Shall we also provide a very brief review of landscape analysis for other problems? Right now, I add a brief review right after \Cref{thm:global-geometry}.}

\paragraph{Notations, Organizations, and Reproducible Research.} 
Throughout the paper, we use bold upper and lowercase letters, such as $\mb X$ and $\mb x$, to denote matrices and vectors, respectively. For a vector $\mb a$ and a matrix $\mb A$, We use $\norm{\mb a}{2}$ and $\norm{\mb A}{F}$ to denote the their standard $\ell_2$-norm and Frobenius norm, respectively. Not-bold letters are reserved for scalars. The symbols $\mb I_K$ and $\mb 1_K$ respectively represent the identity matrix and the all-ones vector with an appropriate size of $K$, where $K$ is some positive integer. We use $[K]:=\Brac{1,2,\dots,K}$ to denote the set of all indices up to $K$. For any matrix $\mb A \in \bb R^{n_1 \times n_2}$, we write $\mb A = \begin{bmatrix} \mb a_1 & \dots & \mb a_{n_2} \end{bmatrix} = \begin{bmatrix}
(\mb a^1)^\top \\
\vdots \\
(\mb a^{n_1})^\top
\end{bmatrix}$, so that $\mb a_i$ ($1\leq i\leq n_2$) and $\mb a^j$ ($1\leq j\leq n_1$) denote a column vector and a row vector of $\mb A$, respectively. For any function $\varphi: \bb R^n \mapsto \bb R$, we use $\nabla \varphi$ and $\nabla^2 \varphi$ to denote its gradient and Hessian, respectively.

The rest of the paper is organized as follows. In \Cref{sec:problem}, we introduce the basic setup of the problem and the motivations behind our model. We present our main theoretical results in \Cref{sec:main} and discuss their implications. In \Cref{sec:experiment}, we provide numerical simulations on practical neural networks to validate our theoretical findings. Finally, we conclude the paper in \Cref{sec:conclusion}. We postpone all the detailed proofs to the Appendix. To reproduce the experimental results in \Cref{sec:experiment}, our code is publicly available on GitHub via the following link
\begin{center}
    \url{https://github.com/tding1/Neural-Collapse}.
\end{center}

\section{The Problem Setup}\label{sec:problem}

In this section, we start by reviewing the basics of deep neural networks in \Cref{subsec:basics}, and then move onto introducing the unconstrained feature model in \Cref{sec:layer-peeled-model}.

\subsection{Basics of Deep Neural Networks}\label{subsec:basics}
%Let us consider a deep neural network for multi-class classification problems (e.g., ) $K$-class classification problem (in logits), which

A deep neural network is essentially a \emph{nonlinear} mapping $\psi(\cdot): \R^{D} \mapsto \R^K$, %for $\mb x\in\mathcal X\subset \mathbb{R}^D$, %\js{the space of $\mb x$ was never defined. Should we say that we consider $\mb x\in\mathcal X\subset \mathbb{R}^d$?}
which can be modeled by a composition of simple maps: $\psi(\mb x) = \psi^{L} \circ \cdots \circ \psi^2\circ \psi^1(\mb x)$ for $\vx\in\R^D$, where $\psi^\ell(\cdot)$ ($1\leq \ell \leq L)$ are called ``layers''. Each layer is composed of an affine transform, represented by some weight matrix $\mb W_\ell$, and bias $\mb b_\ell$, followed by a simple \emph{nonlinear}\footnote{Here, the nonlinear operator may include activations such as ReLU \cite{nair2010rectified}, pooling, and normalization \cite{ioffe2015batch}%\js{Careful: adaptive ''normalization'' stuff isn't entry wise}
, etc.}
%\cy{Doesn't seem right, please make sure. Mean pooling is linear, max pooling is nonlinear but does not really function as a nonlinear layer... strides sounds linear to me as well. BN is linear at least in the test phase, though for training it is tricky due to batch statistics. If there is any doubts, maybe remove this footnote} 
activation function $\sigma(\cdot)$. More precisely, a vanilla $L$-layer neural network can be written as
\begin{align}\label{eq:func-NN}
    \psi_{\mb \Theta}(\mb x) \;=\;  \mb W_L \underbrace{\sigma\paren{ \mb W_{L-1} \cdots \sigma \paren{\mb W_1 \mb x + \mb b_1} + \mb b_{L-1} }}_{\phi_{\vtheta}(\vx)}  + \mb b_L.
\end{align}
%\js{Notation problem: here $f$ is defined with two input parameters (theta and x), but in the paragraph above it's just x. Consider perhaps writing $f_{\mb\Theta}(\mb x)$ instead (which is actually the notation for the map $\phi_\theta(\mb x)$).}
%\zz{Changed the network mapping to $y$, which looks a little bit wired as $y$ is usually used to denote the output of the network. If it looks good, we also need to change all the $\vy$ in the appendix.}

For convenience, we use $\mb \Theta = \Brac{ \mb W_k,\mb b_k }_{k=1}^L$ to denote \emph{all} the network parameters, and use $\mb \theta = \Brac{ \mb W_k,\mb b_k }_{k=1}^{L-1}$ to denote the network parameters up to the last layer. The output of the penultimate layer, denoted by $\phi_{\vtheta}(\vx)$, is usually referred to as the \emph{representation} or \emph{feature} of the input $\vx$ learned from the network. In this way, the function implemented by a neural network classifier can also be expressed as a linear classifier acting upon $\phi_{\vtheta}(\vx)$. %\js{In this way, the function implemented by a neural network can also be expressed as a linear classifier acting upon $\phi_{\vtheta}(\vx)$. }

%\js{I'm rewriting the following paragraph a bit}\\
The goal of deep learning is to fit the parameters $\mb \Theta$ so that the output of the model on an input samples $\mb x$ approximates the corresponding output $\vy$, i.e. so that $\psi_{\mb \Theta}(\mb x)\approx \mb y$, in expectation over a distribution of input-output pairs, $\mc D$. 
% The goal of deep learning is to fit the observation $\mb y$ with the output $f(\mb \Theta;\mb x)$ for any sample $\mb x$ from the distribution $\mc D$, by tuning the network parameters $\mb \Theta$. 
%The goal of deep learning is to fit the parameters $\mb \Theta$ so that the output of the model on input samples approximate the corresponding outputs over a distribution of input-output pairs, $\mc D$. 
This can be achieved by optimizing an appropriate loss function $\calL(\psi_{\mb \Theta}(\mb x),\vy)$ which quantifies this approximation. In this work, we focus on multi-class classification tasks (say, with $K$ classes), where the class label of a sample $\mb x$ is given by a one-hot vector $\vy \in \mathbb{R}^K$ representing its membership to one of the $K$ classes. In this setting, cross-entropy is one of the most popular choices for the loss function. Naturally, the distribution $\mc D$ is unknown, but we have access to training samples that are drawn i.i.d. from $\mc D$. In this way, one can minimize the empirical risk over these samples by optimizing the following problem
\begin{align}\label{eqn:dl-ce-loss}
    \min_{\mb \Theta} \; \sum_{k=1}^K \sum_{i=1}^{n_k} \Lce \paren{ \psi_{\mb \Theta}(\mb x_{k,i}),\vy_k } \;+\; \frac{\lambda}{2} \norm{\mb \Theta}{F}^2,
\end{align}
where $\vy_k \in \bb R^K $ is a one-hot vector with only the $k$th entry equal to unity ($1\leq k \leq K$), $\Brac{n_k}_{k=1}^K$ are the numbers of training samples in each class, and $\lambda>0$ is the regularization parameter (or weight decay penalty), and $\Lce(\cdot,\cdot)$ is the cross-entropy loss:
\begin{align}\label{eqn:ce}
    \Lce(\mb z,\pi_0(\vz)) \;:=\; - \log \paren{  \frac{ \exp(z_k) }{ \sum_{i=1}^K \exp(z_i) } },
\end{align}
where we assume $\vz$ belongs to the $k$th class.
As introduced in \Cref{sec:intro}, recent work \cite{papyan2020prevalence} showed that the features learned by minimizing the above objective (i.e. $\phi_{\mb \theta}(\mb x)$) showcase the \NC\;phenomenon: their within-class variability vanishes, and the features converge to a Simplex ETF.
%For simplicity, we assume that all the bias are zero, so that $\mb \theta = \mb \theta_{\mb W} = \Brac{ \mb W_k }_{k=1}^L$. 

\subsection{Problem Formulation Based on Unconstrained Feature Models}\label{sec:layer-peeled-model}
In deep network models, the nonlinearity and interaction between a large number of layers results in tremendous challenges for analyzing this learning problem. Since modern networks are often highly overparameterized to approximate any continuous function %\cite{hornik1991approximation}\js{This comment is conflicting: it mentions overparametrization in the sense that can fit any training set, but it cites a work on approximation of continuous functions, which is a different problem} 
and the characterization of \NC\; only involves the last-layer features $\phi_{\vtheta}(\vx)$, a natural idea to simplify the analysis is to treat these features as free optimization variables $\mb h = \phi_{\vtheta}(\vx)\in\mathbb R^d$, which motivates the name {\it unconstrained feature model}\footnote{This model is also called {\it layer-peeled model} in \cite{fang2021layer}, where one ``peels'' off the first $L-1$ layers. It has also been recently studied in \cite{lu2020neural,weinan2020emergence}. Throughout the paper, we will simply call it unconstrained feature model.} \cite{mixon2020neural} (see \Cref{fig:layer-peeled} for an illustration). In this way, we can rewrite the network output as $\psi_{\mb \Theta}(\mb x) = \mb W_L \mb h + \mb b_L  $. 

For simplicity, we consider the setting where the number of training samples in each class is balanced (i.e., $n = n_1 =\cdots = n_K$). We also write $\mb W = \mb W_L$ and $\mb b = \mb b_L$ for conciseness. Based on the unconstrained feature model, we consider a slight variant of \eqref{eqn:dl-ce-loss}, given by
\begin{align}\label{eq:obj}
     \min_{\mb W , \mb H,\mb b  } \; \underbrace{ \frac{1}{Kn} \sum_{k=1}^K \sum_{i=1}^{n} \Lce \paren{ \mb W \mb h_{k,i} + \mb b, \mb y_k } \;+\; \frac{\lambda_{\mb W} }{2} \norm{\mb W}{F}^2 + \frac{\lambda_{\mb H} }{2} \norm{\mb H}{F}^2 + \frac{\lambda_{\mb b} }{2} \norm{\mb b}{2}^2 }_{ := f(\mb W,\mb H,\mb b)},
\end{align}
%\js{Notation: just noticed that this choice of $f$ for the cost is problematic, given that $f$ is the network output.} \zz{Good catch. I changed the output of the network to $y$.}\\
with $\mb W \in \bb R^{ K \times d}$, $\mH = \begin{bmatrix}\vh_{1,1} \cdots \vh_{K,n} \end{bmatrix}\in \bb R^{d \times N}$ (here, we denote $N = nK$), $\mb b \in \bb R^K$, and $\lambdaW,\lambdaH,\lambdab>0$ are the penalty parameters for the weight decay.

%\js{note sure why the paragraph that follows is a remark. I would just include this by continuing the presentation (without making explicit the "remark").}
%\paragraph{Remark.} 
As summarized in \Cref{table:overview}, similar optimization problems have been considered in \cite{fang2021layer,lu2020neural,weinan2020emergence}. In contrast to these, our problem formulation here \eqref{eq:obj}, with bias and weight decay, is closer to the loss used in practice for training neural networks; existing work \cite{fang2021layer,lu2020neural,weinan2020emergence} considered constrained\footnote{For example, the work \cite{fang2021layer} considers inequality constraints such that the energy of $\mb W$ and $\mb H$ are bounded; the other work \cite{lu2020neural,weinan2020emergence} enforces $\mb W$ and $\mb H$ on the spheres up to scaling. } variants of \eqref{eq:obj} and without the bias term, which can be implemented but seldom used in practice due to the difficulty of optimization. In the following, we briefly discuss the differences between our simplification and practical settings for training neural networks.
\begin{itemize}[leftmargin=*]
    \item \textbf{Weight Decay on $\mb W$ and $\mb H$.} One simplification of our formulation is in the weight decay. In practice, people usually impose the weight decay on all the network parameters $\mb \Theta$, while we enforce weight decay on the last layer's classifier, $\mb W$, and features, $\mb H$. However, this idealization is reasonable since the energy of the features (i.e., $\|\mb H\|_F$) can indeed be upper bounded by the energy of the weights at every layer if the inputs are bounded (which holds in practice), implying that the norm of $\mb H$ is \emph{implicitly} penalized by penalizing the norm of $\mb \Theta$.  %\js{I would add: Note that $\|\mb H\|$ can indeed be upper bounded by the spectral norm of the weights at every layer if the inputs are bounded (which holds in practice).} 
    Our experiments in \Cref{subsec:exp-validity} demonstrate that both approaches exhibit similar \NC\;phenomena and comparable performance in practice. %\cy{Be careful as to whether we have this at the end}

%    since penalizing the energy of $\mb \Theta$ is implicitly penalizing 
    \item \textbf{Treating the Last-layer Features as Optimization Variables.} One may question that ``peeling off'' the $L-1$ layers might oversimplify the problem. Nonetheless, this simplification (which is also adopted in \cite{fang2021layer,lu2020neural,weinan2020emergence}) is based on the fact that neural networks with sufficient overparameterization can approximate any function -- in \Cref{subsec:exp-validity}, we numerically demonstrate that \NC\; persists even when we train overparametrized networks on randomly generated labels. Moreover, as we shall see in the following sections, both our theory and experiments demonstrate that our simplification preserves the core properties of last-layer classifiers and features during training -- the \NC\;phenomenon. More specifically, in \Cref{sec:main} we show that Simplex ETFs are the only global minimizers to our simplified loss function \eqref{eq:obj}, and the loss function is a strict saddle function with no other spurious local minimizers so that it can be optimized efficiently to global optimality.
\end{itemize}

\section{Main Theoretical Results}\label{sec:main}

%In next section, we will analyze the optimization landscape for the objective function in \eqref{eq:obj}.

In this section, we present our study on global optimality conditions as well as the optimization landscape of the nonconvex loss in \eqref{eq:obj}.
%\zz{Since each subsection only contains one result, do we still need to put them as subsections?}

%To ease the presentation, throughout the paper, denote by $\vw_k$ as the $k$-th row of $\mW$, i.e., $\mW = \begin{bmatrix}\vw_1 & \cdots & \vw_K \end{bmatrix}^\top$.

\subsection{Global Optimality Conditions for \eqref{eq:obj}}
We begin by characterizing the global solutions of \eqref{eq:obj}, showing that the Simplex ETFs are its only global minimizer. 

% in the following result.
%First of all, we can show that the Simplex ETF are the only global solutions. 
\begin{theorem}[Global Optimality Conditions]\label{thm:global-minima}
	Assume that the feature dimension $d$ is no smaller than the number of classes $K$, i.e. $d\ge K$, and the number of training samples in each class is balanced, $n = n_1 = \cdots = n_K$. Then any global minimizer $(\mW^\star, \mH^\star,\vb^\star)$ of $f(\mb W,\mb H,\mb b)$ in \eqref{eq:obj} satisfies 
\begin{align}
\begin{split}
        &  w^\star  :\;=\; \norm{\mb w^{\star 1}}{2} \;=\; \norm{\mb w^{\star 2} }{2} \;=\; \cdots \;=\; \norm{\mb w^{\star K} }{2} , \quad \text{and}\quad \mb b^\star = b^\star \mb 1, \\ 
    & \vh_{k,i}^\star \;=\;  \sqrt{ \frac{ \lambda_{\mb W}  }{ \lambda_{\mb H} n } } \vw^{\star k},\quad \forall \; k\in[K],\; i\in[n],  \quad \text{and} \quad  \ol{\mb h}_{i}^\star \;:=\; \frac{1}{K} \sum_{j=1}^K \mb h_{j,i}^\star \;=\; \mb 0, \quad \forall \; i \in [n],
\end{split}
\end{align}
where either $b^\star = 0$ or $\lambdab=0$, and the matrix $\mW^{\star\top} \in \bb R^{d \times K} $ forms a $K$-Simplex ETF (defined in Definition \ref{def:simplex-ETF}) up to some scaling and rotation, in the sense that for any $\mb U\in \bb R^{d \times d}$ with $\mb U^\top \mb U = \mb I_d$, the normalized matrix $\mb M := \frac{1}{w^\star} \mb U \mW^{\star\top} $ satisfies
\begin{align}\label{eqn:simplex-ETF-closeness}
    \mb M^\top \mb M \;=\; \frac{K}{K-1}  \paren{ \mb I_K - \frac{1}{K} \mb 1_K \mb 1_K^\top }.
\end{align}

\end{theorem}
%\zz{Here $\mb M \mb M^\top \;=\; \mb M^\top \mb M$ is not true. We only have $\bm M^\top \bm M = \frac{K}{K-1}  \paren{ \mb I_K - \frac{1}{K} \mb 1_K \mb 1_K^\top }$.}

% \cy{This seems to be suggesting that feature dimension $d$ should not affect the performance of the model once we have $d > K$. This is a bit contradicting with the belief that wider networks are better. Maybe we can do some experiments with varying $d$ and see if the acc curve is mostly flat with $d > K$ and drops quickly with $d < K$. Also interesting is that it requires $d > K$ but I wonder if there can be any comment for the case $d = K$.}

% \zz{Good point. \Cref{thm:global-minima} actually works when $d=K$, but \Cref{thm:global-geometry} requires $d> K$.}

% \cy{Thanks, maybe better to say $d \ge K$ in \Cref{thm:global-minima} since there is a discussion of the case $d = K$ for \Cref{thm:global-geometry}? Also maybe better to say $d \ge K$ is ``tight'' (see below)} \zz{Agree. I have uncolored the sentence. Thanks.}

At a high level, our proof technique finds lower bounds for the loss in \eqref{eq:obj} and studies the conditions for the lower bounds to be achieved, similar to \cite{lu2020neural,fang2021layer}. We postpone the details to Appendix \ref{app:thm-global}.

\paragraph{Remark.} As can be seen in this result, any global solution of the loss function \eqref{eq:obj} exhibits \NC\;in the sense that the variability of output features $\Brac{\vh_{k,i}^\star}_{i=1}^n$ of each class $k$ ($1\leq k\leq K$) collapses to zero, and any pair of features $\paren{ \vh_{k_1,i}^\star, \vh_{k_2,j}^\star }$ from different classes $k_1\not = k_2$ are the maximally separated, under the constraint of equal-sized angles between all classes. Similar results have been obtained in \cite{fang2021layer,lu2020neural,weinan2020emergence}, which considered different problem formulations, as we have discussed in \Cref{sec:layer-peeled-model}.
\begin{itemize}[leftmargin=*]
    \item \textbf{Relationship between Class Number $K$ and Feature Dimension $d$.} The requirement that $d \ge K$ is necessary for  \Cref{thm:global-minima} to hold, simply because $K$ vectors in $\bb R^d$ cannot form a $K$-Simplex ETF if $K>d$. However, the relationship $d \ge K$ is often true in practice. In general, and in overparameterized models in particular, the feature dimension, $d$, is significantly larger than the number of classes, $K$. For example, the dimension of the features of a ResNet \cite{he2016deep} is typically set to $d=512$ for CIFAR10 \cite{krizhevsky2009learning}, a dataset with $K = 10$ classes. This dimension grows to $d=2048$ for ImageNet \cite{deng2009imagenet}, a dataset with $K = 1000$ classes.
    \item \textbf{Interpretations on the Bias Term $\mb b^\star$.} In contrast to previous works \cite{fang2021layer,lu2020neural,weinan2020emergence}, we consider the bias term in the unconstrained feature model \eqref{eq:obj}. Our result indicates that a collapsing phenomenon also exists in the bias term $\vb^\star$, in the sense that all the elements of $\vb^\star$ are identical. When the features $\mb H$ are completely unconstrained, our result implies that removing the bias term $\vb$ has no influence on the performance of the classifier. However, it should be noted that the ReLU unit is often applied at the end of the penultimate layer, so that $\mb H$ should be constrained to be nonnegative, $\mb H \geq \mb 0$. In such cases, $\overline \vh_i^\star$ will no longer be zero, and neither will $\mb b^\star$. Here, the bias term $\mb b^\star$ will compensate for the global mean of the features, so that the globally-centered features still form a Simplex ETF \cite{papyan2020prevalence}.\footnote{Suppose that an optimal solution to \eqref{eq:obj} is $(\mW^\star,\mH^\star,\vb^\star)$, satisfying the conditions in \Cref{thm:global-minima}. There exists a nonzero vector $\mb \alpha \in \bb R^d$ such that $\widetilde \mH^\star = \mH^\star + \valpha \mb 1^\top \ge 0$. Here, $\valpha \mb $ can be viewed as the global mean of $\widetilde \mH^\star$ since $\mH^\star$ has mean zero. Then, let $\widetilde\vb^\star = - \mb W^\star \valpha$, so that $\wt{\mb W}^\star \widetilde \mH^\star + \widetilde\vb^\star \mb 1^\top \;=\; \mb W^\star \mH^\star + (\mb W^\star \valpha + \widetilde\vb^\star)\mb 1^\top \;=\; \mW^\star\mH^\star$. Therefore, we can see that $(\wt{\mW}^\star,\widetilde\mH^\star,\widetilde\vb^\star)$ achieves the same cross-entropy loss as $(\mW^\star,\mH^\star,\vb^\star)$. } %\js{i think it might be worthwhile to show this formally as a corollary. Otherwise people who don't read the footnote in detail might think "oh, their result doesn't apply to the case of nonnegative features".} \zz{Good point. This is what we did before, but then we realized we can't formally prove it because of the weight decay. Thus, we add a footnote and show the case for cross-entropy.}
    %in the nonnegative orthant. 
\end{itemize}

\subsection{Characterizations of the Benign Global Landscape for \eqref{eq:obj}}\label{subsec:main-geometry}

The global optimality condition in \Cref{thm:global-minima} does not necessarily mean that we can achieve these global solutions efficiently, as the problem is still nonconvex. We now investigate the global optimization landscape of \eqref{eq:obj} by characterizing all of its critical points. Our next result implies that the training loss is a strict saddle function, and every critical point is either a global minimizer or a strict saddle point that can be escaped using negative curvatures. As a consequence, this implies that the global solutions of the training problem in \eqref{eq:obj} can be efficiently found from random initializations.

\begin{theorem}[No Spurious Local Minima and Strict Saddle Property]\label{thm:global-geometry}
		Assume that the feature dimension is larger than the number of classes, $d> K$, and the number of training samples in each class is balanced $n = n_1 = \cdots = n_K$. Then the function $f(\mb W,\mb H,\mb b)$ in \eqref{eq:obj} is a strict saddle function with no spurious local minimum, in the sense that
	\begin{itemize}
	    \item Any local minimizer of \eqref{eq:obj} is a global minimizer of the form shown in \Cref{thm:global-minima}.
	    \item Any critical point $(\mb W,\mb H,\mb b)$ of \eqref{eq:obj} that is not a local minimizer is a strict saddle with negative curvature, i.e. the Hessian $\nabla^2 f(\mb W,\mb H,\mb b)$, at this critical point, is non-degenerate and has at least one negative eigenvalue, i.e. $\exists~ i : ~\lambda_i\paren{\nabla^2 f(\mb W,\mb H,\mb b)}<0$.
	\end{itemize}
\end{theorem}

%\qq{intuition behind creating the negative curvature direction, cite the paper: Loss Landscapes of Regularized Linear Autoencoders, make some connection to low-rank sensing, implicit bias for low-rank matrix factorization }

In a nutshell, our proof relies on connecting the original nonconvex optimization problem to its corresponding low-rank convex counterpart, so that we can obtain the global optimality conditions for the original problem in \eqref{eq:obj} based on the convex one. With this, we can then characterize the properties of all critical points based on the optimality conditions. We defer all details of this proof to Appendix \ref{app:thm-global}.

\paragraph{Remark.} Existing results \cite{fang2021layer,lu2020neural,weinan2020emergence} have \emph{only} studied the global minimizers of the original problem, which has limited implication for optimization. In contrast, \Cref{thm:global-geometry} characterizes the properties for {\it all} critical points of the function in \eqref{eq:obj}. As a consequence of this result, many first-order and second-order optimization methods \cite{bottou2018optimization} optimizing $(\mb W,\mb H,\mb b)$ are guaranteed to converge to a global solution of \eqref{eq:obj}. In particular, the result in \cite{ge2015escaping,lee2016gradient} ensures that (stochastic) gradient descent with random initialization, the \emph{de facto} optimization algorithm used in deep learning, almost surely escapes strict saddles and converges to a second-order critical point -- which happens to be a global minimizer of form showed in \Cref{thm:global-minima} for our problem \eqref{eq:obj}.

\begin{itemize}[leftmargin=*]
    \item \textbf{Constructing the Negative Curvature Direction for Strict Saddles. }%\js{not a big fan of "cooking", how about "finding"?} 
     One of the major difficulties in our proof is to construct the negative curvature direction for strict saddle points. Here, we exploit the fact that the feature dimension $d$ is larger than the number of classes $K$, and construct the negative curvature direction within the null space of $\mW \in \bb R^{K \times d}$. This is also the main reason for the requirement $d>K$ in \Cref{thm:global-geometry}, but we conjecture the results also hold for $d= K$ and could be proved with more sophisticated analysis, which is left as future work.
    \item \textbf{Relationship to Low-Rank Matrix Recovery.} As discussed in \Cref{sec:intro}, it has been recently shown that the strict saddle property holds for a wide range of nonconvex problems in machine learning \cite{qu2014finding,zhao2015nonconvex,sun2016complete1,sun2015complete,li2016,qu2017convolutional,sun2018geometric,zhu2018global,li2019symmetry,zhu2018distributed,chi2019nonconvex,qu2020geometric,qu2020finding,qu2020exact}, including low-rank matrix recovery \cite{ge2016matrix,bhojanapalli2016global,ge2017no,li2019non,li2019symmetry,chi2019nonconvex}. 
   As we know that $\norm{\mb Z}{*} \;=\; \min_{ \mb Z = \mb W \mb H } \; \frac{1}{ 2  } \paren{ \norm{\mb W}{F}^2 +  \norm{\mb H}{F}^2  }$ (see \cite{haeffele2015global} for a proof), our formulation in \eqref{eq:obj} is closely related to low-rank matrix problems \cite{ge2016matrix,bhojanapalli2016global,ge2017no,li2019non,li2019symmetry,chi2019nonconvex} with the Burer-Moneirto factorization approach \cite{burer2003nonlinear}, by viewing $\mW$ and $\mH$ as two factors of a matrix $\mb Z = \mb W \mb H $. The differences lie in the loss functions and statistical properties of the problem.\footnote{We consider the cross-entropy loss rather than the least-squares loss due to the differences in the task -- we focus on classification instead of recovery problems. On the other hand, the results on low-rank matrix recovery are often based on certain statistical properties, such as the randomness in the measurements \cite{ge2016matrix,bhojanapalli2016global}, or restricted well-conditionedness property of the objective function~\cite{zhu2018global,li2019non}. In contrast, these statistical properties do not exist in our problem, where the model and analysis are purely deterministic.} Therefore, our result establishes a connection between the study of low-rank matrix factorization and neural networks under the unconstrained feature model.
    \item \textbf{Comparison to Existing Landscape Analysis on Neural Network.} \Cref{sec:intro} provided a comprehensive discussion on the relationship between our result and previous works on landscape analysis for deep neural networks. Although the unconstrained feature model can be viewed as a two-layer linear network with input being the columns of an identity matrix, as preluded in \Cref{sec:intro}, our result has much broader implications than the previous results \cite{baldi1989neural,kawaguchi2016deep,nouiehed2018learning,yun2018global,laurent2018deep,kunin2019loss,zhu2019global}. First, our problem formulation \eqref{eq:obj} is closer to practical settings for classification tasks, which considers the widely adopted cross-entropy loss while including weight decay and a bias term, while most existing results~\cite{baldi1989neural,kawaguchi2016deep,nouiehed2018learning,yun2018global,laurent2018deep,kunin2019loss,zhu2019global} either do not incorporate any regularization and bias, or focus on the squared loss for the regression problem. More importantly, our result characterizes the precise form of the global solutions (i.e., \NC) for the last layer features and classifiers, and shows that they can be efficiently achieved. Moreover, convincing numerical results in \cite{papyan2020prevalence} and the next section demonstrate that the global solutions do appear and can be achieved by practical networks on various standard image datasets. %\js{I'm a bit confused by the previous sentence: do you mean that that this has been extensively studied as in Papyan et al, or that this is what we will show in the next section?} 
    Our study of last-layer features could have profound implications for studying generalization and robustness of the deep networks. %In comparison, most of the existing results only show the loss function can be efficiently optimized to zero under overparamterized models, that these results have limited implications for practice. \js{I think we could remove the last three sentences, as we have already mentioned this enough (seems a bit repetitive).}

  %  in comparison to ours, the settings of the existing work~\cite{baldi1989neural,kawaguchi2016deep,nouiehed2018learning,yun2018global,laurent2018deep,kunin2019loss,zhu2019global} are far from practical interests, in the sense that  rather than the cross entropy loss we considered here. 
    
  %  The unconstrained feature model also highly relates to shallow linear neural network with input being the columns of the identity matrix. However, most of the existing work  \cite{baldi1989neural,kawaguchi2016deep,nouiehed2018learning,yun2018global,laurent2018deep,kunin2019loss,zhu2019global} on the global landscape analysis for shallow linear neural network either focus on the regression problem with the least-squares loss function, or do not incorporate any regularizer and bias term. Our result may shed lights on analyzing the landscape for shallow linear neural network with cross-entropy loss and weight decay regularizer, and we leave this as the subject of future work. 
\end{itemize}

 %Instead, we exploit the fact that the feature dimension is larger than the number of classes, and construct the negative curvature direction within the null space of $\mW$ for all the critical points that are not global solutions. We note that this is also the main reason for the requirement $d>K$ in \Cref{thm:global-geometry}, but we conjecture the results also hold for $d= k$ and can be proved with more sophisticated analysis, which is left as future work. %\zz{With this descrption, we may change the assumption in \Cref{thm:global-minima} to $d\ge K$?} 

%    because previous results are focused on recovery problems while we are solving classification tasks, the choice of loss functions is quite different -- they are based on the least-squares rather than the cross entropy loss we studied here. Additionally, they often studied the problem under certain statistical properties\footnote{} that 

%difference from shallow two-layer network, connections to nonconvex low-rank matrix recovery \zz{I think this can be discussed when we formally introduce the peel-layered model.} \qq{good point}

\section{Experiments}\label{sec:experiment}

% \begin{itemize}
%     \item dataset: MNIST, CIFAR 10, random data
%     \item Common neural networks
%     \item verify: algorithm independent: SGD, Adam, LBFGS
%     \item add new experiments with weight decay on the features
%     \item compare the test results with and without the bias term. Our theorem indicate there is no need to have the bias term in the last layer.
% \end{itemize}

% Figures to be added:
% \begin{itemize}
%     \item one figure for feature collapse on training data, one figure for bias term, one figure for $\mW$ close to ETF with the following measure
%     \begin{align}
%   \norm{ \frac{1}{w_t^{2}} \mW_t \mW_t^{\top}\;-\; \frac{K}{K-1}  \paren{ \mb I_K - \frac{1}{K} \mb 1_K \mb 1_K^\top } }{F}/\norm{\frac{K}{K-1}  \paren{ \mb I_K - \frac{1}{K} \mb 1_K \mb 1_K^\top }}{F}.
% \end{align}
% where $w_t^2 = \norm{\mW_t}{F}^2/K$ (or $w_t = \frac{1}{K}\sum_{k=1}^K \norm{\vw^k}{2}$)
% \item repeat the above for MNIST, CIFAR10, and a random case
% \end{itemize}

% \paragraph{exploration experiments}

% \begin{itemize}
%     \item fix the weights of last layer as Simplex ETF, compare generaliztaion performance on ResNet for CIFAR and ImageNet
%     \item Consider multilayer perception network, test neural collapse for each layer, conjecture: we observe less collapse for more shallow layers.
% \end{itemize}

In this section, we run extensive experiments not only verifying our theoretical discoveries on modern neural networks, but also demonstrating the potential practical benefits of understanding \NC. More specifically, while \Cref{thm:global-geometry} is for the simplified models, 
in \Cref{sec:exp-NC} we run experiments on practical network architectures and show that our analysis of simplified models captures the core \NC\;on practical network architectures that the prevalence of \NC\;is due to the geometry rather than the algorithmic bias, by showing that different type of optimization algorithms (e.g., SGD, Adam, and limited-memory BFGS (LBFGS) \cite{bottou2018optimization}) \emph{all} achieve \NC\ during the terminal phase of training.
%More specifically, in \Cref{sec:exp-NC} we corroborate \Cref{thm:global-geometry} via experiments on practical network architectures, showing that our analysis of simplified models captures the core \NC\; on practical network architectures that the prevalence of \NC\;is due to the geometry rather than the algorithmic bias, by showing that different type of optimization algorithms (e.g., SGD, Adam, and limited-memory BFGS (LBFGS) \cite{bottou2018optimization}) \emph{all} achieve \NC\ during the terminal phase of training. %\js{We should be careful with how this is stated -- we're not "demonstrating Theorem 3.1", because this result is based on a simplified model and does not apply to these real networks.} 
In \Cref{subsec:exp-validity}, we verify the validity of the simplification based on the unconstrained feature model in \Cref{sec:layer-peeled-model}. Moreover, the universality of \NC\ implies that there is no need for training the last-layer classifiers since the weights can be simply fixed as a Simplex ETF throughout the training process. In \Cref{sec:exp-fix-classifier}, we demonstrate that such a strategy achieves essentially the same generalization performance as classical training algorithms, while improving on memory and computation. We begin by describing the basic experimental setup, including the network architectures, evaluation datasets, training procedures, and metrics for measuring \NC.

\paragraph{Setup of Network Architectures, Dataset, and Training.} In \Cref{sec:exp-NC} and \Cref{subsec:exp-validity}, we train the cross-entropy loss \eqref{eqn:dl-ce-loss} on ResNet18 architecture~\cite{he2016deep} on MNIST~\cite{lecun1998mnist} and CIFAR10~\cite{krizhevsky2009learning} datasets for the classical image classification task. Without explicitly mentioning, the images are normalized (channel-wise) by their mean and standard deviation. We include no data augmentation in this section, because our focus is to study the behavior associated with \NC\ instead of obtaining state-of-the-art performance. In \Cref{sec:exp-fix-classifier}, We train the cross-entropy loss \eqref{eqn:dl-ce-loss} on the same ResNet18 architecture on MNIST and a modified version\footnote{Here, we use a modified version because the original ResNet50~\cite{he2016deep} is fine-tuned for ImageNet dataset~\cite{deng2009imagenet} which does not achieve performance on CIFAR10.} of ResNet50~\cite{res50} architecture on CIFAR10 datasets for the classical image classification task. For fair comparisons with the results reported on CIFAR10, we use the same data augmentation in \cite{shah2016deep}. We train the network for 200 epochs with three distinct optimizers: two first-order methods (SGD and Adam) and one second-order method (LBFGS). In particular, we use SGD with momentum $0.9$, Adam with $\beta_1=0.9, \beta_2=0.999$, and LBFGS with a memory size of $10$. The initial learning rates for SGD and Adam are set to $0.05$ and $0.001$, respectively, and decreased by a factor of $10$ for every $40$ epochs. For LBFGS, we use an initial learning rate of $0.1$ and employ a strong Wolfe line-search strategy for subsequent iterations. Except otherwise specified, the weight decay is set to $5\times 10^{-4}$ for all the experiments.
%}

\paragraph{Metrics for Measuring \NC\;During Network Training.} We measure \NC\; for the learned last-layer classifiers and features based on the properties presented in \Cref{sec:intro}. Some of the metrics are similar to those presented in \cite{papyan2020prevalence}. First, we define the global mean and class mean of the last-layer features $\Brac{ \mb h_{k,i} }  $ as 
%\zz{Not sure if it is a good idea to use $\star$ here since $\star$ means the global minimizer in \Cref{thm:global-minima}.}
\begin{align*}
    \mb h_G \;=\; \frac{1}{nK} \sum_{k=1}^K  \sum_{i=1}^n \vh_{k,i},\quad \ol{\mb h}_k \;=\; \frac{1}{n} \sum_{i=1}^n \vh_{k,i} \;(1\leq k\leq K).
\end{align*}

\begin{itemize}[leftmargin=*]
    \item \textbf{\emph{Within-class Variability Collapse for the Learned Features $\mb H$.}} We introduce the within-class and between-class covariance matrices as
\begin{align*}
    \mb \Sigma_W \;:=\; \frac{1}{nK} \sum_{k=1}^K \sum_{i=1}^n \paren{  \mb h_{k,i}  -  \ol{\mb h}_k } \paren{ \mb h_{k,i} - \ol{\mb h}_k }^\top   ,\quad \mb \Sigma_B \;:=\; \frac{1}{K} \sum_{k=1}^K \paren{ \ol{\mb h}_k - \mb h_G } \paren{ \ol{\mb h}_k - \mb h_G }^\top.
\end{align*}
Thus, we can measure the within-class variability collapse by measuring the magnitude of the between-class covariance $\mb \Sigma_B \in \bb R^{d \times d} $ compared to the within-class covariance $\mb \Sigma_W \in \bb R^{d \times d} $ of the learned features via
    \begin{align}\label{eq:NC1}
    \mc {NC}_1\;:=\;\frac{1}{K}\trace\parans{\mSigma_ W\mSigma_B^\dagger},
\end{align}
where $\mSigma_B^\dagger$ denotes the pseudo inverse of $\mSigma_B$.
\item \textbf{\emph{Convergence of the Learned Classifier $\mb W$ to a Simplex ETF.}} For the learned classifier $\mW \in\R^{K\times d}$, we quantify its closeness to a Simplex ETF up to scaling by % \begin{align}\label{eq:NC2}
% \mc {NC}_2\;:=\; \frac{\sqrt{K-1}}{K}  \norm{ \frac{K}{\norm{\mW }{F}^2} \mW \mW^{\top}\;-\; \frac{K}{K-1}  \paren{ \mb I_K - \frac{1}{K} \mb 1_K \mb 1_K^\top } }{F},
% \end{align}
\begin{align}\label{eq:NC2}
\mc {NC}_2\;:=\;  \norm{ \frac{\mW \mW^\top}{\norm{\mW \mW^\top}{F}} \;-\; \frac{1}{\sqrt{K-1}}  \paren{ \mb I_K - \frac{1}{K} \mb 1_K \mb 1_K^\top } }{F},
\end{align}
where we rescale the ETF in \eqref{eqn:simplex-ETF-closeness} so that $\frac{1}{\sqrt{K-1}}  \paren{ \mb I_K - \frac{1}{K} \mb 1_K \mb 1_K^\top } $ has unit energy (in Frobenius norm).
%where $\frac{\sqrt{K}}{\norm{\mW }{F}}\mW$ is the scaled version of $\mW$ such that it has the same energy (in Frobenius norm) with the Simplex ETF, as shown in \Cref{def:simplex-ETF} in the appendix.  
It should be noted that our metric $\mc {NC}_2$ combines two metrics used in \cite{papyan2020prevalence} to quantity to what extent the classifier approaches equiangularity and maximal-angle equiangularity.
\item \textbf{\emph{Convergence to Self-duality.}} Next, we measure the collapse of the learned features $\mb H$ to its dual $\mb W$. Let us define the centered class-mean matrix as
\begin{align*}
    \ol{\mb H} \;:=\; \begin{bmatrix}
\ol{\mb h}_1 - \mb h_G & \cdots & \ol{\mb h}_K - \mb h_G
\end{bmatrix} \in \bb R^{d \times K}.
\end{align*}
Thus, we measure the duality between the classifiers $\mb W$ and the centered class-means $\ol{\mb H}$ by
% \begin{align}\label{eq:NC3}
%     \mc {NC}_3\;:=\; \norm{ \frac{\mW^{\top}}{ \norm{\mW}{F}} - \frac{\ol{\mH} }{ \norm{\ol{\mH}}{F} } }{F}.
% \end{align}
%\zz{Maybe the following term}
\begin{align}\label{eq:NC3}
 \mc {NC}_3\;:=\; \norm{\frac{\mW \ol{\mH}}{\norm{\mW\ol{\mb H}}{F}}   \;-\; \frac{1}{\sqrt{K-1}}  \paren{ \mb I_K - \frac{1}{K} \mb 1_K \mb 1_K^\top} }{F}.
\end{align}
 
\item \textbf{\emph{Collapse of the Bias.}} In many cases, the global mean $\mb h_G$ of the features might not be zero,\footnote{For example, as discussed after \Cref{thm:global-minima} all the feature vectors in $\mb H$ would be nonnegative, because the nonnegative nonlinear operator ReLU has been applied at the end of the penultimate layer.} and the bias term $\vb$ would compensate for the global mean $\mb h_G$ in the sense that 
\begin{align*}
    \mb W \mb h_{k,i} + \mb b\;=\;  \mb W \paren{\mb h_{k,i} - \mb h_G} + \underbrace{ \mb W\mb h_G +  \mb b}_{= \mb 0}.
\end{align*}
Thus, we capture this collapsing phenomenon by measuring 
\begin{align}\label{eq:NC4}
    \mc {NC}_4 \;:=\; \norm{\vb + \mW \vh_G}{2}. 
\end{align}
\end{itemize}

%\js{$\Sigma_W$ is slightly confusing, because it seems to refer to the between-class covariance of the classifier, W, instead of the features. This is somewhat confusing when looking at the results in Sec 4.3 (where W is fixed). } \zz{It should be $\mSigma_W$ with $W$ (not bold) representing ``within-class". But we can change to other variables.} given that $\norm{\frac{K}{K-1}  \paren{ \mb I_K - \frac{1}{K} \mb 1_K \mb 1_K^\top }}{F} = \frac{K}{\sqrt{K-1}}$

%from the following perspectives. Let $\mSigma_W$ be the within-class covariance of the last-layer activations (of the training data), and let $\mSigma_B$ be the corresponding between-class covariance. Following , 

%For any $\mW\in\R^{K\times d}$, we quantify its closeness (the scaled version) to a Simplex ETF by

%where $\frac{\sqrt{K}}{\norm{\mW}{F}}\mW$ is the scaled version of $\mW$ such that it has the same energy (Frobenius norm) as the Simplex ETF, and $\norm{\frac{K}{K-1}  \paren{ \mb I_K - \frac{1}{K} \mb 1_K \mb 1_K^\top }}{F} = \frac{K}{\sqrt{K-1}}$. Our metric, $\mc {NC}_2$, combines the two measures that were used in \cite{papyan2020prevalence} to quantity to what extent the classifier approaches equiangularity and maximal-angle equiangularity. By denoting $\dot{\mH}\in\R^{d\times K}$ the matrix consisting of the centered \js{you mean zero-mean, right?} train class-means, we then measure the distance between the classifiers and the centered class-means as in \cite{papyan2020prevalence} by

% \[
% \|\mW\overline\mH + \vb - ETF\|
% \]

%(or $w_t = \frac{1}{K}\sum_{k=1}^K \norm{\vw^k}{2}$)
\begin{figure}[t]
    \centering
    \subfloat[$\mc {NC}_1$ (MNIST)]{\includegraphics[width=0.24\textwidth]{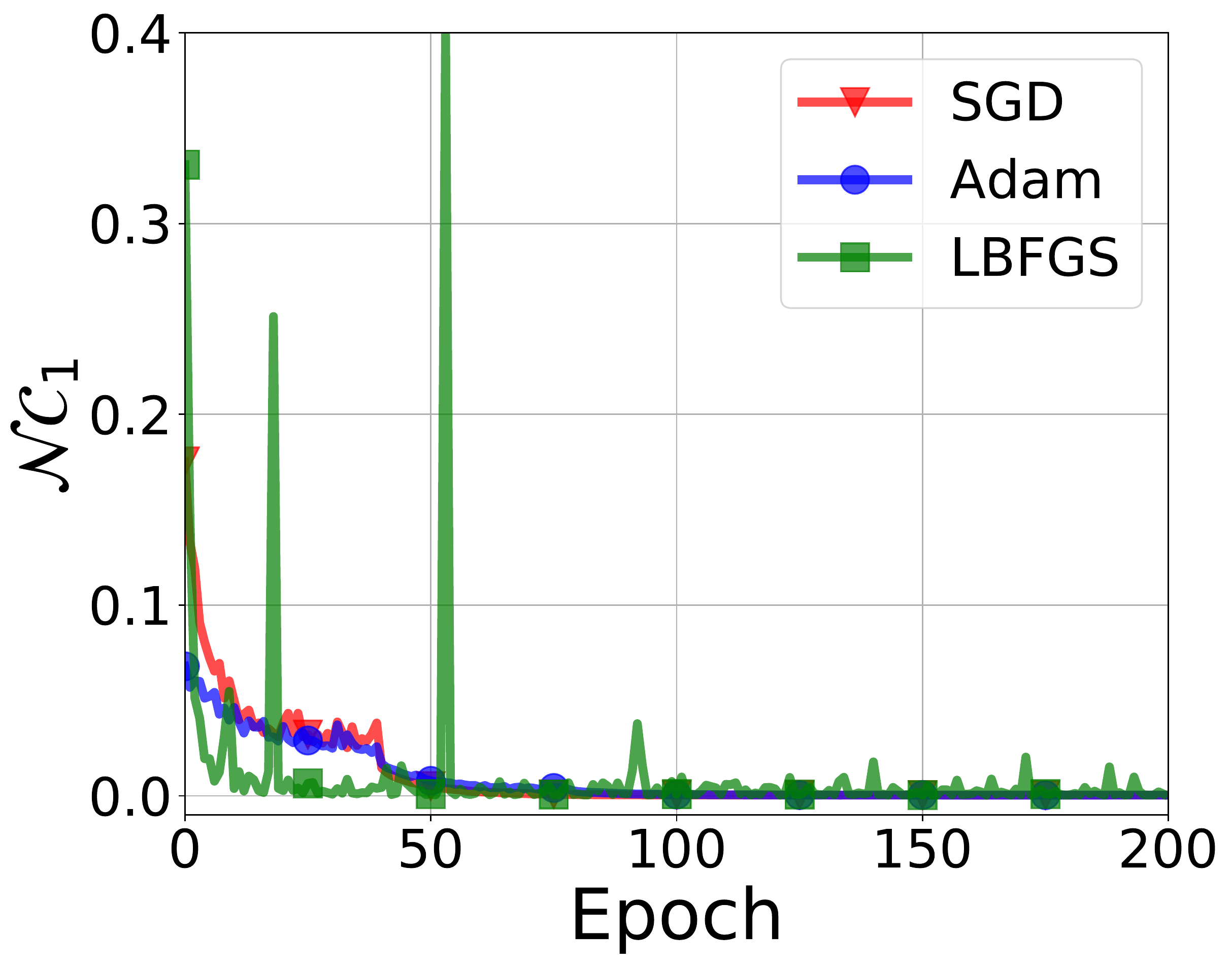}} \
    \subfloat[$\mc {NC}_2$ (MNIST)]{\includegraphics[width=0.24\textwidth]{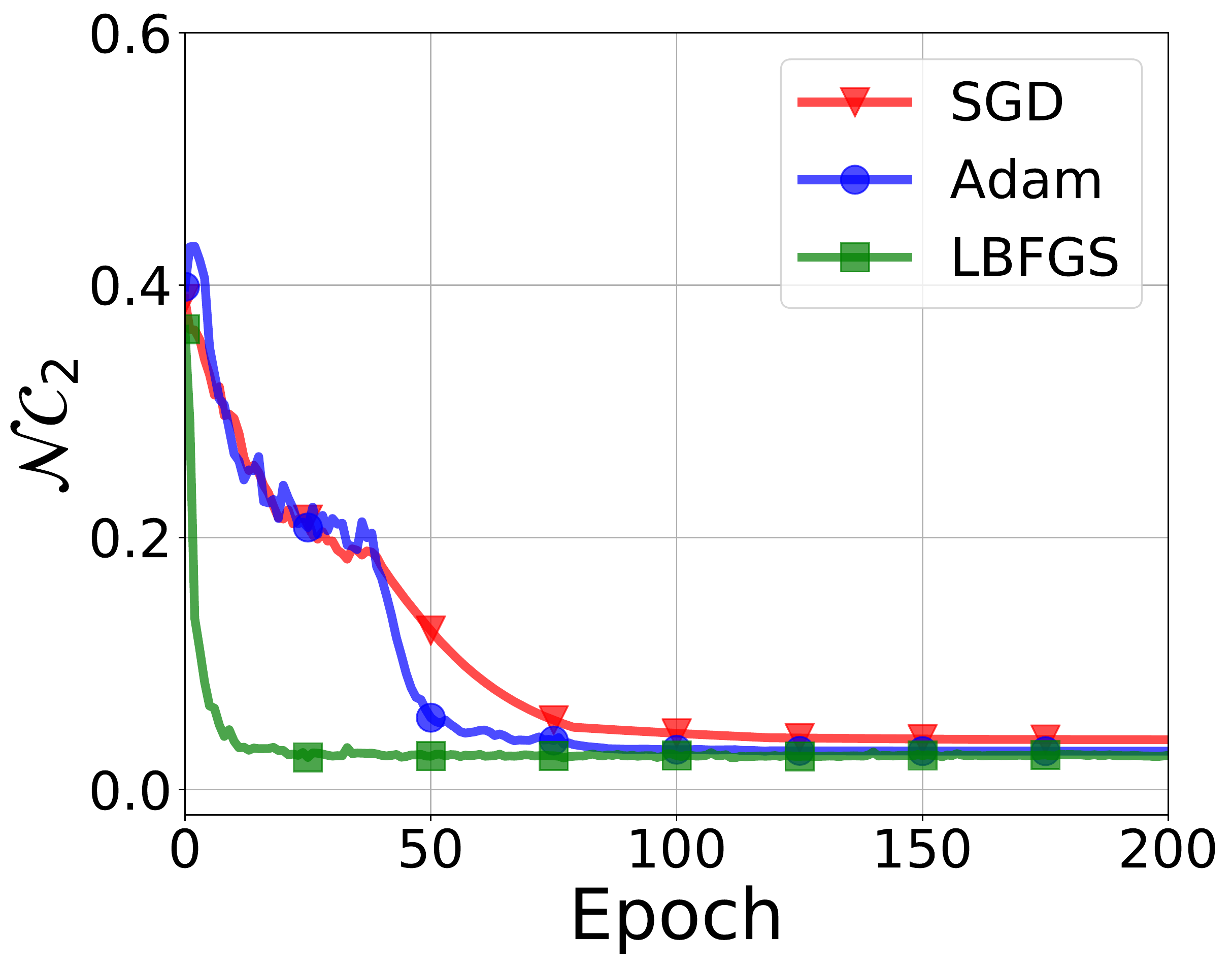}} \
    \subfloat[$\mc {NC}_3$ (MNIST)]{\includegraphics[width=0.24\textwidth]{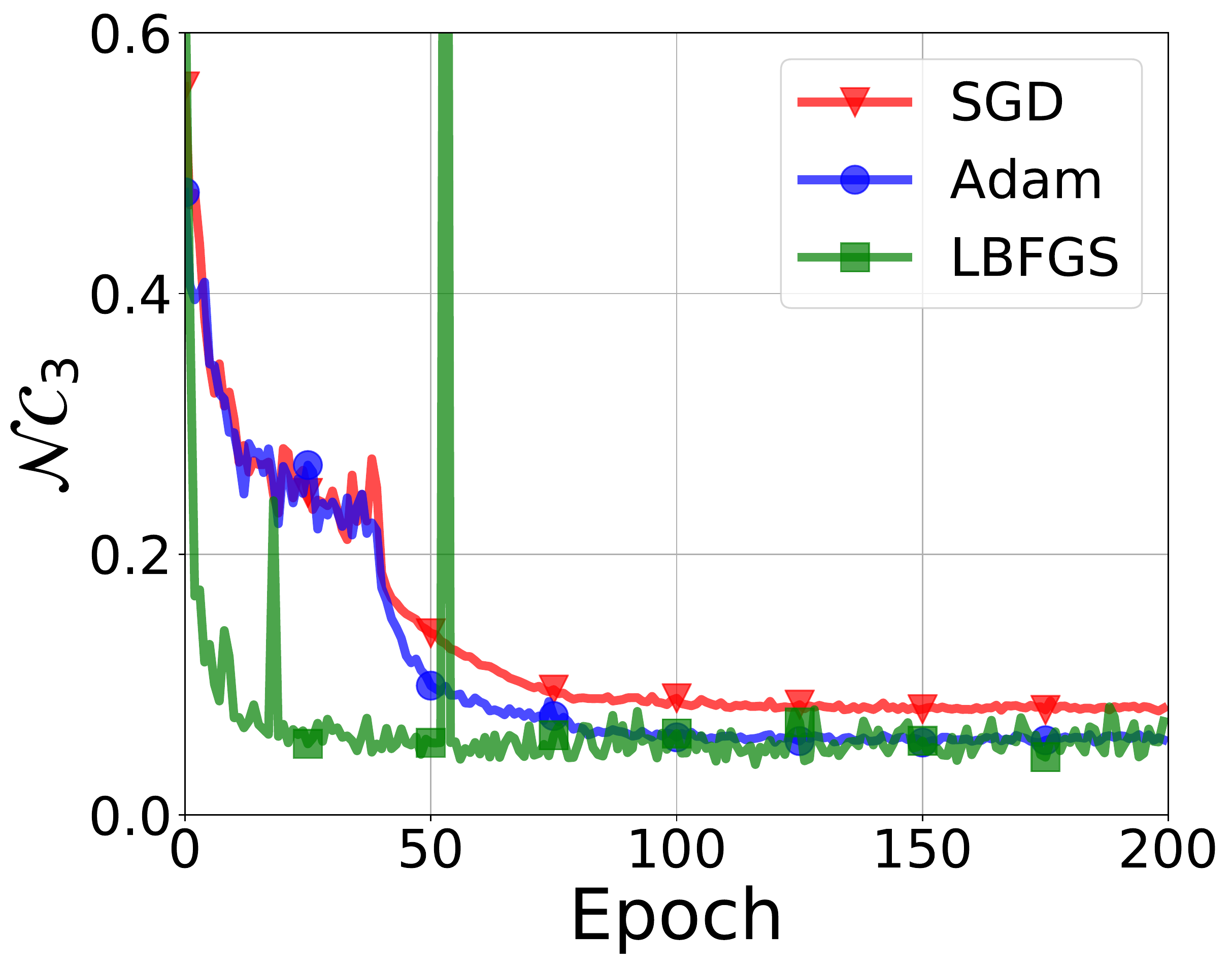}} \
    \subfloat[$\mc {NC}_4$ (MNIST)]{\includegraphics[width=0.24\textwidth]{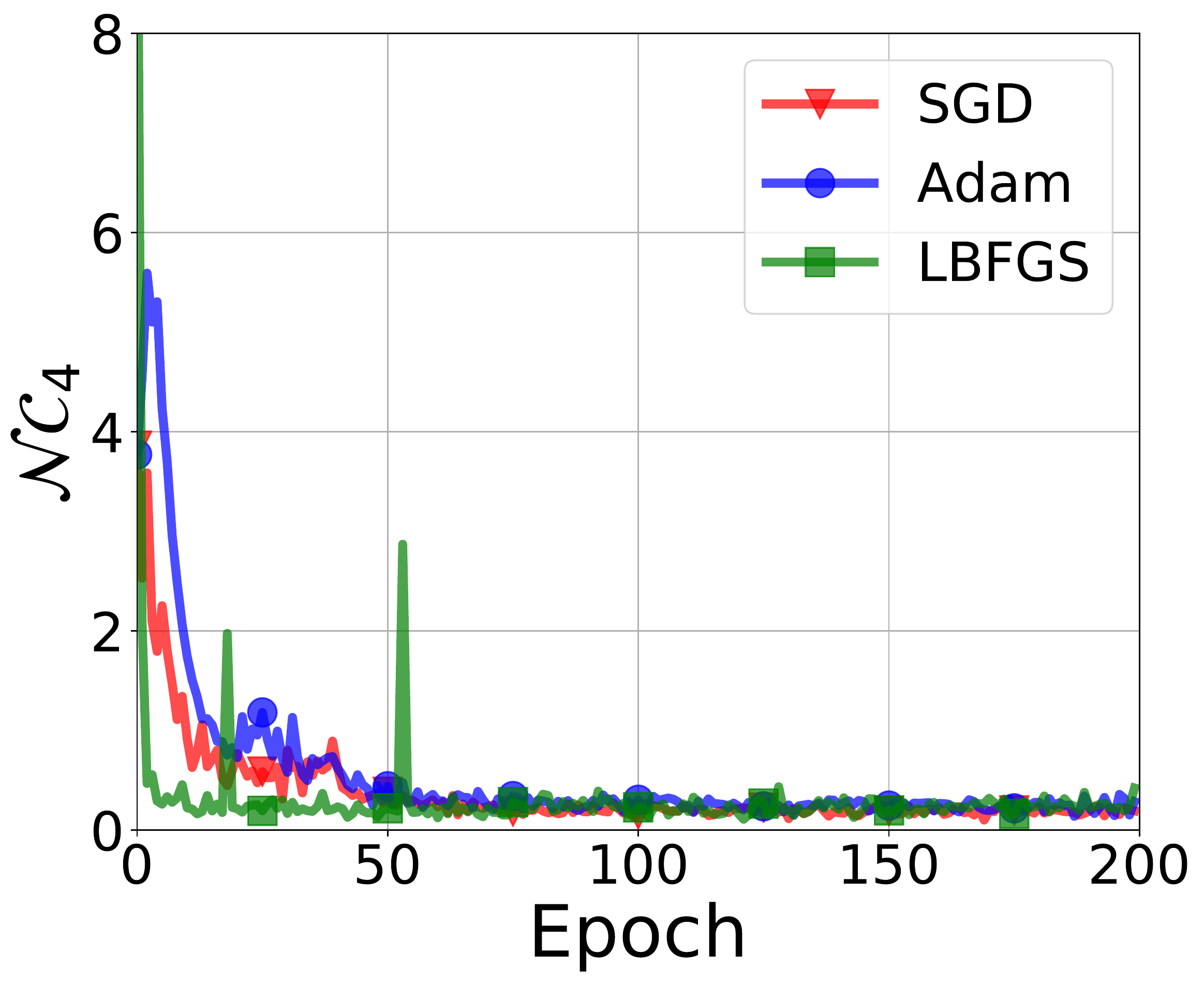}} \\
    \subfloat[$\mc {NC}_1$ (CIFAR10)]{\includegraphics[width=0.24\textwidth]{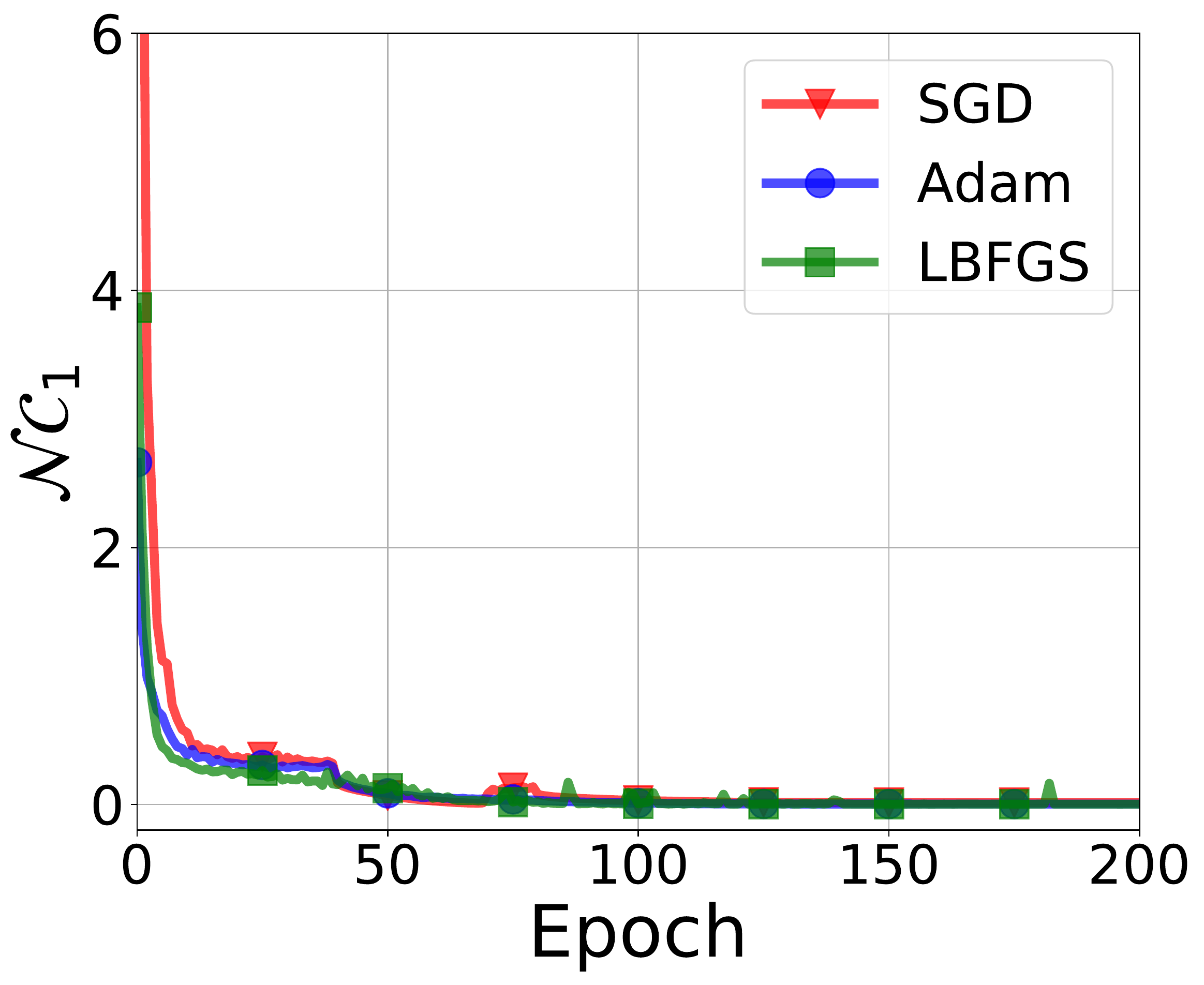}} \
    \subfloat[$\mc {NC}_2$ (CIFAR10)]{\includegraphics[width=0.24\textwidth]{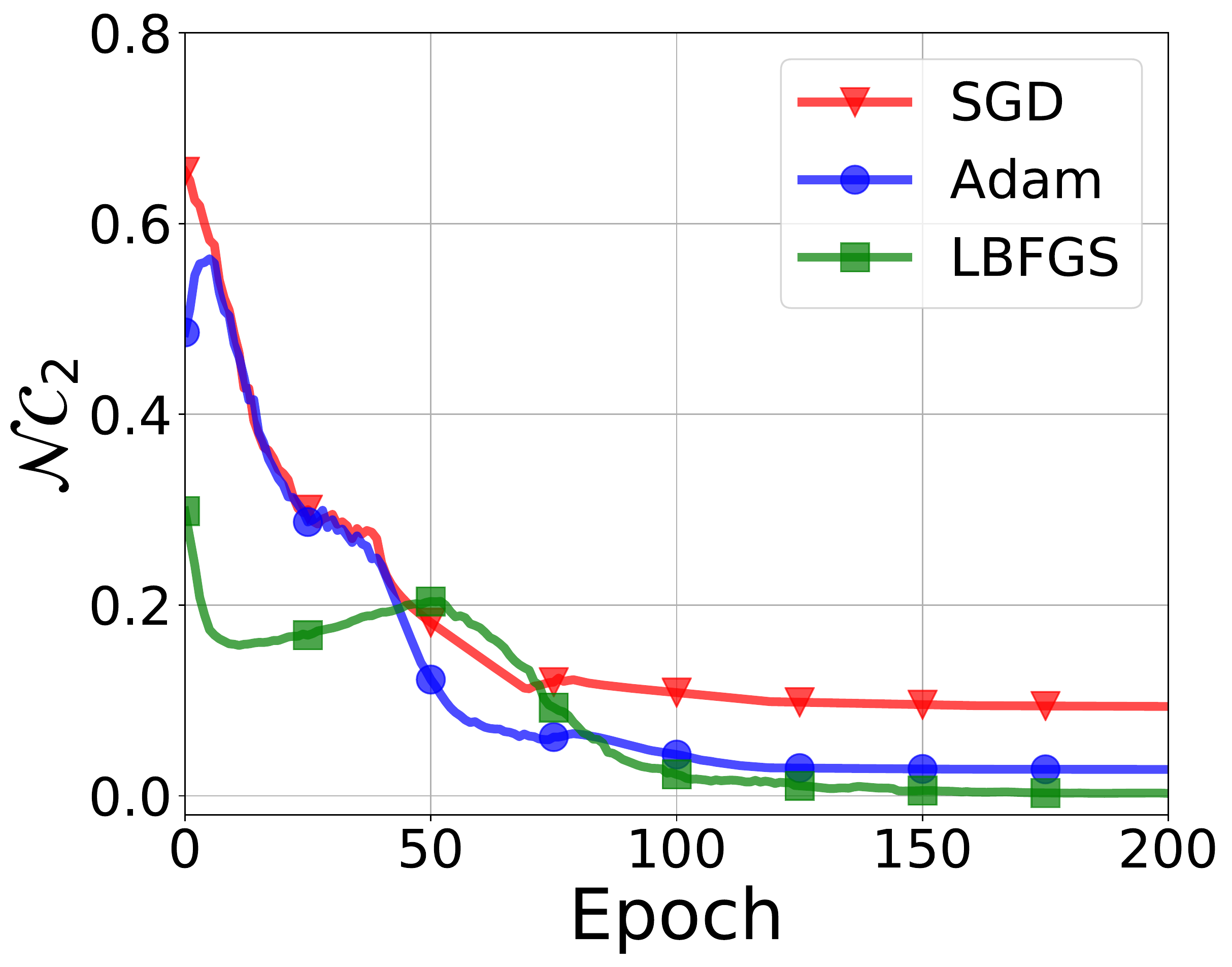}} \
    \subfloat[$\mc {NC}_3$ (CIFAR10)]{\includegraphics[width=0.24\textwidth]{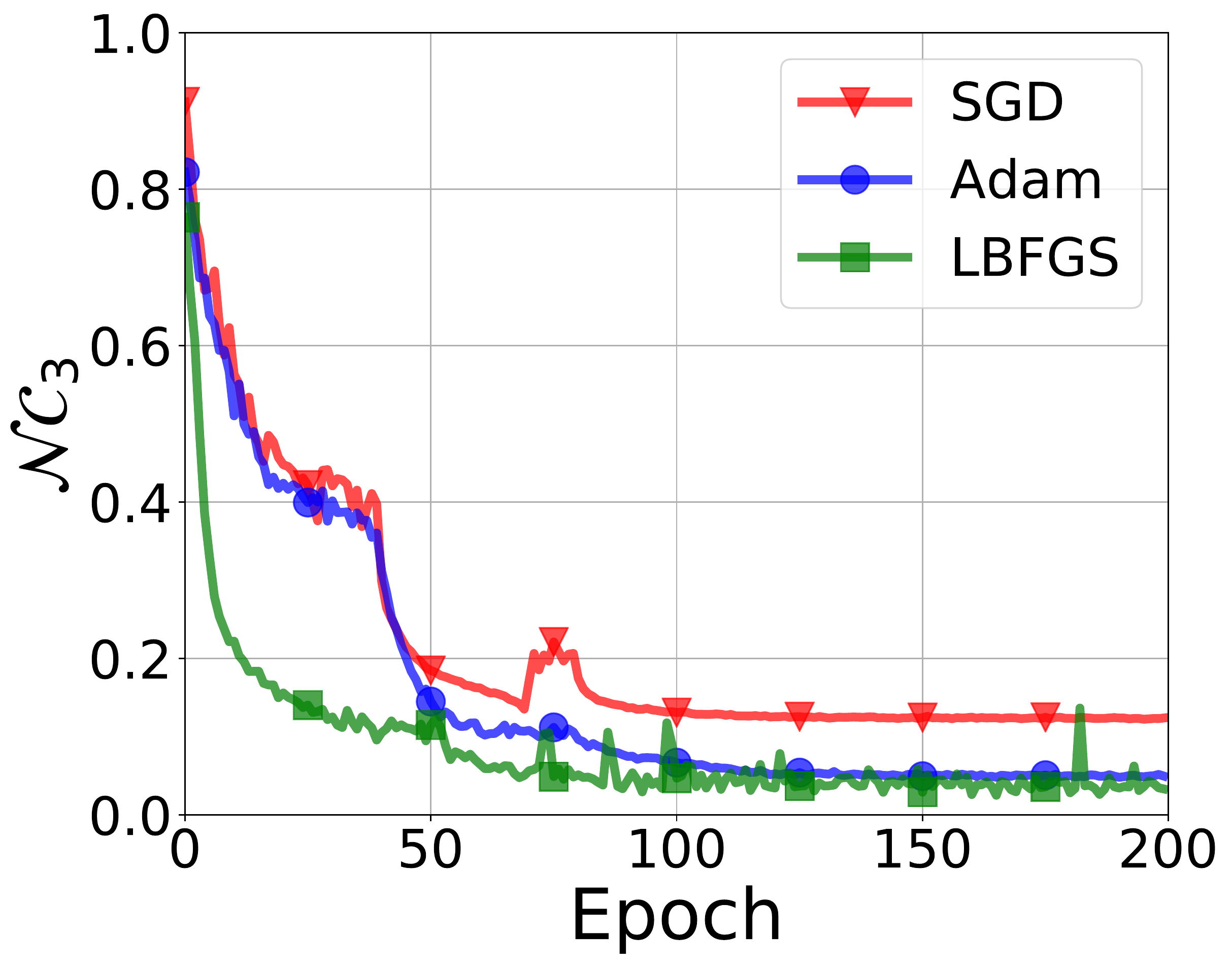}} \
    \subfloat[$\mc {NC}_4$ (CIFAR10)]{\includegraphics[width=0.24\textwidth]{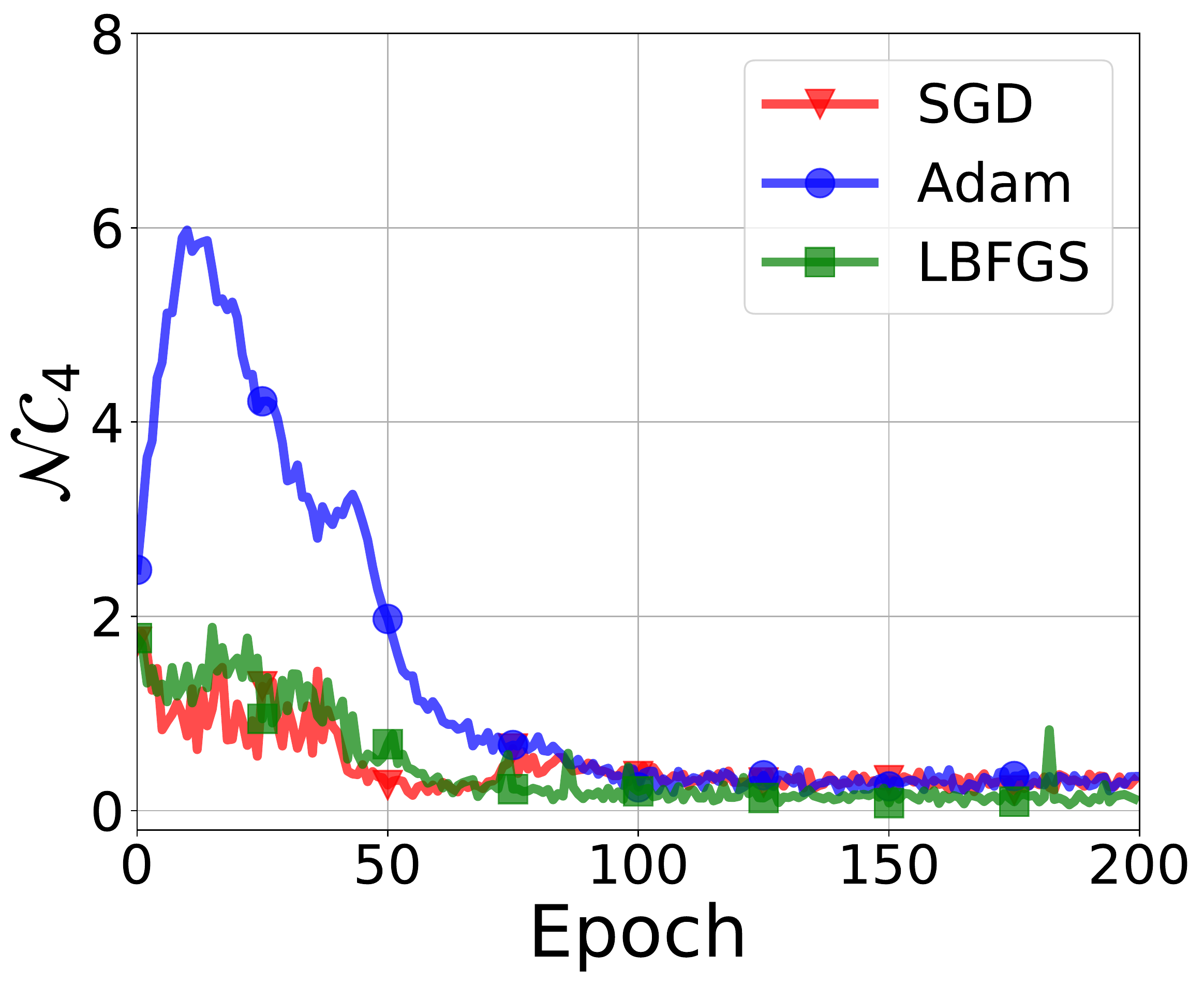}}
    \caption{\textbf{Illustration of \NC\;across different training algorithms} with ResNet18 on MNIST (top) and CIFAR10 (bottom). From the left to the right, the plots show the four metrics, $\mc {NC}_1, \mc {NC}_2, \mc {NC}_3$, and $\mc {NC}_4$, for measuring \NC,  defined in \eqref{eq:NC1}, \eqref{eq:NC2}, \eqref{eq:NC3}, and \eqref{eq:NC4}, respectively. %\zz{To add the figure for $\mc {NC}_3$. Also, is it possible to change the Ylabel to $\mc {NC}_1, \mc {NC}_2, \mc {NC}_3$, and $\mc {NC}_4$? Maybe it's better to use the same line styles (solid lines) for all the figures, if the vision quality is not that bad.} \revise{Tianyu: One figure to illustrate the relation between W and H will be added later. Jinxin tried a variant of LBFGS, which may be added later. } 
    }
    \label{fig:NC-MNIST-CIFAR10}
\end{figure}

\subsection{The Prevalence of \NC\;Across Different Optimization Algorithms}\label{sec:exp-NC}

%one figure to illustrate the relation between $W$ and $H$. Add the figure for $\|\vb + W\alpha\|_2$ 
%Add a figure for SGD with different learning rates 

Our result in \Cref{thm:global-geometry} reveals a benign global landscape for the optimization of neural networks under the unconstrained feature model, which further implies that the prevalence of \NC\ is independent of the choice of particular training methods. In this subsection, we validate our theoretical discovery on modern network architectures and standard datasets. We show different types of training methods (e.g., SGD, Adam, and LBFGS) all achieve \NC\ during the terminal phase of training. In \Cref{fig:NC-MNIST-CIFAR10} and \Cref{fig:Accuracy-MNIST-CIFAR10}, we run all the experiments with ResNet18 on MNIST and CIFAR10 without modification. The results lead to the following observations:
\begin{itemize}[leftmargin=*]
    \item {\textbf{\emph{\NC\;is Algorithm Independent}.}} 
    %\js{why the 'almost'? ie from these results, there's no clear difference between them, right?} 
    \Cref{fig:NC-MNIST-CIFAR10} shows the evolution of the four metrics $\mc {NC}_1, \mc {NC}_2, \mc {NC}_3$, and $\mc {NC}_4$, for measuring \NC\ as training progresses, defined in \eqref{eq:NC1}, \eqref{eq:NC2}, \eqref{eq:NC3}, and \eqref{eq:NC4}, respectively. We consistently observe that all four metrics collapse to zero, trained by different types of algorithms -- SGD, Adam, and LBFGS. This implies that \NC\ happens regardless of the choice of training methods. The last-layer features learned by the network are always maximally linearly separable, and correspondingly the last-layer classifier is a perfect linear classifier for the features.
    
  %  These results imply that the last-layer features learned by the network are maximally linearly separable, and the last-layer classifer  the last-layer classifiers of neural networks always converge These results imply that neural networks always converge to a classifier in which the peeled-off layers learn maximally separated features, and the final layer learns a max-margin linear classifier for the learned features.
    %from the first column that the within-class variation collapses, from the second column that the weights of the classifier converge to a Simplex ETF, from the third column that the classifiers converge to the training class-means, and from the forth column that the bias compensates for the global mean of the features. Moreover, such neural collapse phenomena occur for all the three learning algorithms, SGD, Adam and LBFGS. These results illustrate the fact that neural networks always converge to a classifier in which the peeled-off layers learn maximally separated features, and the final layer learns a max-margin linear classifier for the learned features.
    \item {\textbf{\emph{Relationship between \NC\;and Generalization.}}} %The results in \Cref{fig:NC-MNIST-CIFAR10} manifest the fact that neural networks always converge to a perfect linear classifier in which the peeled-off layers learn maximally separated features and the final layer learns a maximal margin linear classifier for the learned features.
    \Cref{fig:Accuracy-MNIST-CIFAR10} depicts the learning curves in terms of both the training and test accuracy for all three optimization algorithms (i.e., SGD, Adam, and LBFGS). These experimental results\footnote{Note that here we use the default version of LBFGS in PyTorch. %\zz{Please verify if this is correct}. 
    Other variants of quasi-Newton methods \cite{bollapragada2018progressive,ren2021kronecker} may give different or better generalization performance.} show that different training algorithms learn neural networks with notably different generalization performances, even though all of them exhibit \NC. Since \NC\;is only a characterization of the training data, it does not directly translate to unseen data. As the network is highly overparameterized, there are infinitely many networks that produce the same $\mb H$ with \NC\ for a particular training dataset, but with different generalization performance. This suggests that study generalization needs to consider the algorithmic bias and the learned weights for the feature $\mb H$. A thorough investigation between \NC\; and generalization is the subject of future work.
    %Therefore, the universality of \NC\;actually shows that such a phenomenon cannot fully explain network generalization. A pertinent study of generalization will necessarily require a scrutiny into the peeled-off layers, where different optimization method imposes different implicit bias on the learned network parameters.
\end{itemize}

%First, we corroborate \Cref{thm:global-geometry} 

%we confirm the \emph{prevalence} of \NC\;in network training, which is independent of 

%Our result in \Cref{thm:global-geometry} reveals a benign global landscape for the optimization of neural networks under the unconstrained feature model. This benign landscape implies that neural collapse is \emph{universal}, in the sense that it occurs not only with gradient descent but also with other local search algorithms in general. %This is in sharp contrast to work on implicit algorithmic bias, which emphasizes the pivotal role of optimization algorithm in determining the solution that the network converges to. In the following experiment, we validate the universality of neural collapse by showing that \NC~ indeed occurs for networks optimized with SGD, Adam and LBFGS. %We first use the original datasets MNIST and CIFAR10 without modification, and we train a ResNet18 with the three different learning algorithms mentioned above. The results lead us to make the following observations:

%\paragraph{True Labels} 

\begin{figure}[t]
    \centering
    \subfloat[MNIST: Training (left) vs. Testing (right)]{\includegraphics[width=0.24\textwidth]{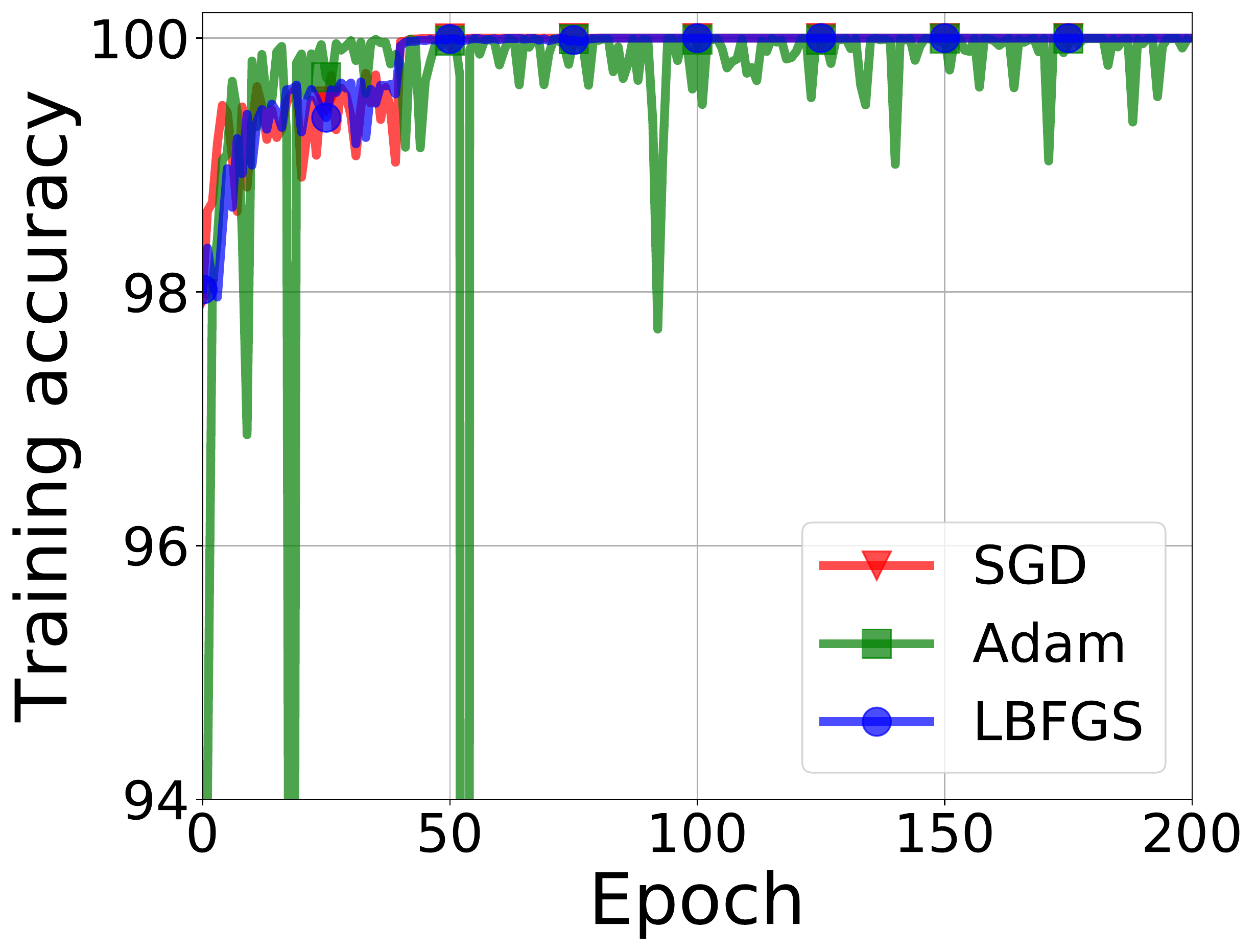} \includegraphics[width=0.24\textwidth]{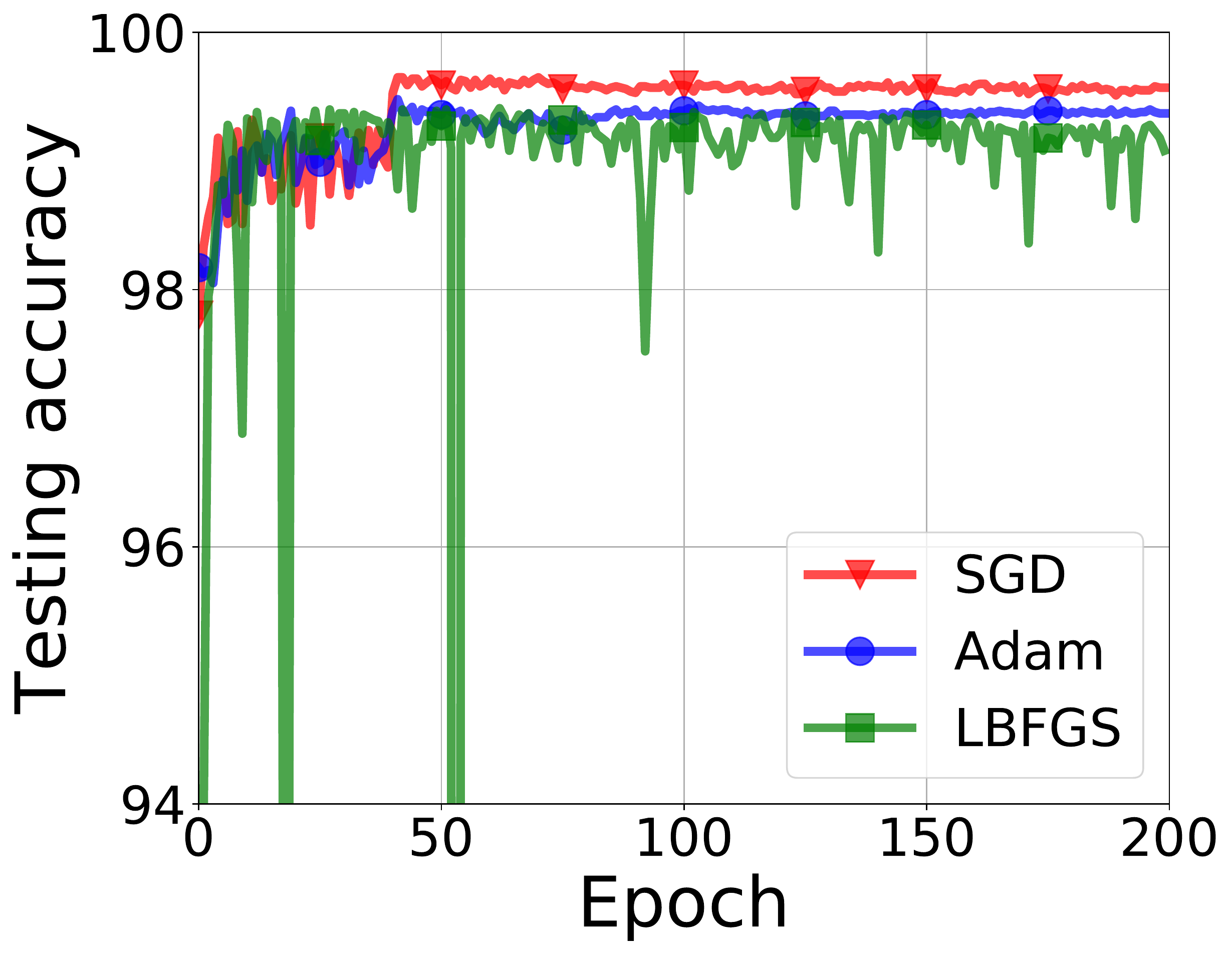}} \
    \subfloat[CIFAR10: Training (left) vs. Testing (right)]{\includegraphics[width=0.24\textwidth]{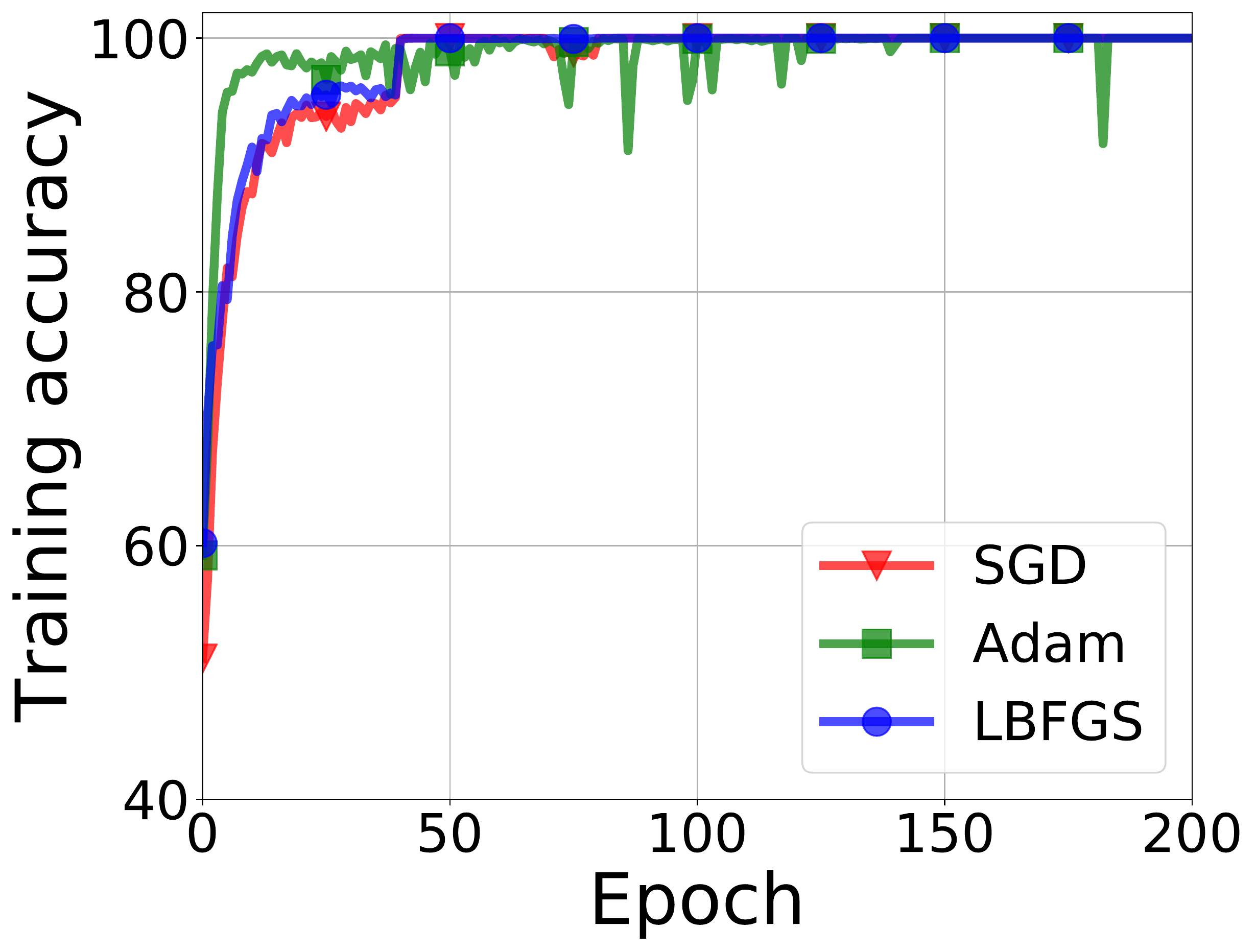}\includegraphics[width=0.24\textwidth]{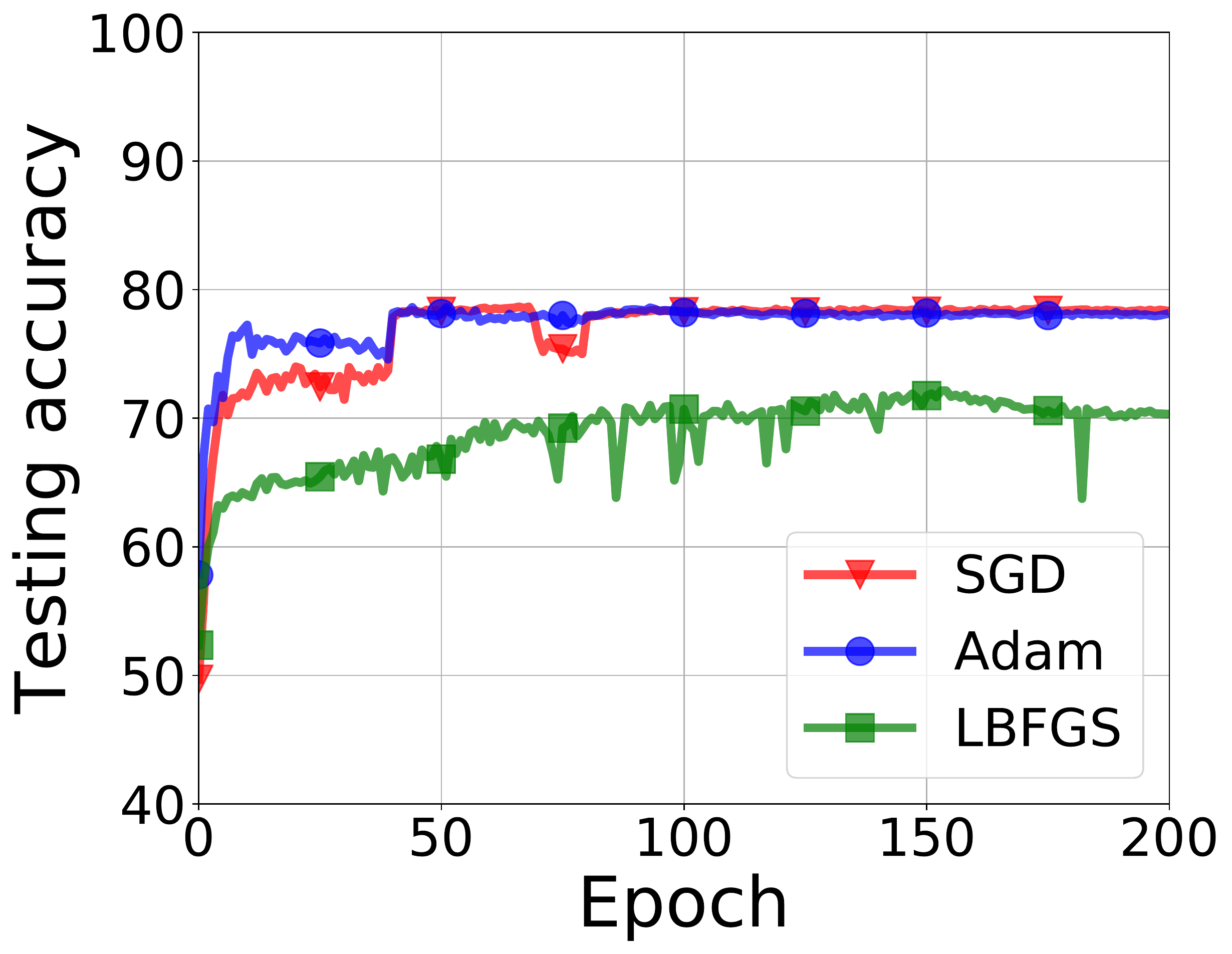}}
    \caption{\textbf{Illustrations of training and test accuracy} for three different training algorithms (i.e., SGD, Adam, and LBFGS) with ResNet18 on MNIST and CIFAR10.  %\zz{Also add the training accuracy to illustrate all the algorithms achieve (almost) zero training error?}
    %\zz{Remove the plot for $Std(b_k)$.}%\zz{Does this case have the ReLu at the end of the penultimate layer? If yes, for the bias term, could you also plot $\|\vb + \mW\valpha\|_2^2$, where $\valpha$ is the global mean of the entire features.}
    }
    \label{fig:Accuracy-MNIST-CIFAR10}
\end{figure}

\subsection{The Validity of \eqref{eq:obj} Based on Unconstrained Feature Models for \NC}\label{subsec:exp-validity}

%\paragraph{Random Labels.} 
The premise of our global landscape analysis of \eqref{eq:obj} for studying \NC\ in deep neural networks is based upon the unconstrained feature model introduced in \Cref{sec:layer-peeled-model}, which simplifies the network by synthesizing the first $L-1$ layers as a universal approximator that generates a simple decision variable for each training sample. Here, we demonstrate through experiments that such a simplification is reasonable for overparameterized networks, at least sufficient for characterizing \NC\ in practical network training.

\paragraph{The Validity of Unconstrained Feature Models.} First, we demonstrate that overparameterization is crucial for \NC\ phenomenon during network training, while the input plays minimal influence. These provide numerical supports on treating $\mb H$ as a free optimization variable for studying \NC. In particular, we modify the training dataset CIFAR10 by replacing {\em all} the correct label for each training sample with a \emph{random} counterpart. 
For the ease of studying the effects of model sizes (i.e., overparameterization) on \NC, besides ResNet18 networks\footnote{Here, for ResNet18 we adopt the method in \cite{yang2020rethinking} to change its network width.}, we train 4-layer multilayer perceptrons (MLP) of different network width using SGD with learning rate $0.01$ and weight decay $10^{-4}$. We report the corresponding \NC\ behaviors in \Cref{fig:random_label}, which shows how training misclassification rate and \NC\ evolve over epochs of training for networks with different widths. As the network is sufficiently large, it has enough capacity to memorize the training data and achieves zero training error, which is consistent with the observations in \cite{zhang2016understanding}. Moreover, we find from \Cref{fig:random_label} that the training accuracy is highly correlated with \NC\ in the sense that a larger network (i.e., larger width) tends to exhibit severe \NC\ and achieves smaller training error. In other words, while the emerging consensus is that the network can memorize any training data, our results show that such memorization happens in a particular way -- the features are maximally separated, followed by a max-margin linear classifier.

%\footnote{We also conducted experiments on ResNet18, where the experimental observations are similar and are left in the appendix.}
%Also, to control and study the effect of model sizes on \NC, we use here a multilayer perceptron (MLP) and ResNet classifiers. \js{I'm confused: I thought that these were going to explore results on both MLP and ResNet, but in Fig 6 I only see MLP} In this part, in contrast to the the previous experiments (that involved true labels), we choose to use SGD as an optimizer and a small weight decay term of $1e-4$, which converges faster than other optimization algorithms. We modify the network width and depth and report the corresponding  \NC\ behaviors in \Cref{fig:random_label}. %\zz{To describe the detail of the MLP.} \js{Figure 6 appears before Fig 5}

\begin{figure}[t]
    \centering
    % \subfloat[$\mc {NC}_1$ (MLP)]{\includegraphics[width=0.24\textwidth]{figs/random_label/rc_collapse_NC1.pdf}} \
    % \subfloat[$\mc {NC}_2$ (MLP)]{\includegraphics[width=0.24\textwidth]{figs/random_label/rc_ETF_NC2.pdf}} \
    % \subfloat[$\mc {NC}_3$ (MLP)]{\includegraphics[width=0.24\textwidth]{figs/random_label/rc_W-H_NC3.pdf}} \
    % \subfloat[Training Error]{\includegraphics[width=0.24\textwidth]{figs/random_label/rc_train_error.pdf}} \\
    \subfloat[Random Label CIFAR10-MLP (from left to right): $\mc {NC}_1$ (log scale), $\mc {NC}_2$, $\mc {NC}_3$, Training Error Rate]
        {\includegraphics[width=0.24\textwidth]{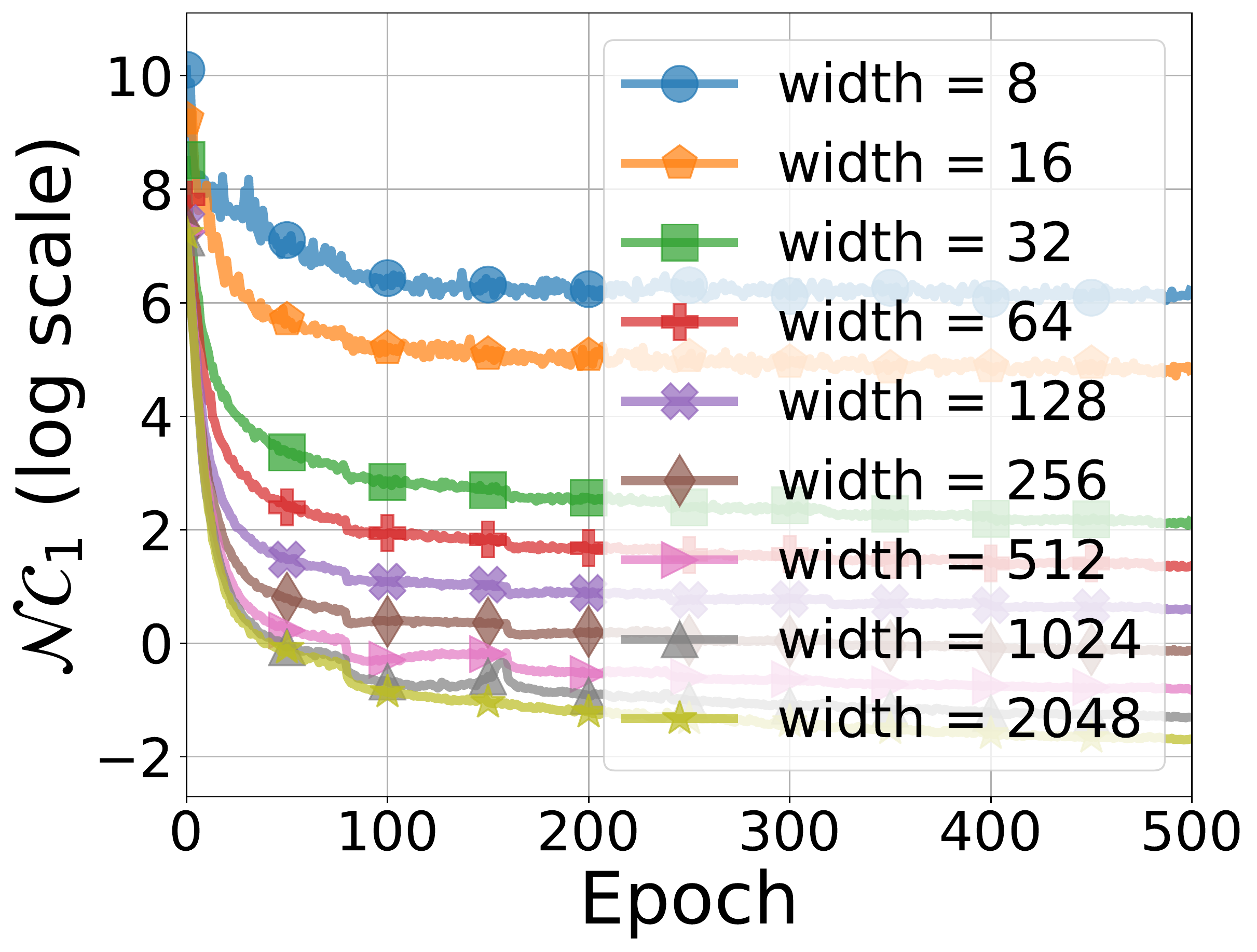} \
        \includegraphics[width=0.24\textwidth]{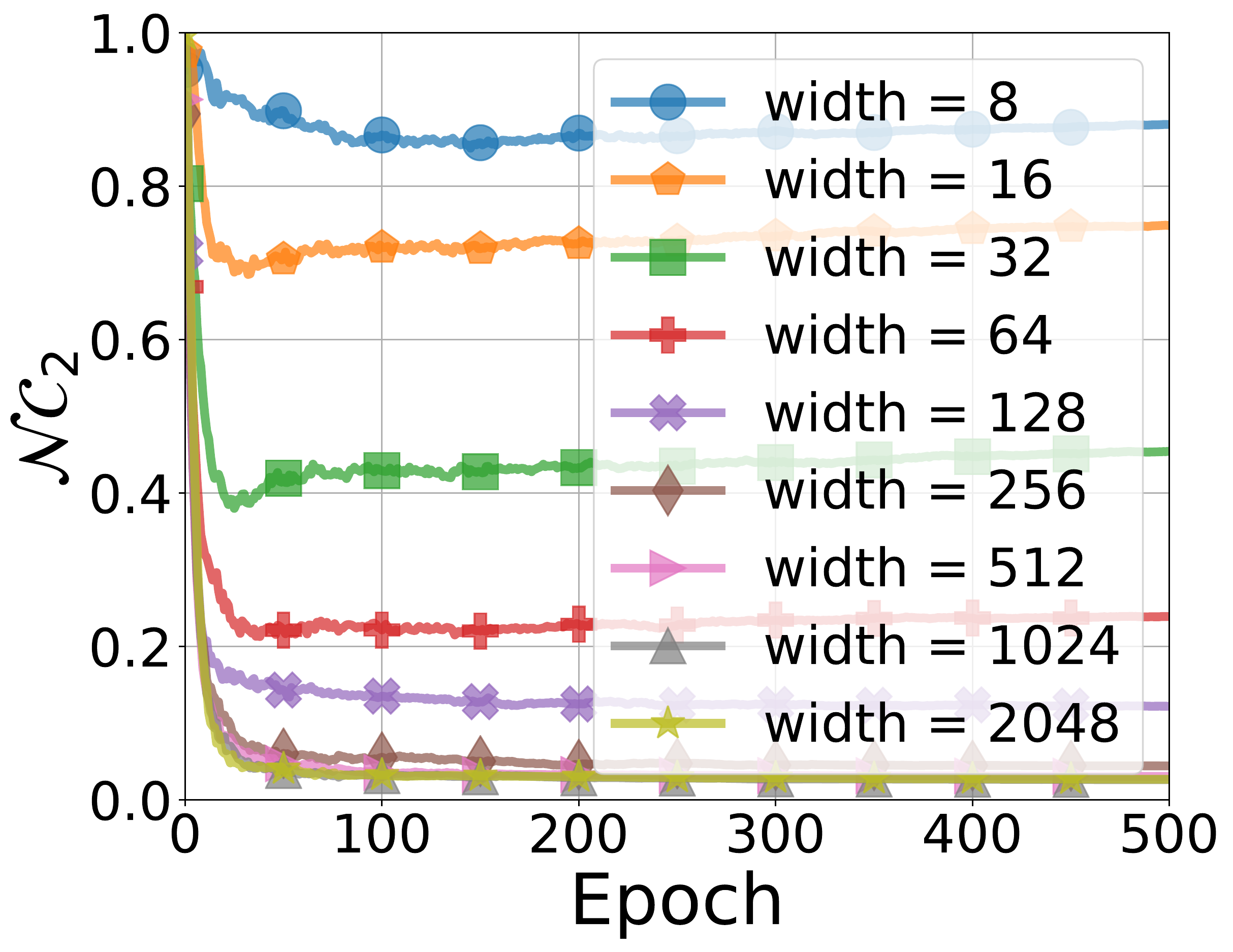} \
        \includegraphics[width=0.24\textwidth]{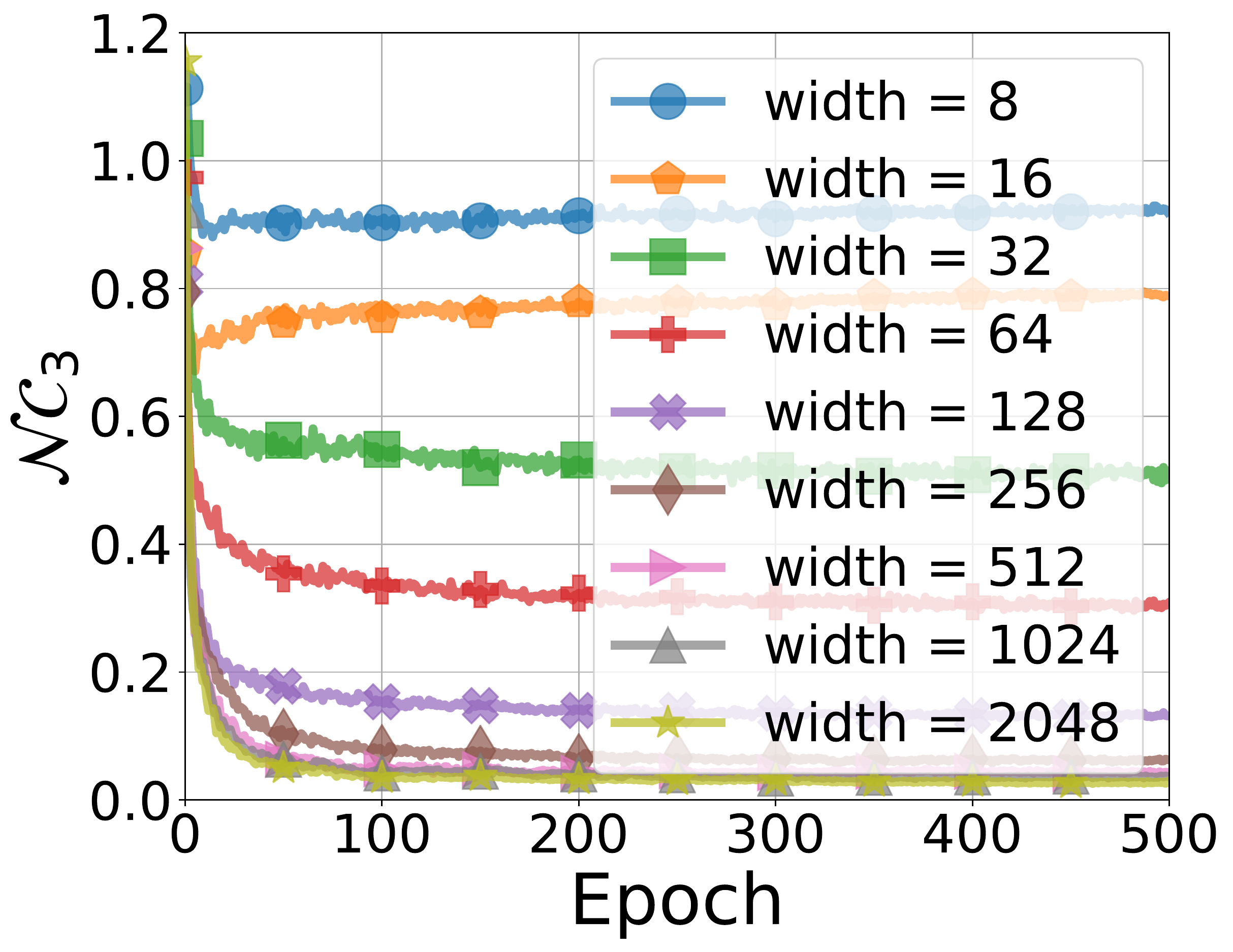} \
        \includegraphics[width=0.24\textwidth]{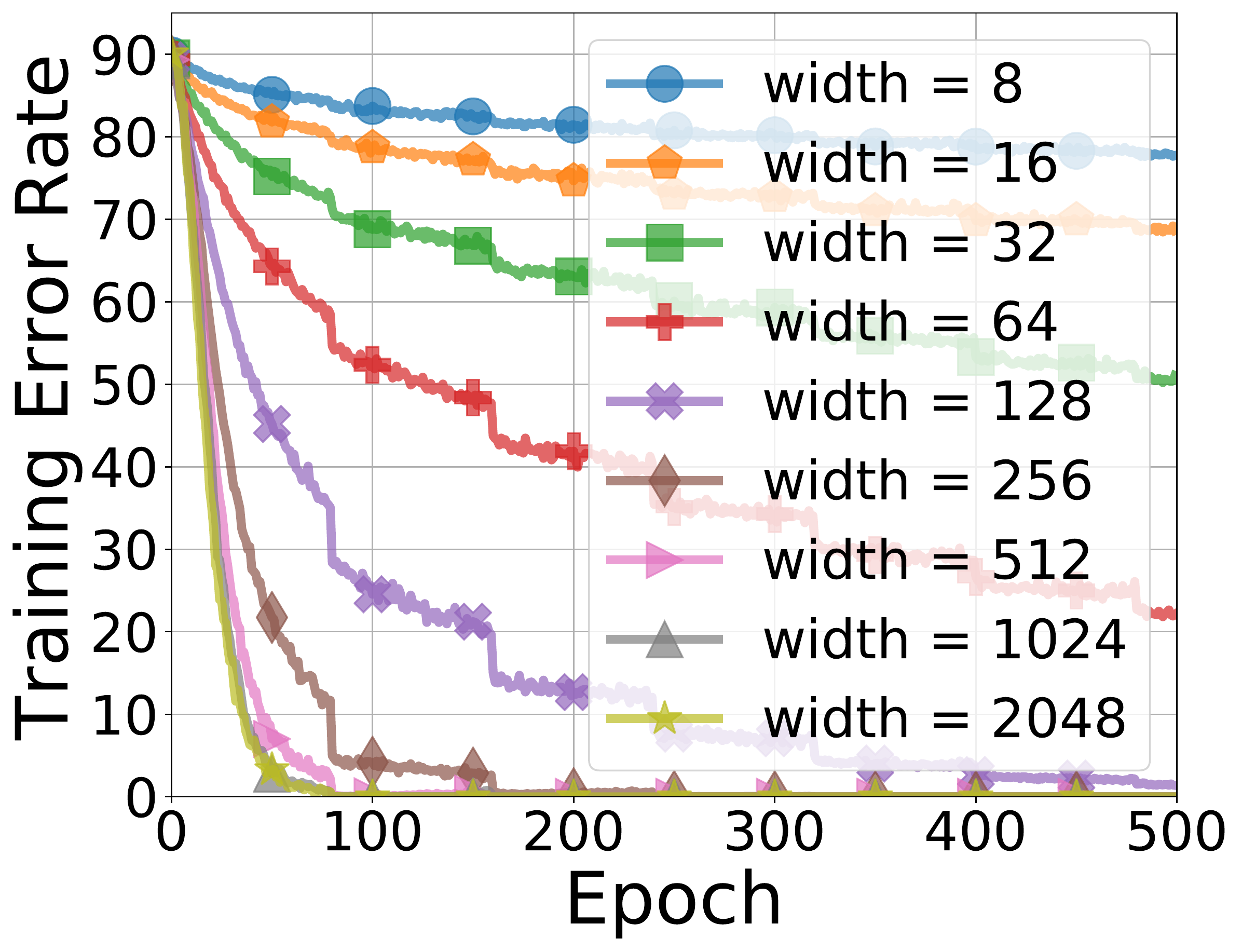}} \\
    
    \subfloat[Random Label CIFAR10-ResNet18 (from left to right): $\mc {NC}_1$ (log scale), $\mc {NC}_2$, $\mc {NC}_3$, Training Error Rate]
        {\includegraphics[width=0.24\textwidth]{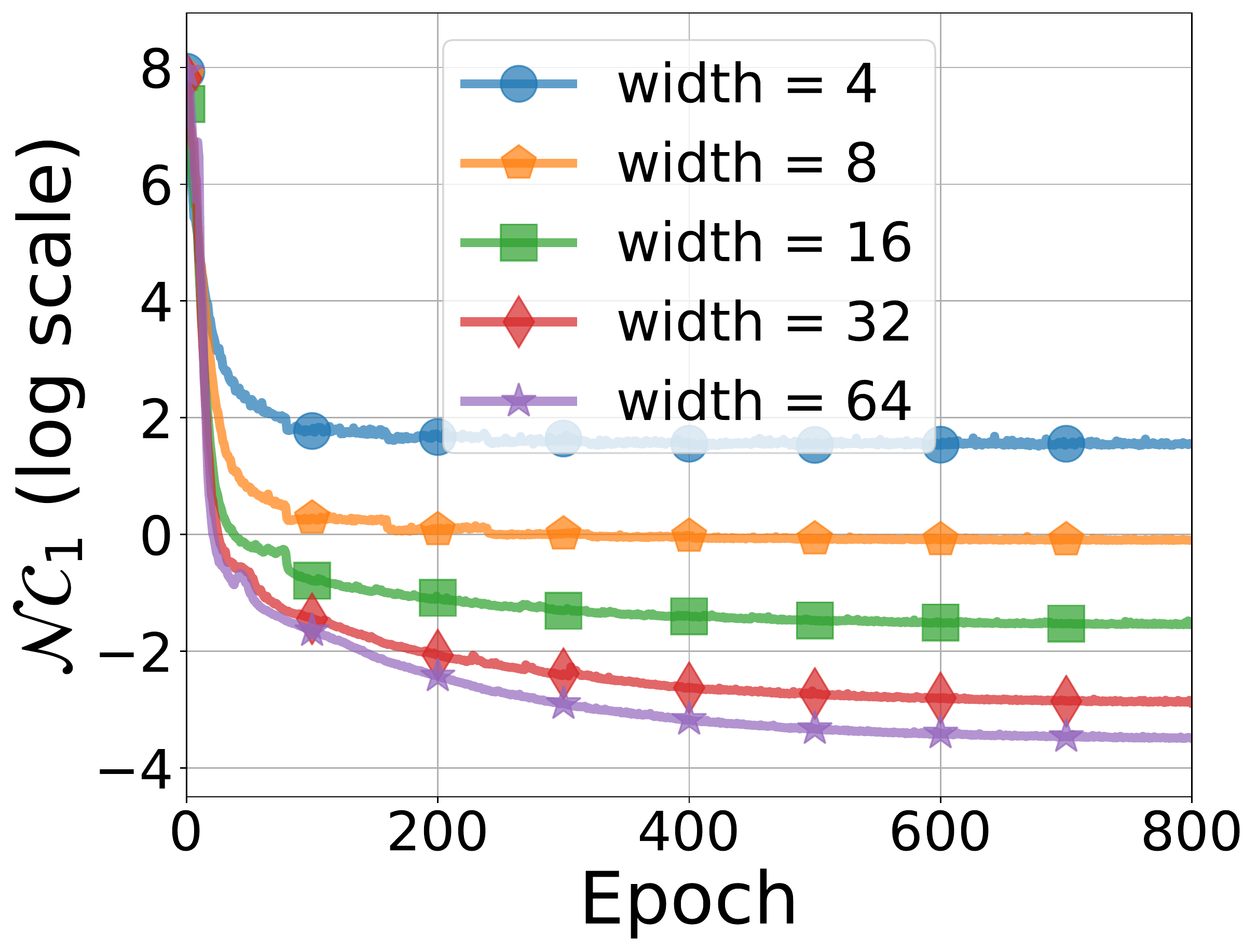} \
        \includegraphics[width=0.24\textwidth]{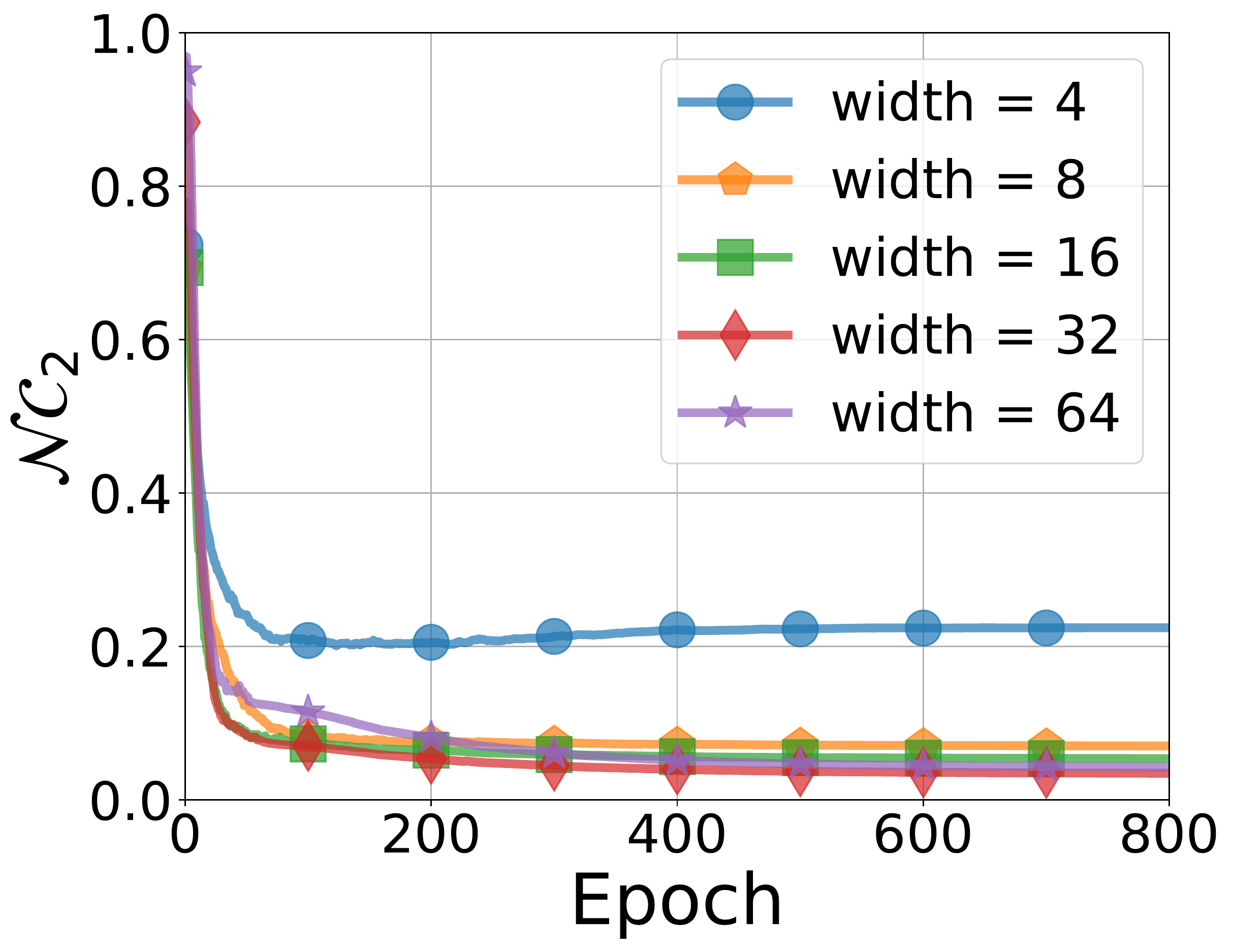} \
        \includegraphics[width=0.24\textwidth]{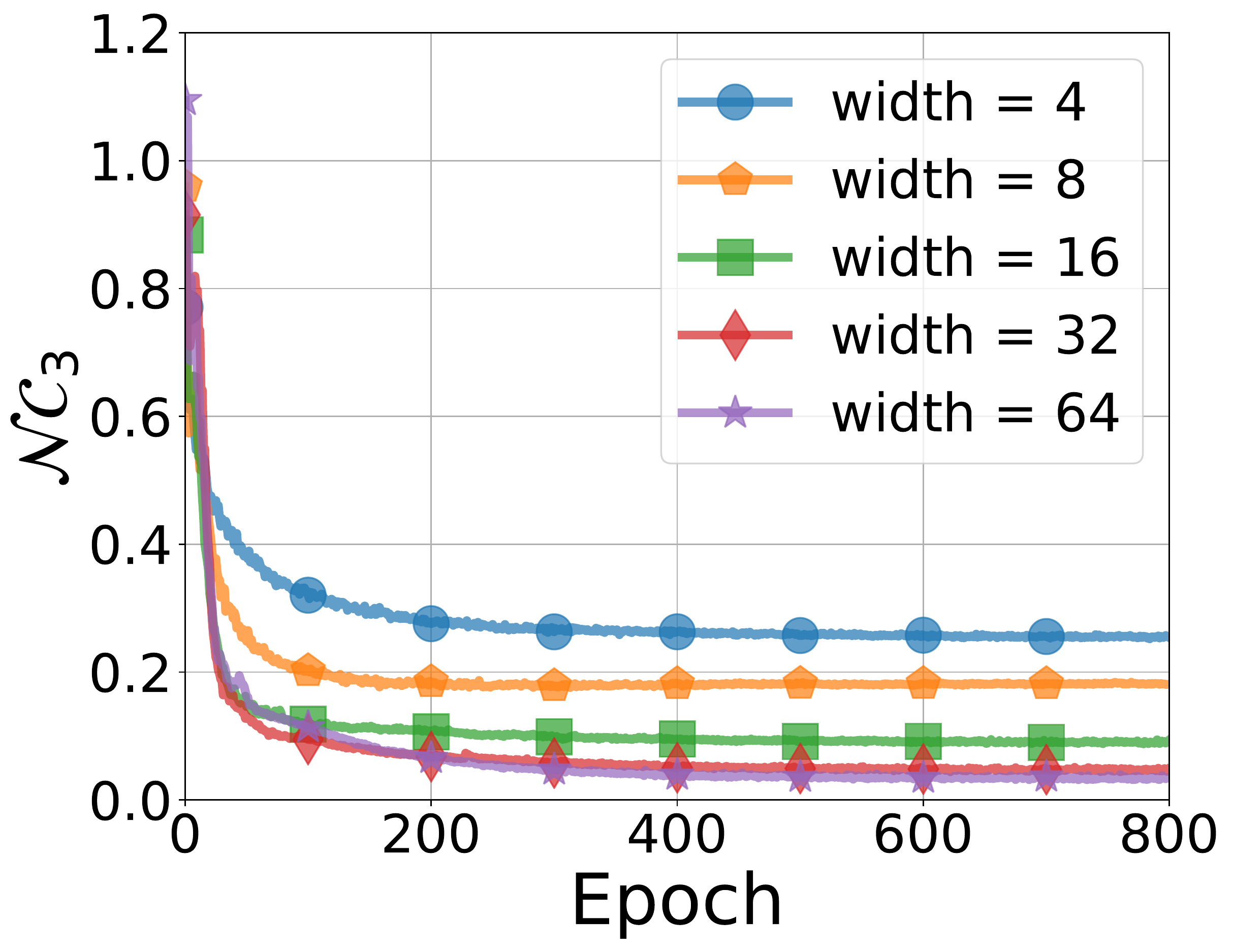} \
        \includegraphics[width=0.24\textwidth]{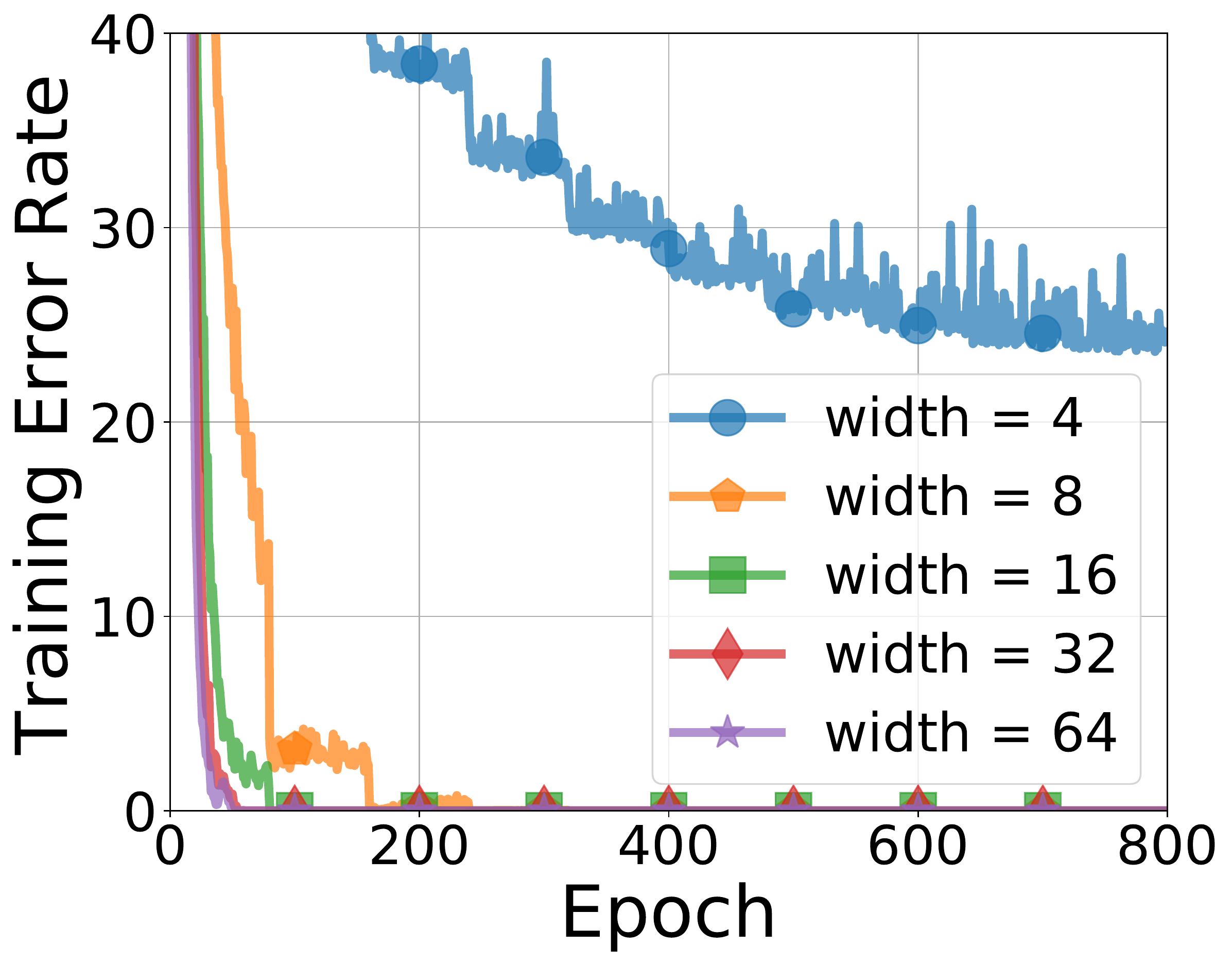}} \
    \caption{\textbf{Training results of CIFAR10 with completely random label.} Multilayer perceptron (MLP) of a fixed depth of $4$ (top) and ResNet18 (bottom) models are used with various feature width. Note that the right column is the misclassification percentage of all samples in training. %\zz{Change the Labels for the first two figures. Also the fond for the Yticks can be bigger. Maybe we can plot training error and then we can easily observe similar curves among these three figures.}
    }
    \label{fig:random_label}
\end{figure}

\paragraph{Comparison of Weight Decay on the Network Parameter $\mb \Theta$ vs. on the Features $\mb H$ in \eqref{eq:obj}.} 
In comparison to typical training protocols of deep networks which enforce weight decay on all the network weights $\mb \Theta$, our problem formulation \eqref{eq:obj} based on the unconstrained feature model replaced $\mb \Theta$ by penalizing the feature $\mb H$ produced by the $L-1$ ``peeled-off'' layers. To check the practicality of such formulation, we empirically run experiments using ResNet18 on MNIST and CIFAR10. \Cref{fig:NC-WD} shows the \NC\ evolution for both the classical formulation and our ``peeled'' formulation, we notice that the \NC\ behavior happens in both scenarios comparably. We also point that without extensive hyper-parameter tuning, the models trained under the ``peeled'' set-up could already achieve test accuracy of $99.57\%$ and $77.92\%$ on MNIST and CIFAR10 respectively. We note that such performances are on-par with the test accuracy of the 
classical formulation \eqref{eqn:dl-ce-loss}, with test accuracy of $99.60\%$ and $78.42\%$ on MNIST and CIFAR10 as reported in \Cref{fig:Accuracy-MNIST-CIFAR10}.
%In \textbf{xx}, we demonstrate that such a difference leads to comparable performance in practice -- the performance of the trained network remains similar, and \NC\ still occurs. 

%While the practice of deep learning adopts weight decay on the parameters $\mb \theta$ of the $L-1$ ``peeled-off'' layers, the unconstrained feature model proposes to instead penalize the norm of the features $\mb H$ that are produced by the ``peeled-off'' layers. In the following, we demonstrate that such a difference is negligible in practice: even if we replace the weight decay on $\mb \theta$ -- as in a typical training protocol of deep networks-- with weight decay on $\mb H$, t

%\zz{Any results yet? We may add this later.}

\begin{figure}[t]
    \centering
    \subfloat[$\mc {NC}_1$ (MNIST)]{\includegraphics[width=0.24\textwidth]{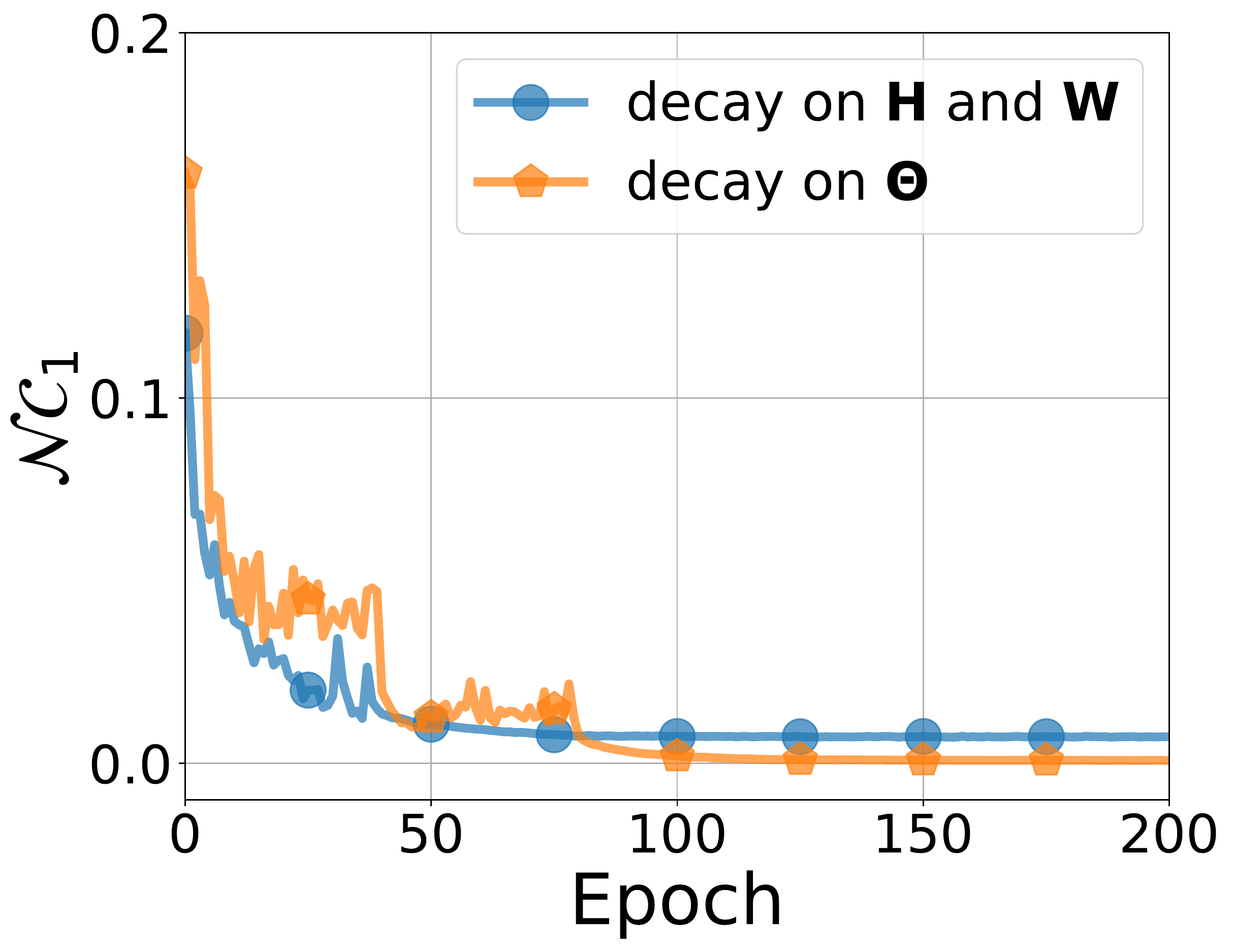}} \
    \subfloat[$\mc {NC}_2$ (MNIST)]{\includegraphics[width=0.24\textwidth]{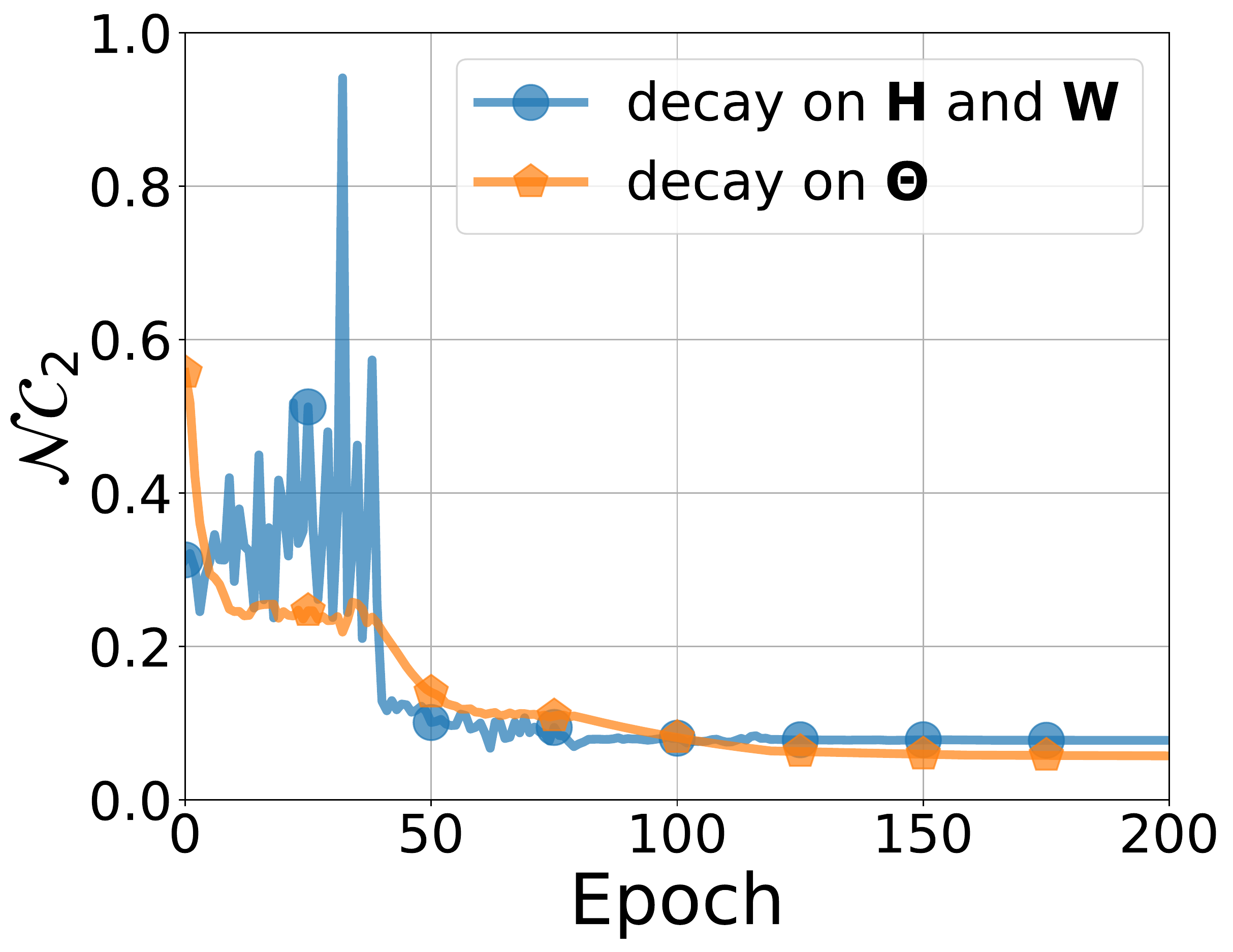}} \
    \subfloat[$\mc {NC}_3$ (MNIST)]{\includegraphics[width=0.24\textwidth]{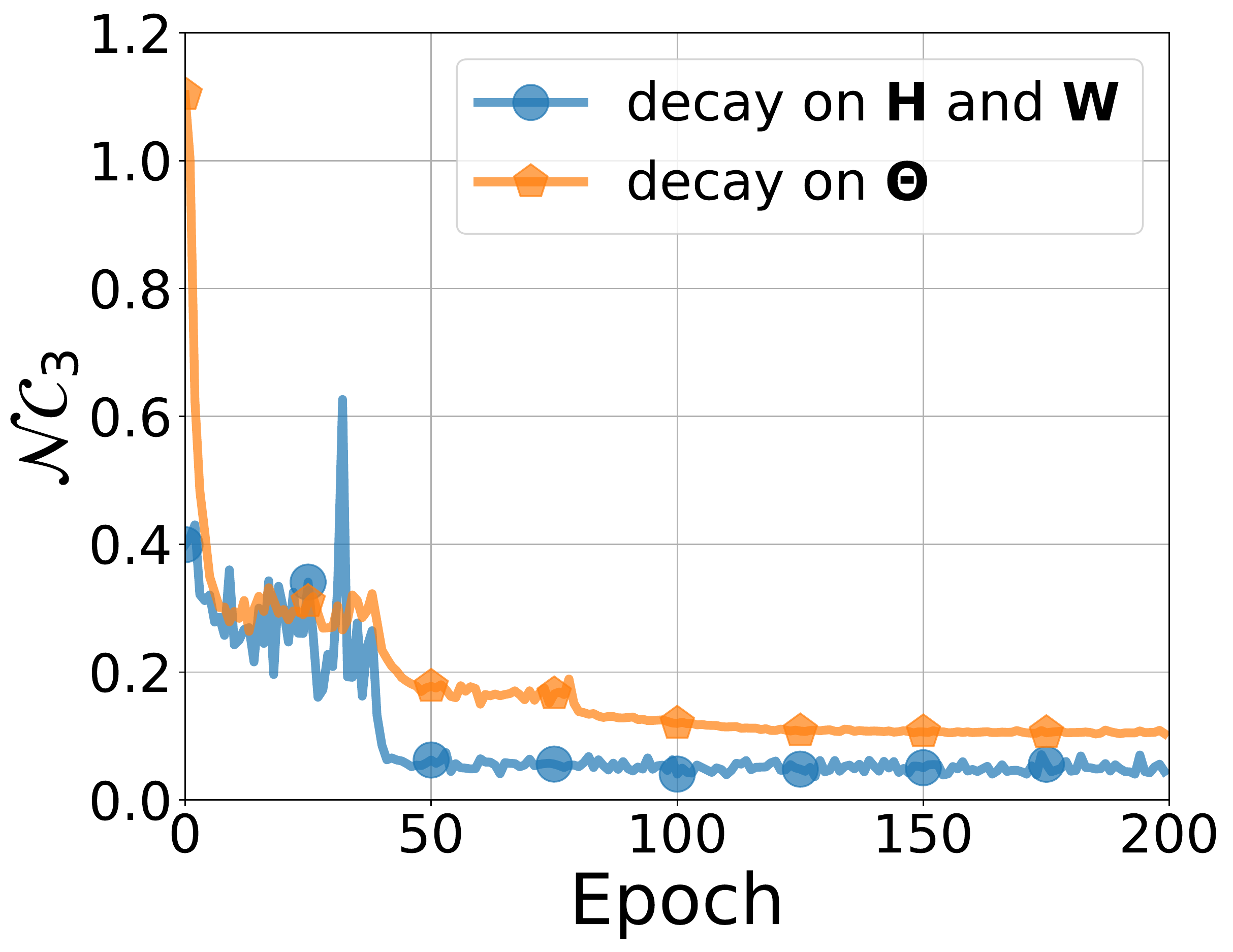}} \
    \subfloat[$\mc {NC}_4$ (MNIST)]{\includegraphics[width=0.23\textwidth]{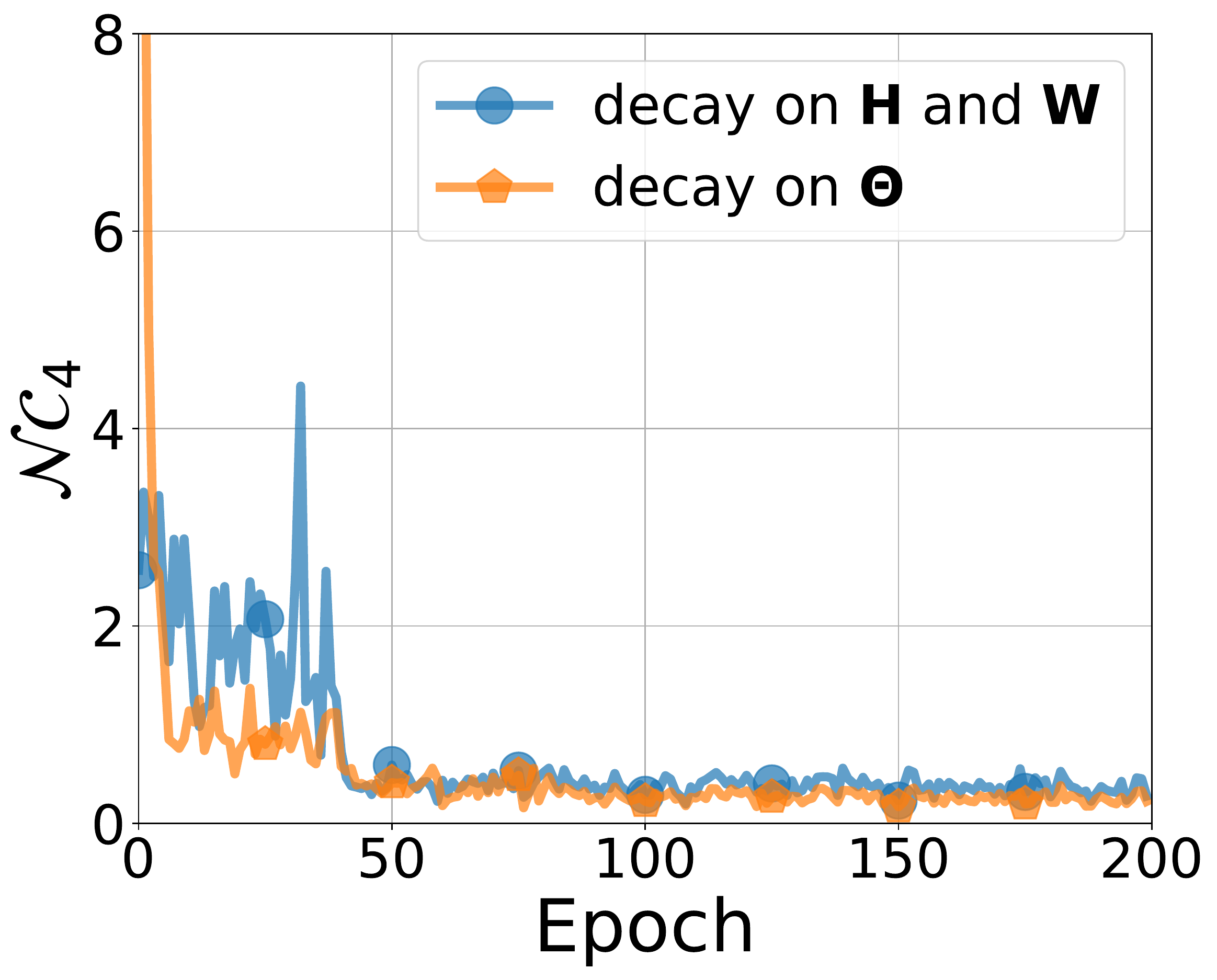}} \\
    
    \subfloat[$\mc {NC}_1$ (CIFAR10)]{\includegraphics[width=0.24\textwidth]{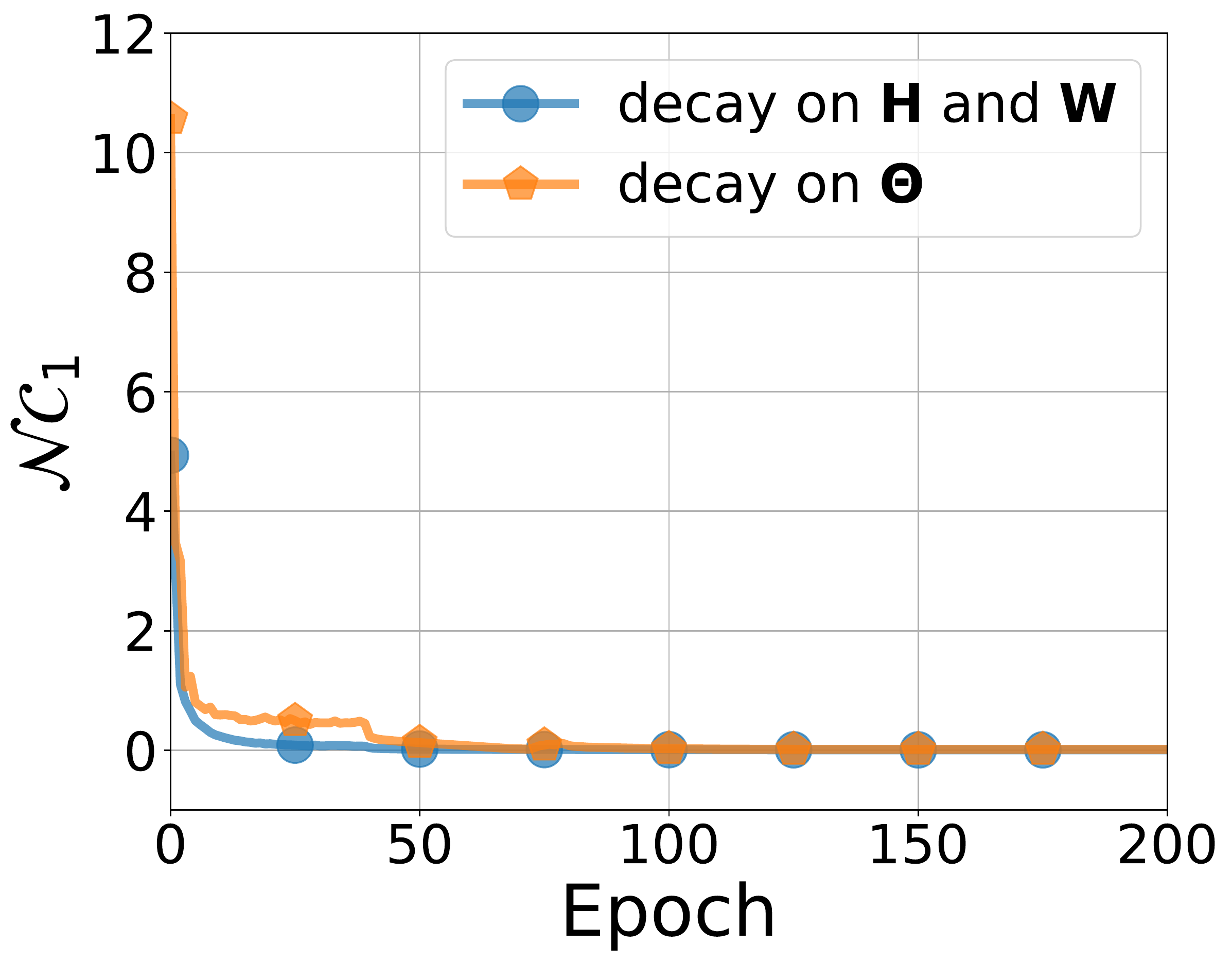}} \
    \subfloat[$\mc {NC}_2$ (CIFAR10)]{\includegraphics[width=0.24\textwidth]{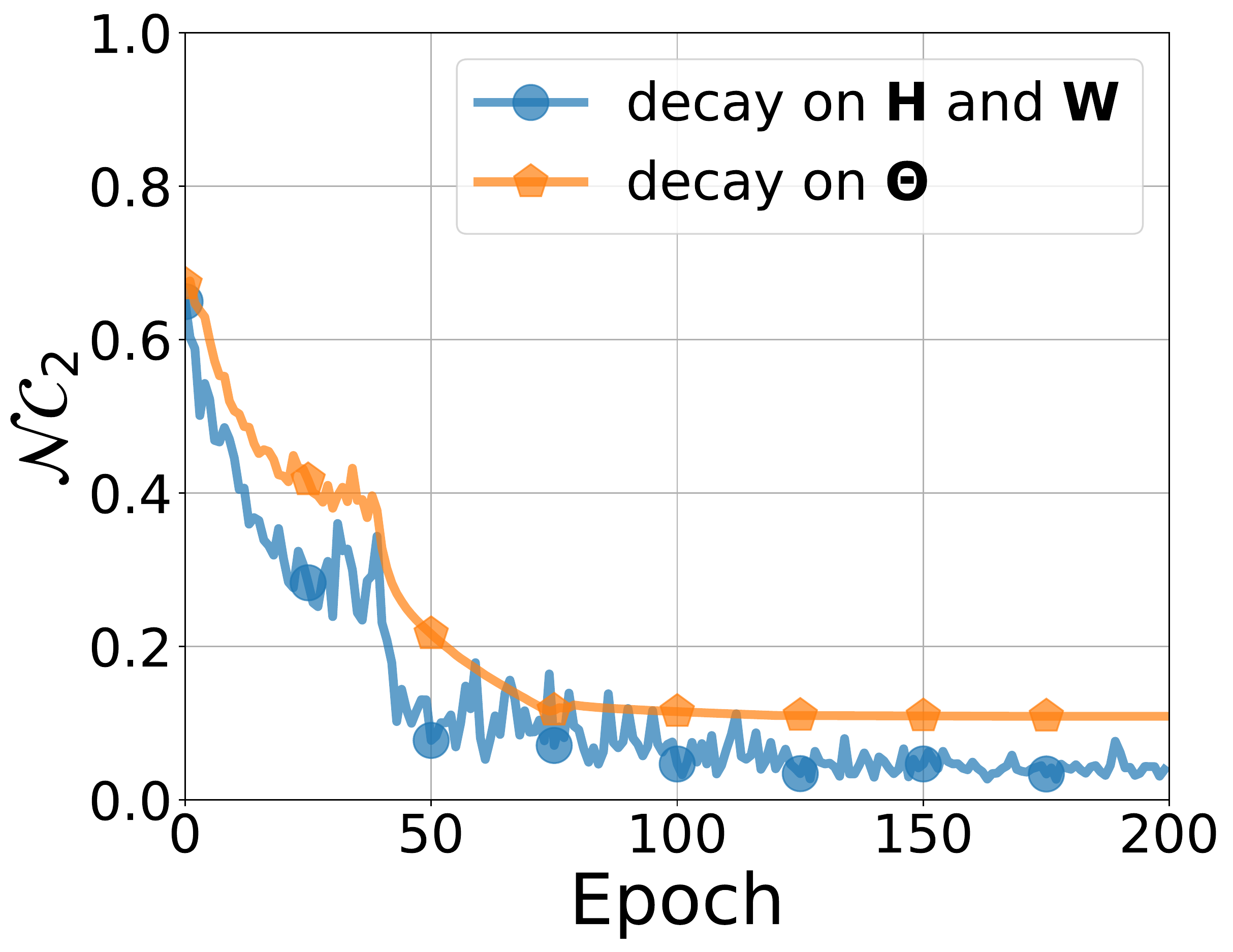}} \
    \subfloat[$\mc {NC}_3$ (CIFAR10)]{\includegraphics[width=0.24\textwidth]{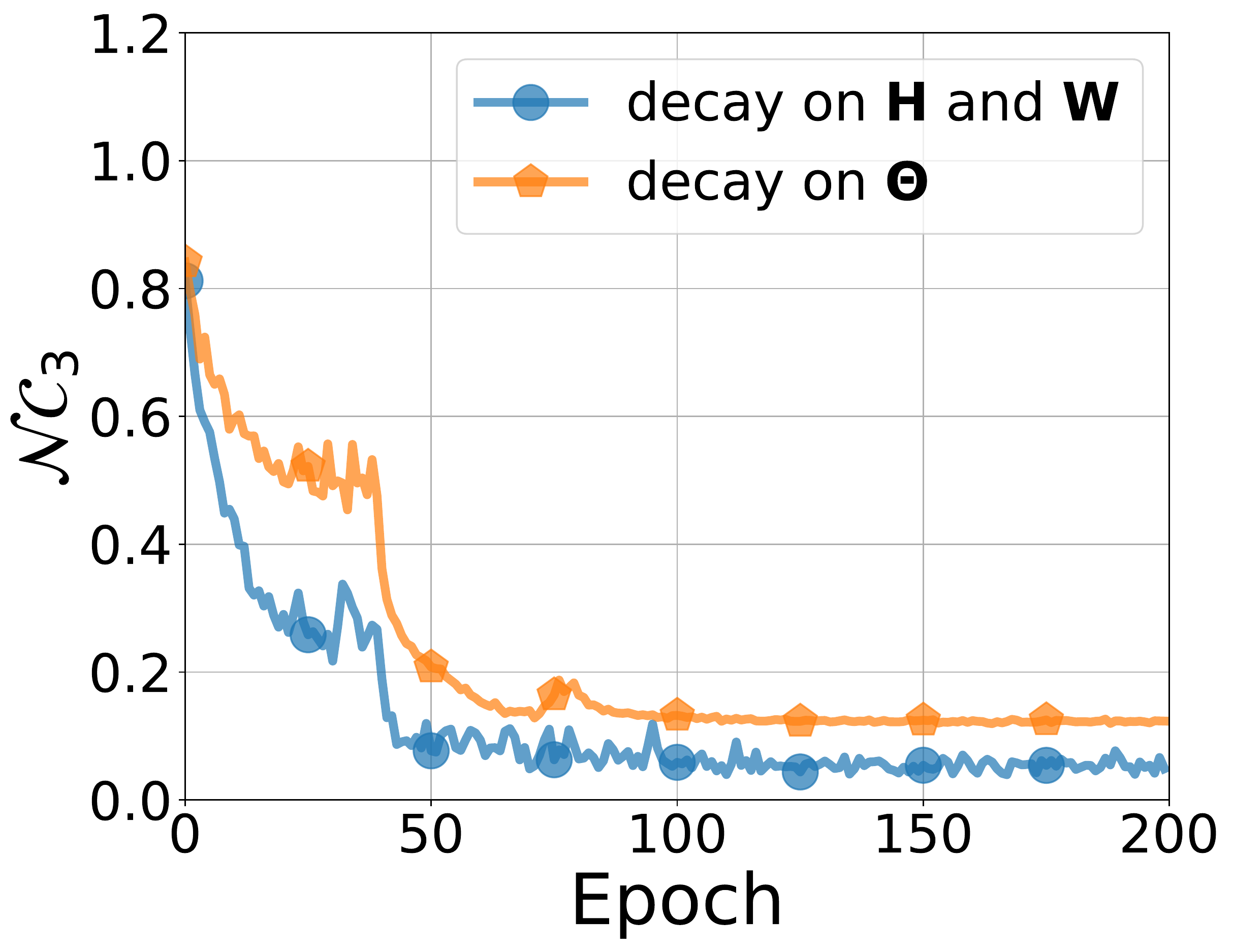}} \
    \subfloat[$\mc {NC}_4$ (CIFAR10)]{\includegraphics[width=0.23\textwidth]{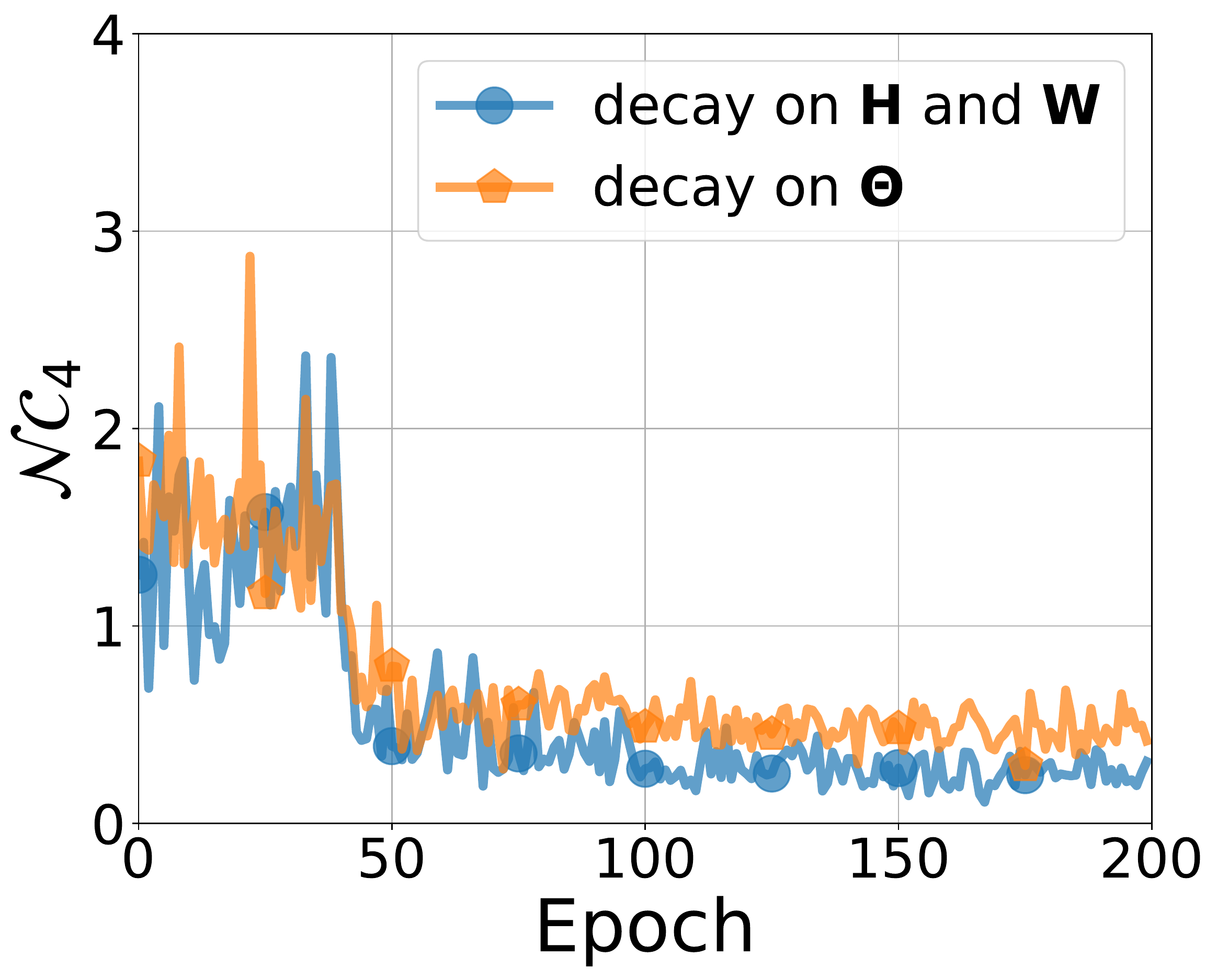}}
    \caption{\textbf{Comparison of \NC\;behavior for weight decay on $\mb \Theta$ vs. on $(\mb H,\mb W)$.}  For the latter set up, we choose $\lambda_{\mb W} = \lambda_{\mb b} = 0.01$ and $\lambda_{\mb H} = 0.00001$.
    }
    \label{fig:NC-WD}
\end{figure}

\subsection{Insights from \NC\ for Improving Network Designs}
\label{sec:exp-fix-classifier}

Finally, we conduct two exploratory experiments to demonstrate the practical benefits of \NC\ phenomenon. The universality of \NC\;implies that the final classifier (i.e. the $L$-th layer) of a neural network always converges to a Simplex ETF, which is fully determined up to an arbitrary rotation and happens when $K\leq d$. Thus, based on the understandings of the last-layer features and classifiers, we show that we can substantially improve the cost efficiency on network architecture design without the sacrifice of performance, by \emph{(i)} fixing the last-layer classifier as a Simplex ETF, and \emph{(ii)} reducing the feature dimension $d=K$. Here, to demonstrate our method can achieve state-of-the-art performance, it should be noted that for CIFAR10 dataset we also run our experiments on a modified ResNet50 architecture \cite{res50} with data augmentation which achieves around 95\% test accuracy.

\begin{figure}[t]
   \centering
    % \subfloat[$\mc {NC}_1$ (MNIST)]{\includegraphics[width=0.24\textwidth]{figs/resnet18/mnist-SGD-ETF-fixdim/mnist-resnet18-NC1.pdf}} \
    % % \subfloat[$\mc {NC}_2$ (MNIST)]{\includegraphics[width=0.32\textwidth]{figs/resnet18/mnist-SGD-ETF-fixdim/mnist-resnet18-NC2.pdf}} \
    % \subfloat[$\mc {NC}_3$ (MNIST)]{\includegraphics[width=0.24\textwidth]{figs/resnet18/mnist-SGD-ETF-fixdim/mnist-resnet18-NC3.pdf}} \
    % % \subfloat[$\mc {NC}_4$ (MNIST)]{\includegraphics[width=0.32\textwidth]{figs/resnet18/mnist-SGD-ETF-fixdim/mnist-resnet18-NC4.pdf}} \
    % \subfloat[Training Error]{\includegraphics[width=0.24\textwidth]{figs/resnet18/mnist-SGD-ETF-fixdim/mnist-resnet18-train-acc.pdf}} \
    % \subfloat[Test Error]{\includegraphics[width=0.24\textwidth]{figs/resnet18/mnist-SGD-ETF-fixdim/mnist-resnet18-test-acc.pdf}} \\
        \subfloat[MNIST-ResNet18 (from left to right): $\mc {NC}_1$, $\mc {NC}_3$, Training Accuracy, Testing Accuracy]{\includegraphics[width=0.24\textwidth]{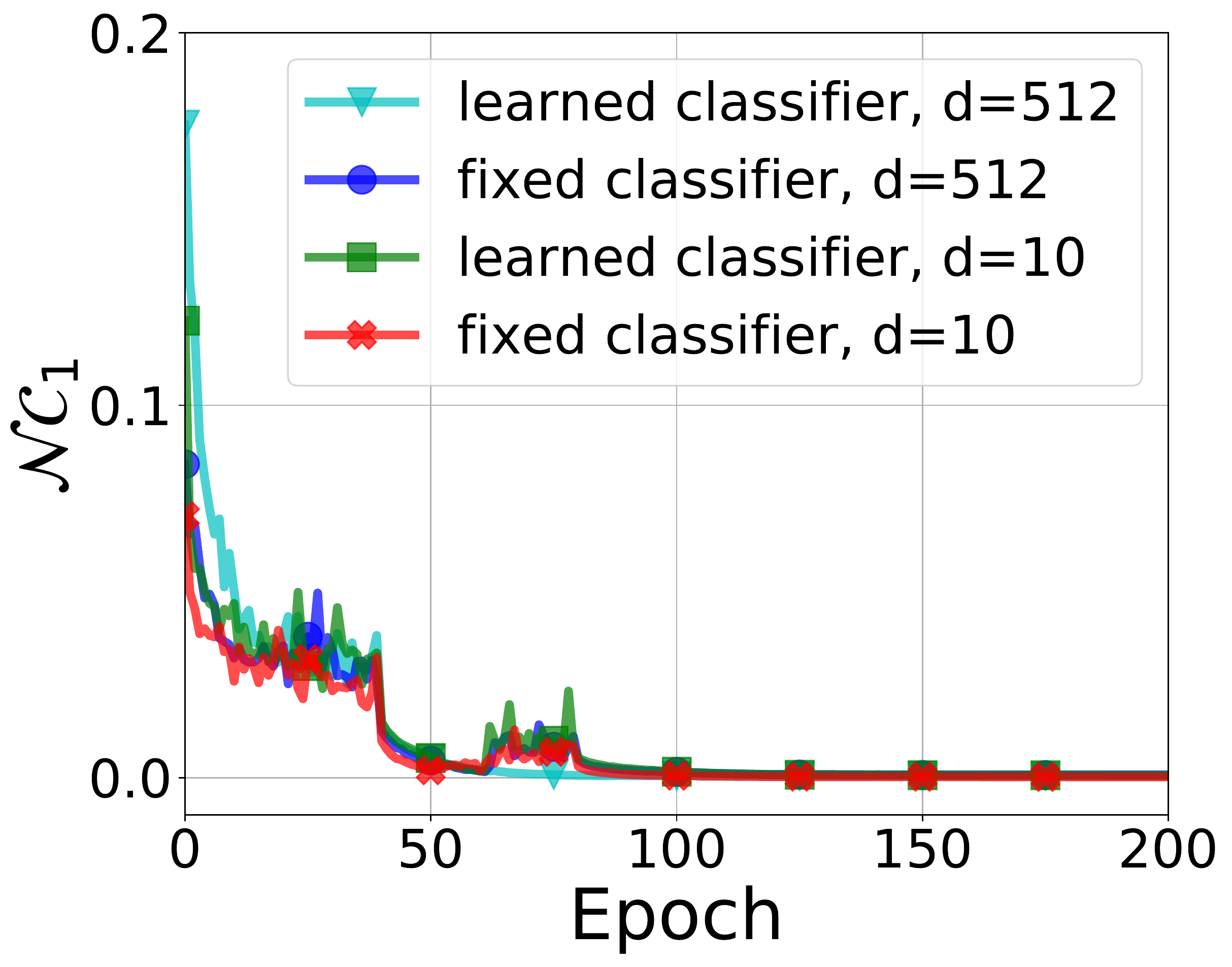} 
    % \subfloat[$\mc {NC}_2$ (MNIST)]{\includegraphics[width=0.32\textwidth]{figs/resnet18/mnist-SGD-ETF-fixdim/mnist-resnet18-NC2.pdf}} \
    \includegraphics[width=0.24\textwidth]{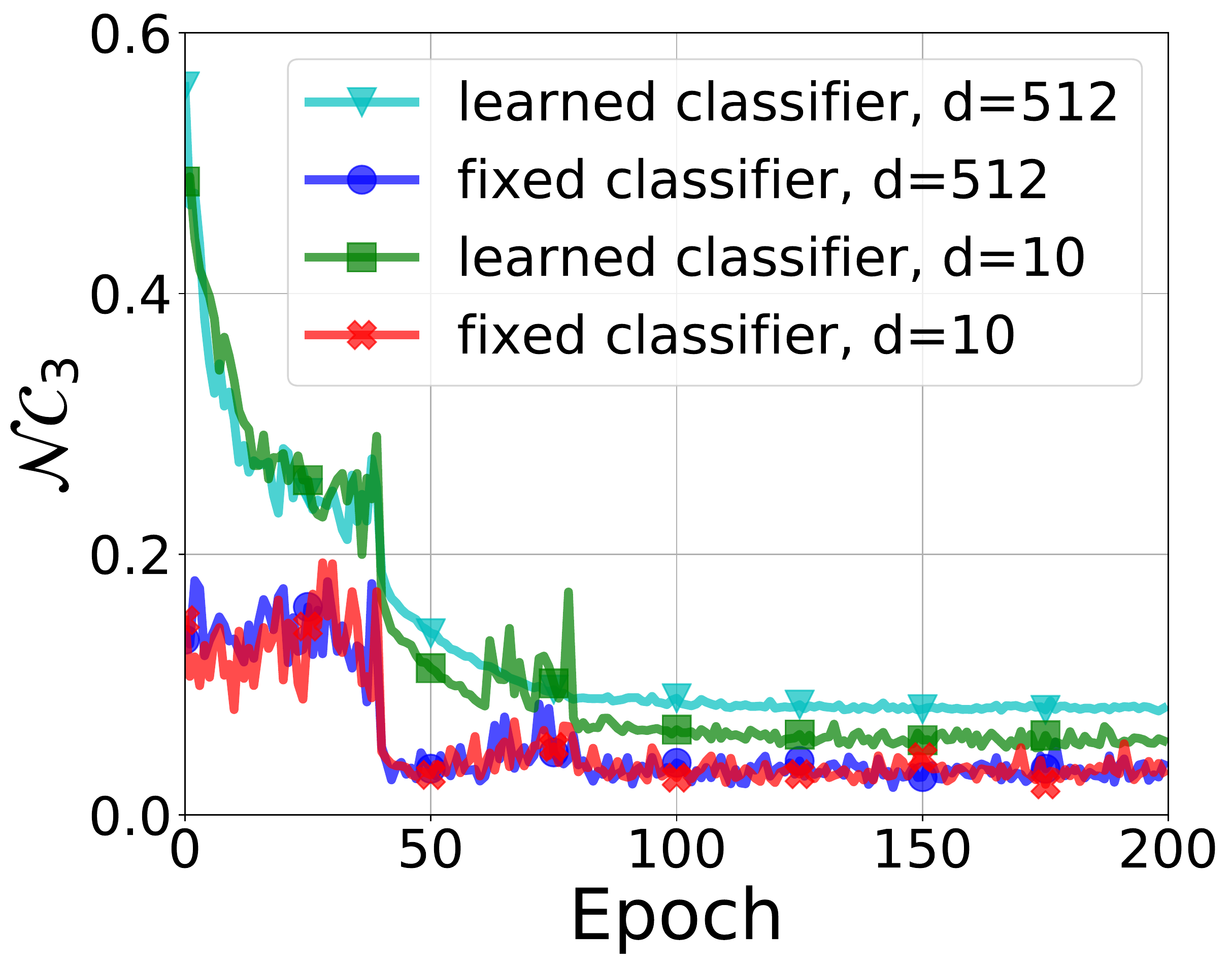} \
    % \subfloat[$\mc {NC}_4$ (MNIST)]{\includegraphics[width=0.32\textwidth]{figs/resnet18/mnist-SGD-ETF-fixdim/mnist-resnet18-NC4.pdf}} \
    \includegraphics[width=0.24\textwidth]{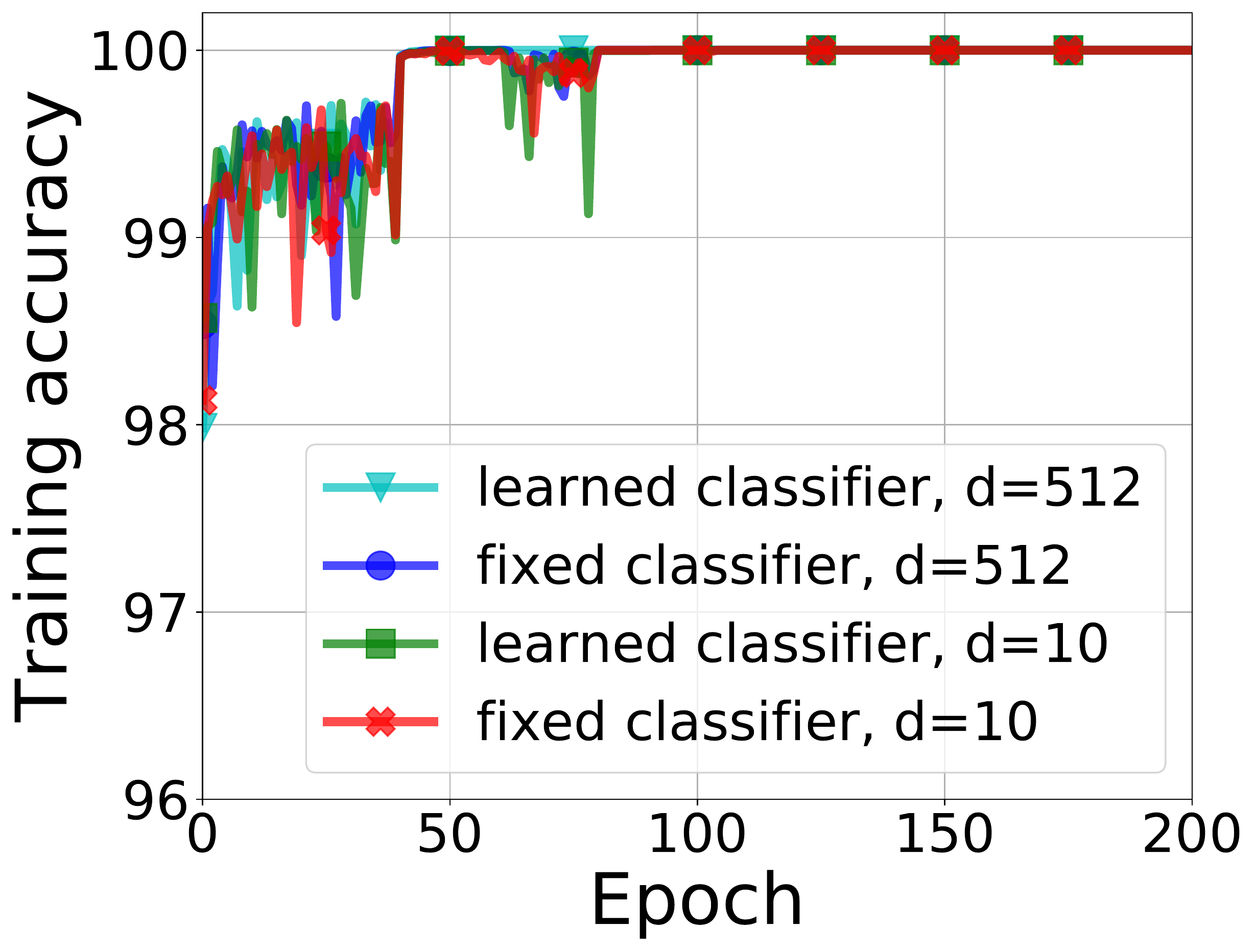} \
   \includegraphics[width=0.24\textwidth]{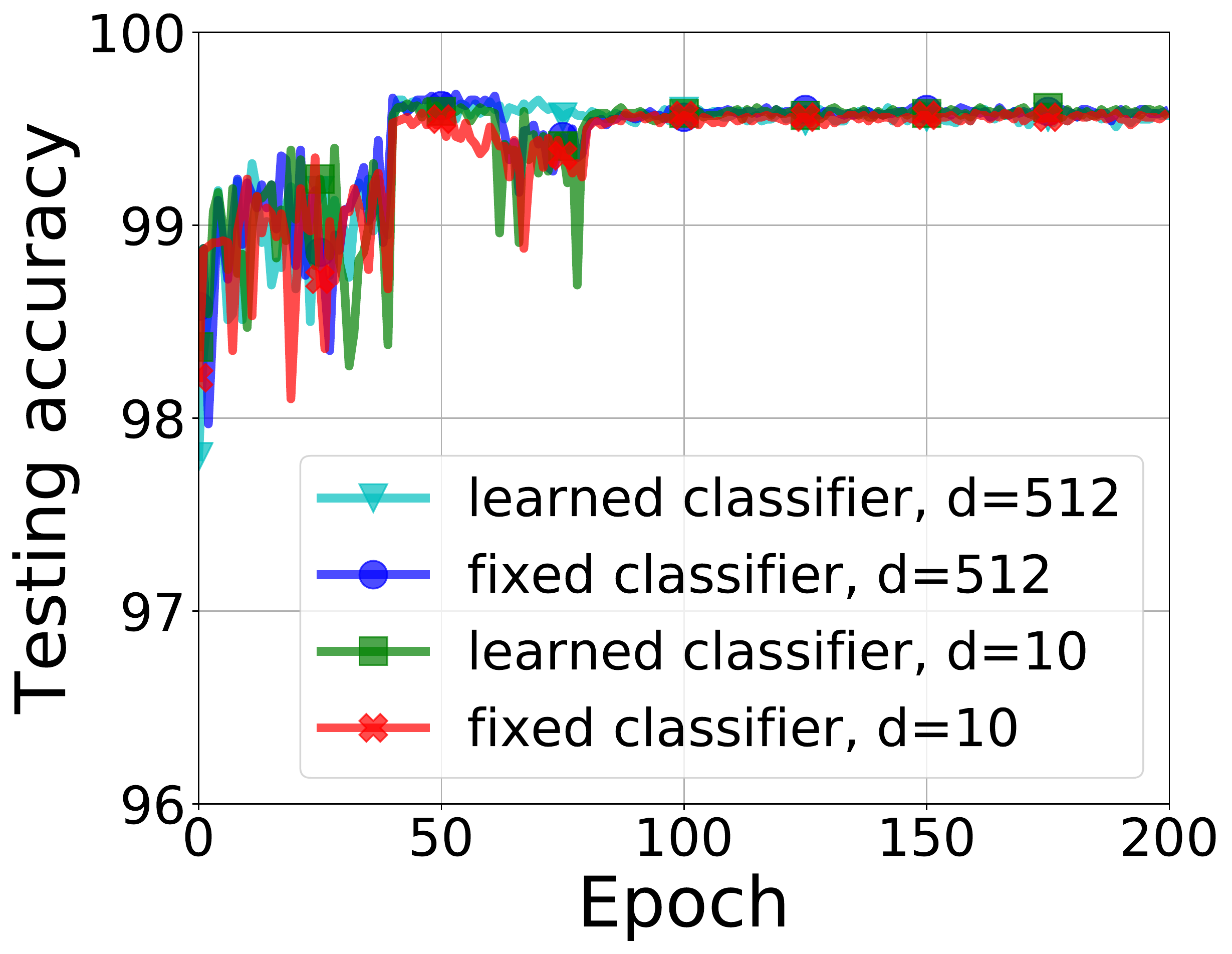}} \\
    \subfloat[CIFAR10-ResNet18 (from left to right): $\mc {NC}_1$, $\mc {NC}_3$, Training Accuracy, Testing Accuracy]{\includegraphics[width=0.24\textwidth]{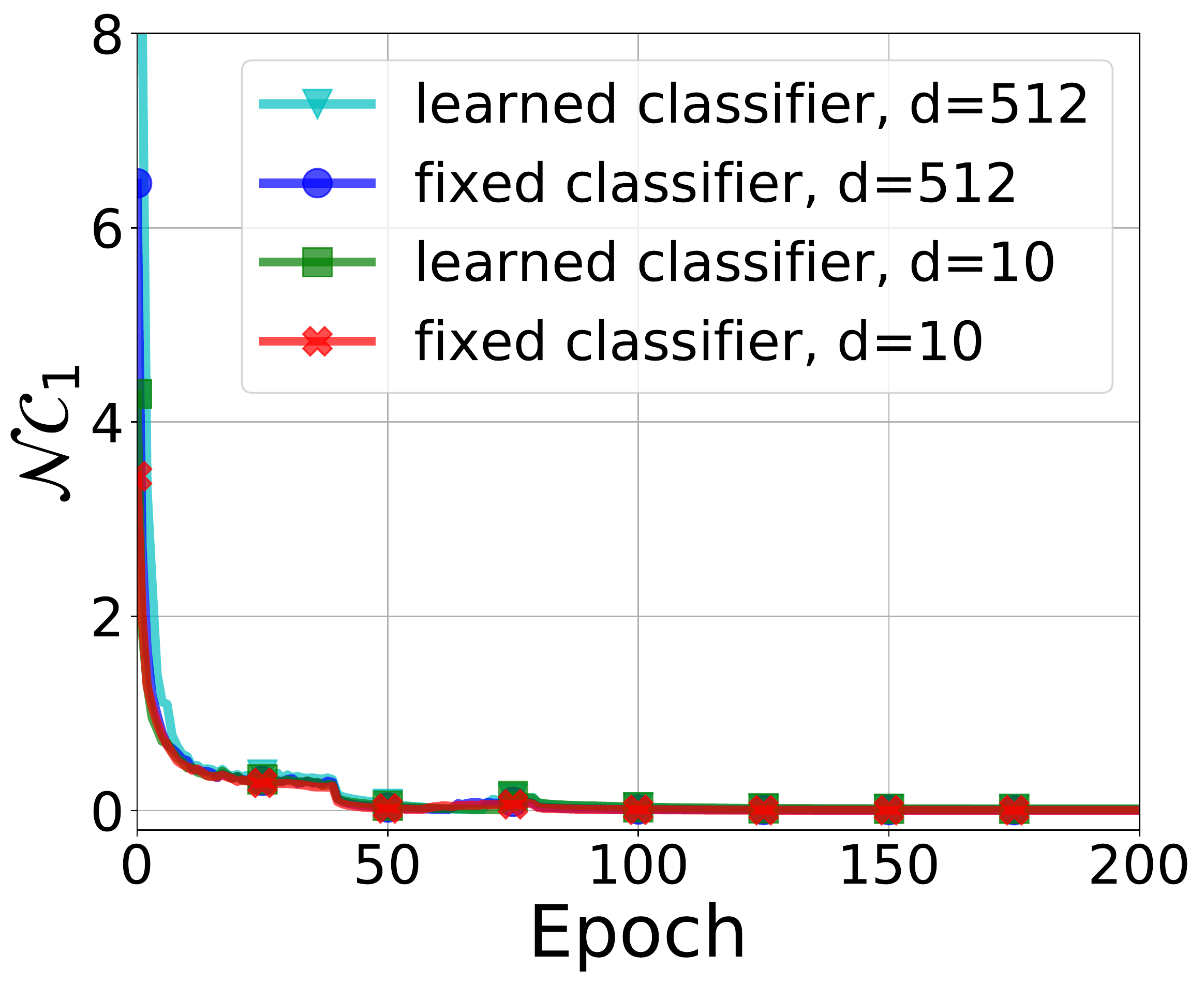} 
    % \subfloat[$\mc {NC}_2$ (MNIST)]{\includegraphics[width=0.32\textwidth]{figs/resnet18/mnist-SGD-ETF-fixdim/mnist-resnet18-NC2.pdf}} \
    \includegraphics[width=0.24\textwidth]{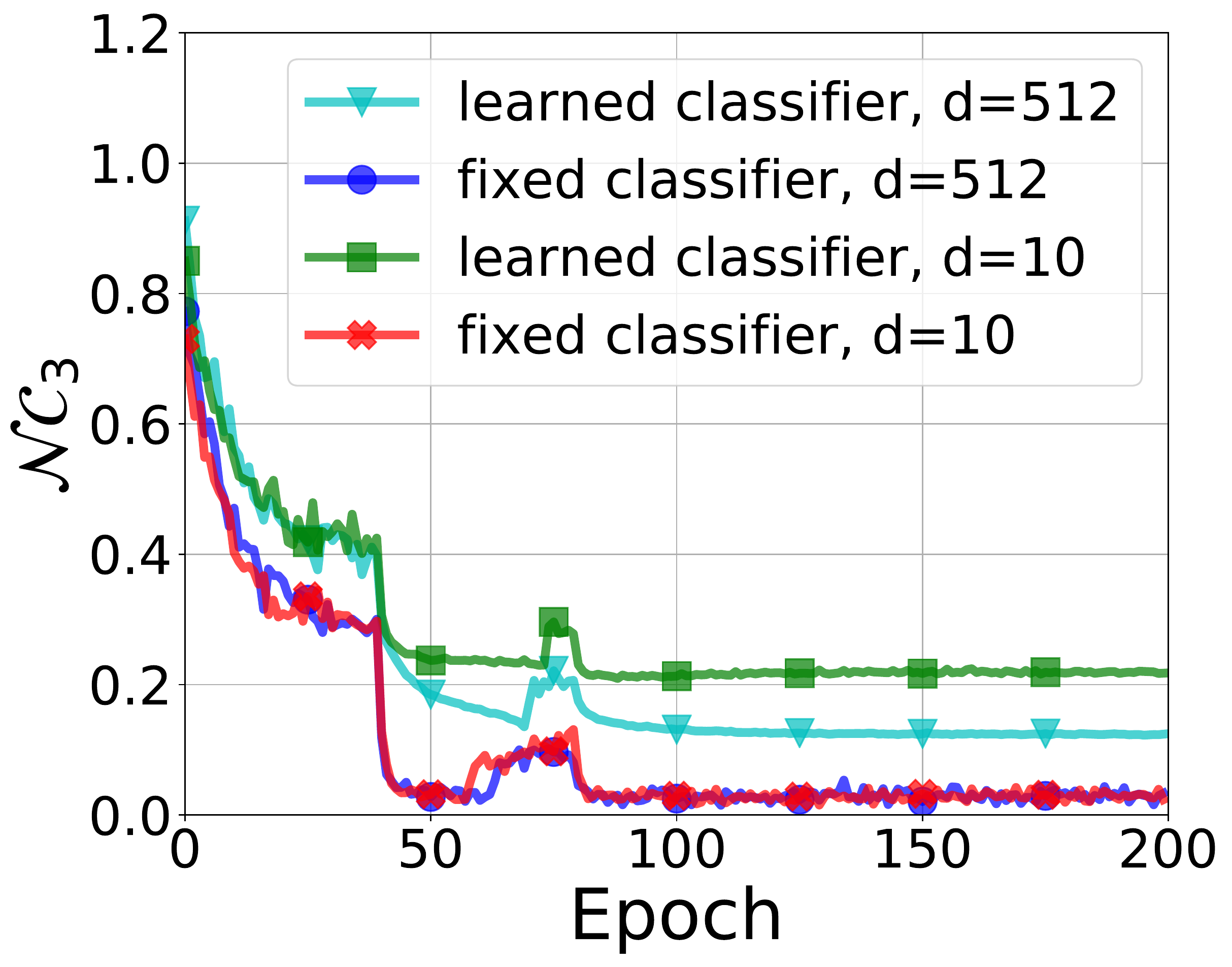} \
    % \subfloat[$\mc {NC}_4$ (MNIST)]{\includegraphics[width=0.32\textwidth]{figs/resnet18/cifar10-SGD-ETF-fixdim/cifar10-resnet18-NC4.pdf}} \
    \includegraphics[width=0.24\textwidth]{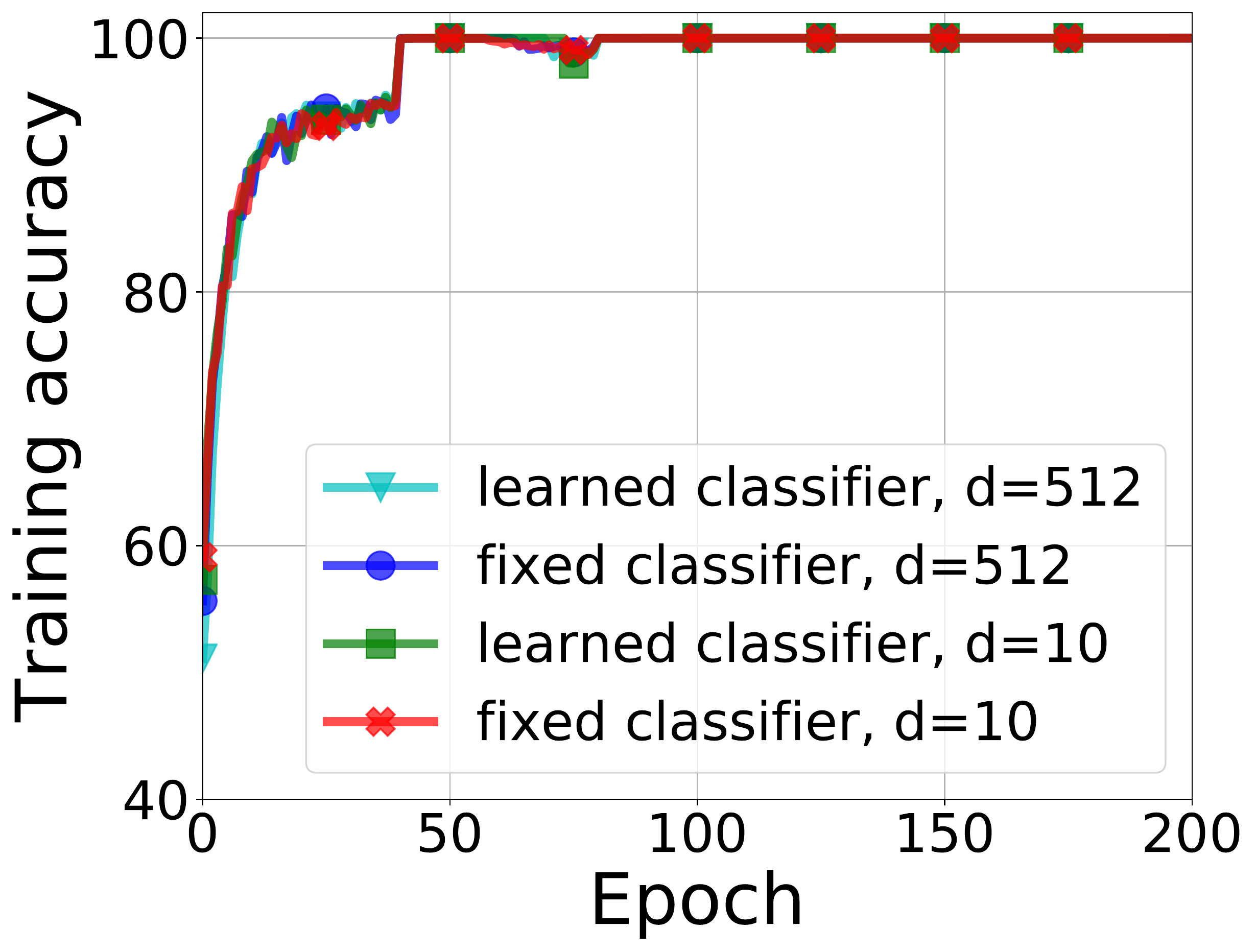} \
   \includegraphics[width=0.24\textwidth]{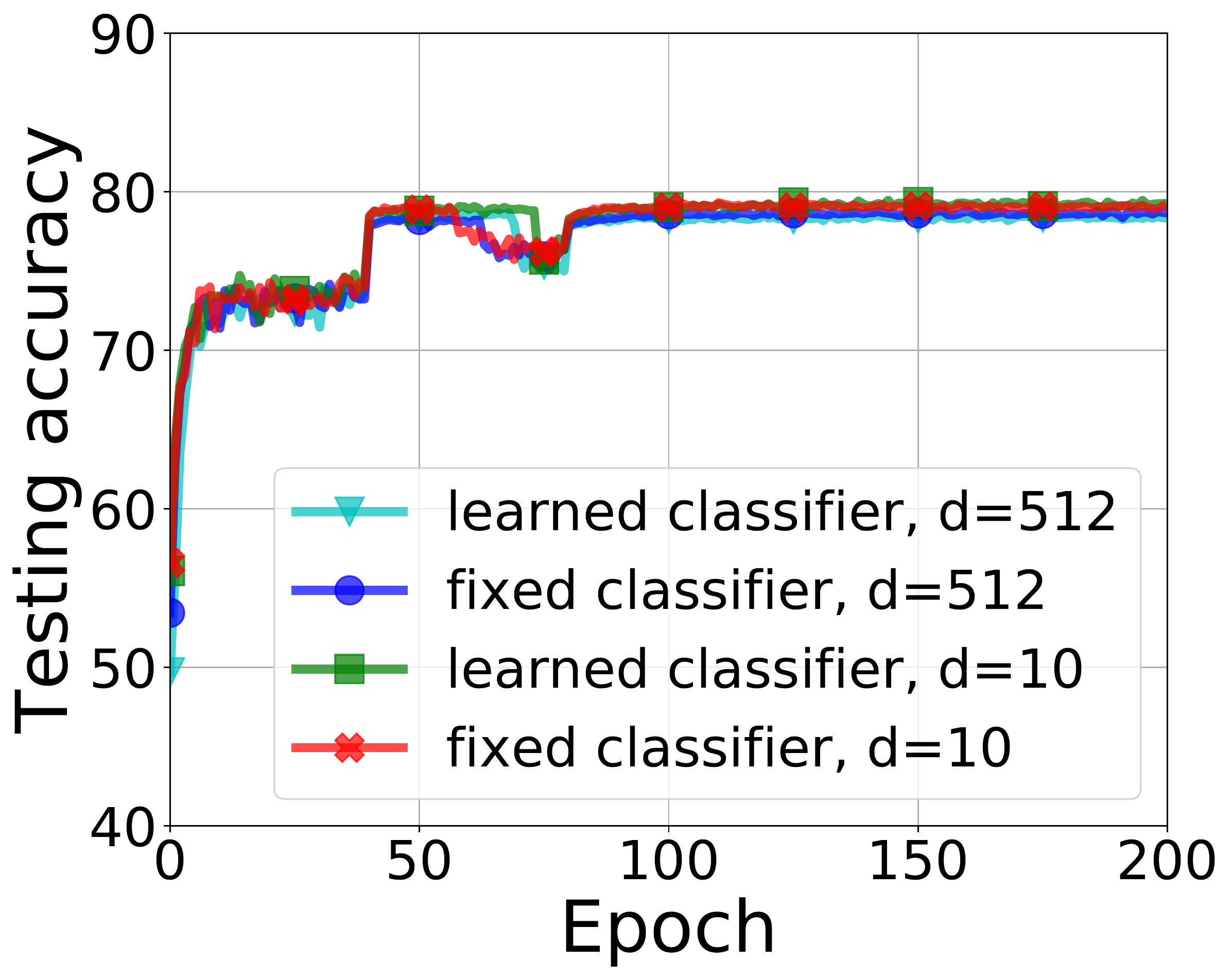}}\\
   \subfloat[CIFAR10-ResNet50 (from left to right): $\mc {NC}_1$, $\mc {NC}_3$, Training Accuracy, Testing Accuracy]{\includegraphics[width=0.24\textwidth]{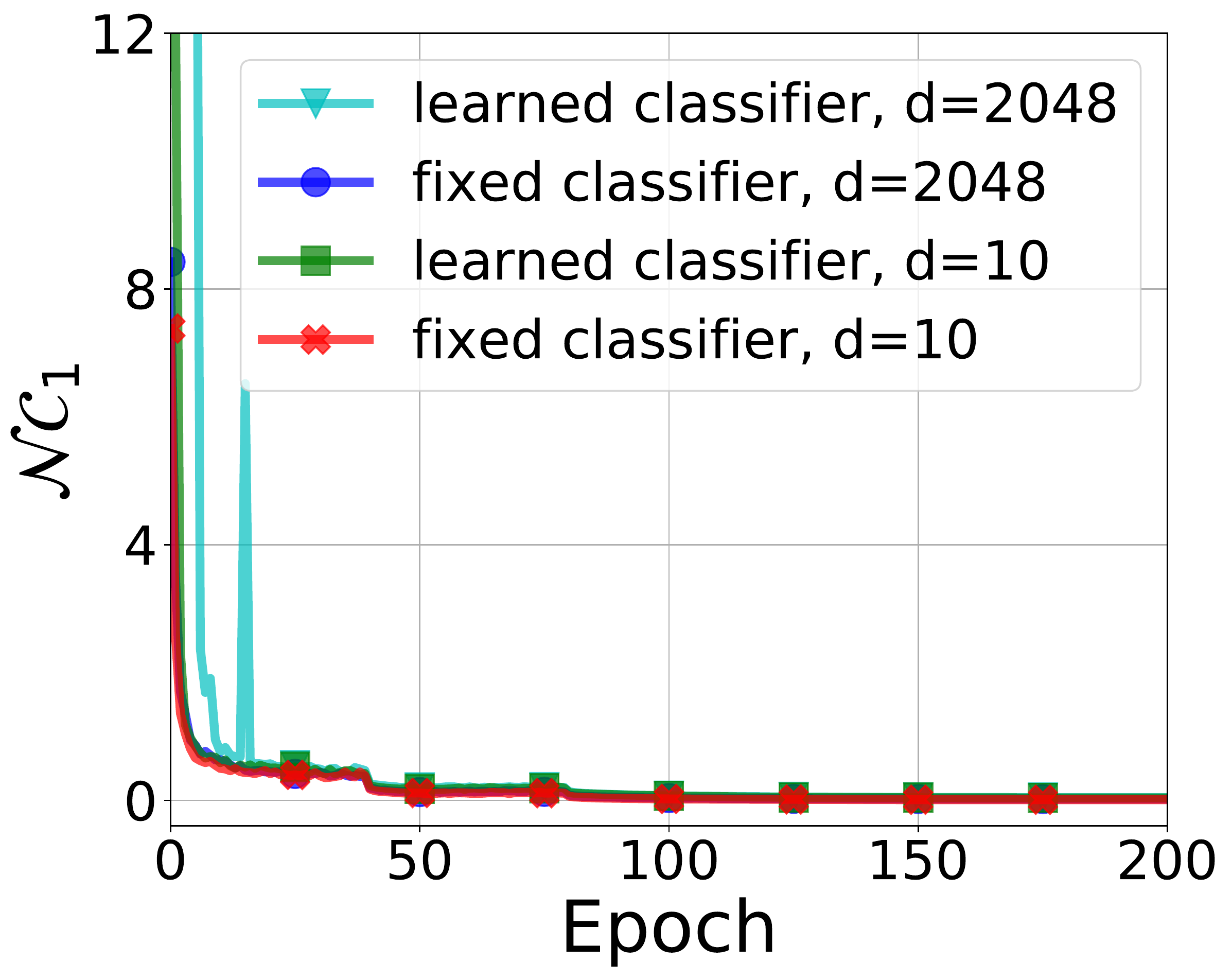} 
    % \subfloat[$\mc {NC}_2$ (MNIST)]{\includegraphics[width=0.32\textwidth]{figs/resnet50/cifar10-SGD-ETF-fixdim/mnist-resnet50-NC2.pdf}} \
    \includegraphics[width=0.24\textwidth]{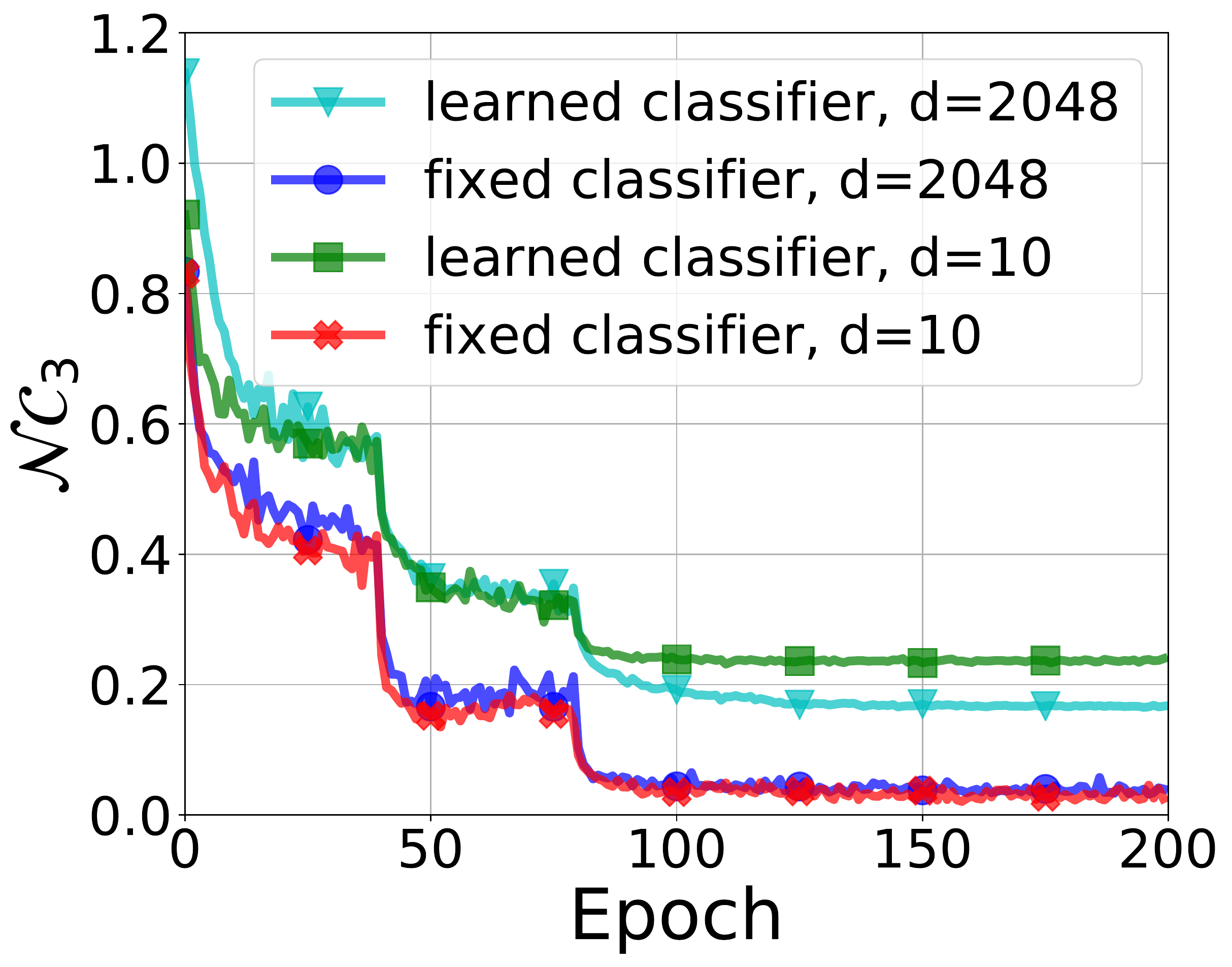} \
    % \subfloat[$\mc {NC}_4$ (MNIST)]{\includegraphics[width=0.32\textwidth]{figs/resnet50/cifar10-SGD-ETF-fixdim/cifar10-resnet50-NC4.pdf}} \
    \includegraphics[width=0.24\textwidth]{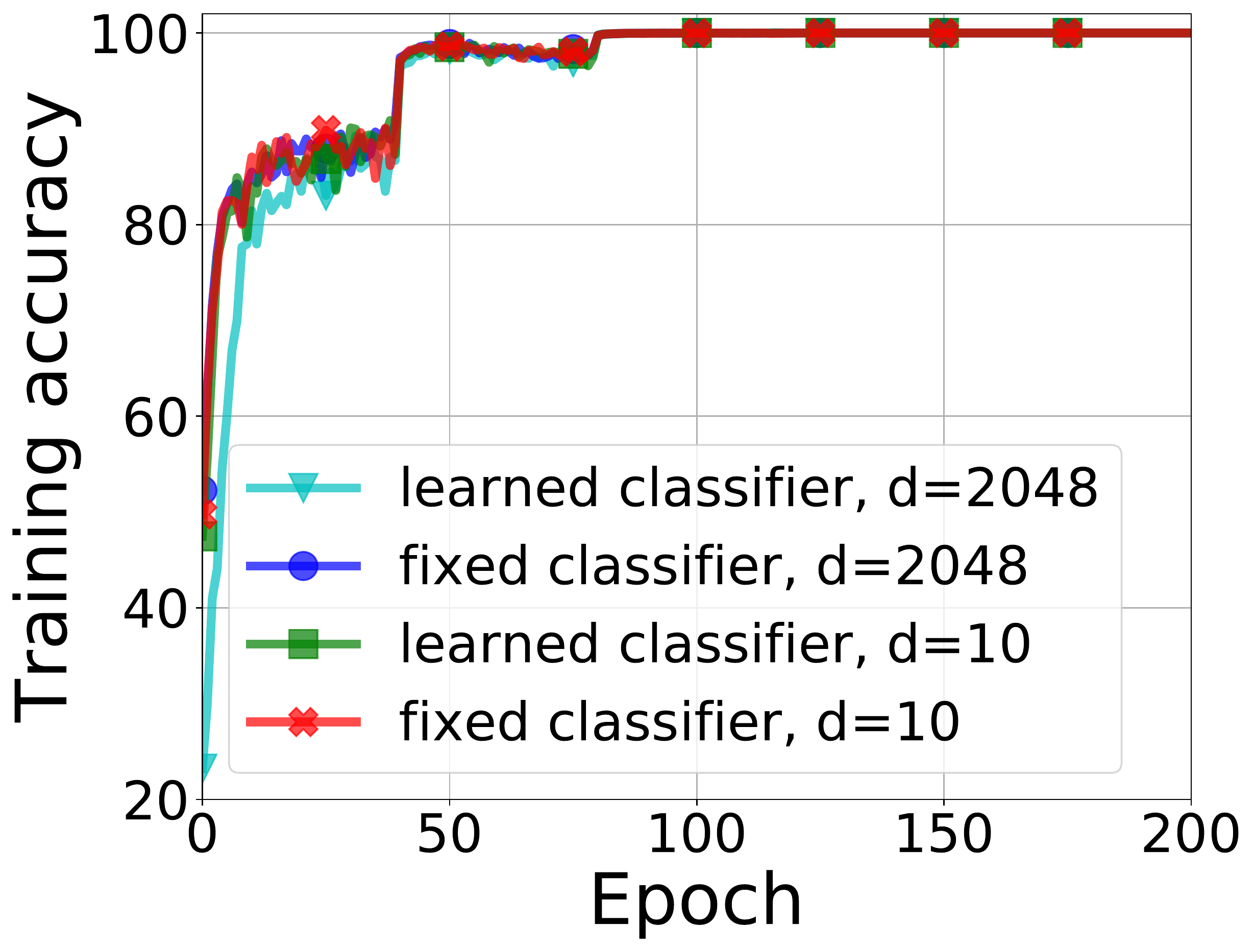} \
   \includegraphics[width=0.24\textwidth]{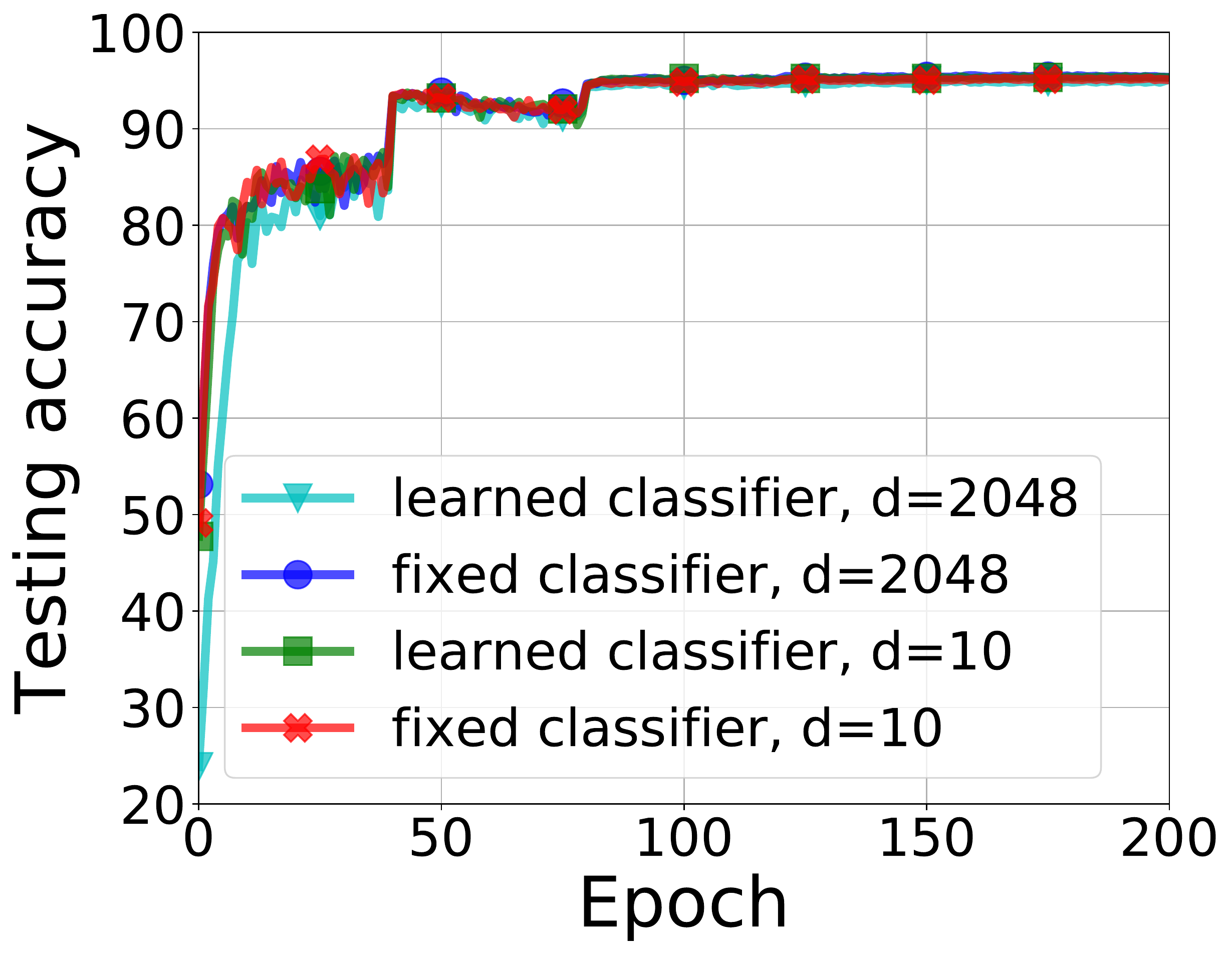}}
    \caption{\textbf{Comparison of the performances on learned vs. fixed last-layer classifiers.} We compare within-class variation collapse $\mc {NC}_1$, self-duality $\mc {NC}_3$, training accuracy, and test accuracy, on fixed and learned classifier on MNIST-ResNet18 (Top), CIFAR10-ResNet50 (Middle) and CIFAR10-ResNet50 (Bottom). Data augmentation is only used for CIFAR10-ResNet50 (Bottom). All networks are trained by SGD optimizer. %\zz{If $d=K$ has similar results, then we can just show the results $d = K$ and remove Figure 7. If $d=K$ leads to slightly different results, then we can also show here. Maybe there is no need to show neural collapse here.} %\zz{We can just show the results on SGD, plot the results for fixed and learned classifier in the same figure with legend "fixed classifer" and "learned classifer". Not sure if it's better to plot the error; see Fig.2 in \cite{hoffer2018fix}.}
  %\zz{Does this case have the ReLu at the end of the penultimate layer? If yes, for the bias term, could you also plot $\|\vb + \mW\valpha\|_2^2$, where $\valpha$ is the global mean of the entire features.}
    }
    \label{fig:fxied-classifier}
\end{figure}

\paragraph{Fixing the Last-layer Classifier as a Simplex ETF.} 
Since we know that the last-layer classifiers and features exhibit \NC, there is no need to train the layer -- one can fix the classifier $\mb W$ as a Simplex ETF throughout the training process. In the following, we demonstrate that such a strategy achieves on-par performance with classical fully training protocols and substantially reduces the number of training parameters. More specifically, we train a ResNet18 and fix the weights in the last layer as a Simplex ETF\footnote{Specifically, we set $\mW^\top = \sqrt{\frac{K}{K-1}} \mP \paren{ \mb I_K - \frac{1}{K} \mb 1_K \mb 1_K^\top}$ where $\mP\in\R^{d\times K}$ contains the first $K$ columns of a $d\times d$ identity matrix, which lifts an $K\times K$ ETF to $d\times K$ matrix. %\qq{what is $\mb P$ used for?} \zz{Use to lift the an $K\times K$ ETF to $d\times K$} matrix.
}
throughout training for MNIST dataset. For CIFAR10 dataset, we train a ResNet18 and also a ResNet50 \cite{res50} (so to match the state-of-the-art performance on CIFAR10) and fix the weights in the last layer as a Simplex ETF. We use the data augmentation introduced in \cite{shah2016deep} during the training phase for ResNet50.\footnote{More concretely, the type of data augmentations we use are: \emph{(i)} Add  for 4 pixels on each side for the color-normalized images, \emph{(ii)} do a random $32 \times 32$ crop from the padded images, and then \emph{(iii)} do a random horizontal flip with probability $0.5$. Thus, during the testing phase, the inputs to the network are only $32\times32$ color-normalized images.} We then learn the rest of the parameters of the model\footnote{For simplicity, we also learn the bias term in the last layer, though our result indicates that it can be set as $\mW\overline\vh$, where $\mW$ the classifier and $\overline\vh$ is the global mean of the features.} using SGD with the same settings as before for both MNIST and CIFAR10 experiments. \Cref{fig:fxied-classifier} presents a comparison of learned and fixed classifiers in terms of within-class variation collapse ($\mc {NC}_1$), self-duality ($\mc {NC}_3$), training accuracy, and test accuracy. These results imply that the fixed classifier exhibits the same within-class variation collapse for the feature $\mb H$, and achieves the same classification accuracy as the \emph{fully-trained} classifier. On the other hand, fixing the classifier can reduce the number of parameters and the computational complexity for training. The number of parameters in the classifier can be significant for tasks with a large number of classes and large feature dimensions. For example, for ImageNet, a dataset with $K=1000$ classes, fixing the classifier can reduce 8.01\%, 11.76\%, and 52.56\% of total learning parameters for ResNet50, DenseNet169 \cite{huang2017densely}, and ShuffleNet \cite{zhang2018shufflenet}, respectively.
%\js{I love this observation :)}

Finally, we note that our result also provides a theoretical justification for the work in \cite{hoffer2018fix} that fixes the classifier as orthonormal matrices. Indeed, these are close to simplex ETFs, particularly when the number of classes is large. Specifically, for a simplex ETF of size $d\times K$, the inner product between each pair of columns is $-\frac{1}{K-1}$, which is close to $0$ (and thus, the matrix is close to an orthonormal matrix) when $K$ is large. Therefore, for a dataset with many classes, such as ImageNet, setting the last layer classifier as an orthonormal matrix or a Simplex ETF is almost identical.

%We will now demonstrate that such a strategy does not hurt the performance of practical neural networks and, since it reduces the number of training parameters, has a potential computation benefit.

\noindent{\bf Feature Dimension Reduction for $\mb H \in \bb R^{d \times nK} $ by Choosing $d = K$.}\footnote{Although \Cref{thm:global-geometry} only holds for $d>K$, we conjecture it holds for any $d\geq K$ as we have discussed in \Cref{subsec:main-geometry}.}
In many classification problems, the practice of deep learning typically uses a feature dimension $d$ that is much larger than the number of classes $K$. In contrast, \NC\ implies that there is no need to choose a $d$ that is much larger than the number of classes $K$. Reducing the dimension $d$ can lead to substantial reductions in memory and computation cost. %\qq{need to rephrase here} 
As shown in \Cref{fig:fxied-classifier}, we also train all the weights of ResNet18 on MNIST and ResNet50 on CIFAR10 using SGD with $d=K$, respectively. The results demonstrate that \NC\ persists even when we choose $d=K$, and the network achieves on-par performance with networks of large $d$, in terms of training and test accuracy. This implies that when the number of classes $K$ is small, we can choose a small feature dimension $d \in O(K)$ instead of using a large universal $d$ to reduce the computation and memory costs for training. By setting $d=K$, this reduces the amount of parameters and hence the memory cost in ResNet18 and ResNet50 by 20.70\% and 4.45\% respectively.

\section{Conclusion}\label{sec:conclusion}

%\subsection{Relationship to Prior Arts}

%\subsection{Future Directions}

In this work, we have provided an in-depth analysis to demystify the \NC\;phenomenon, which usually happens during the terminal phase of training deep networks in classification problems. Based on the unconstrained feature model \cite{fang2021layer,lu2020neural,weinan2020emergence}, we proved that Simplex ETFs are the only global minimizers of the cross-entropy training loss with weight decay and bias. Moreover, we showed that the loss function is a strict saddle function with respect to the last-layer features and classifiers, with no other spurious local minimizers. In contrast to existing landscape analyses for deep neural networks, which mostly focus on the optimization perspective and is often disconnected from practice, our simplified analysis not only characterizes the features are learned in the last layer, but it also explains why they can be efficiently optimized. This provides support for empirical observations in practical deep network architectures. Moreover, the study of last-layer features could have profound implications for optimization, generalization, and robustness of broad interests, which we discuss in the following.

\paragraph{Investigating Deeper Models.} Our analysis has so far been focused on the last-layer features, treating them as free optimization variables thanks to the overparameterized nature of the network. A natural extension of these ideas is to further investigate features learned in shallower layers. For example, given that we know that the last-layer classifier $\mb W_L$ is a Simplex ETF, we can fix the last layer classifier $\mb W_L$ and study the following 2-layer unconstrained feature model, % Peeled Model,\zz{Change the name of layer peeled model}
\begin{align*}
    \psi_{\mb \Theta}(\mb x) \;=\;  \mb W_L \sigma (\mb W_{L-1} \underbrace{ \sigma(  \mb W_{L-2}\cdots \sigma \paren{\mb W_1 \mb x + \mb b_1}}_{\phi_{\vtheta'}'(\vx)} ) + \mb b_{L-1} )  + \mb b_L,
\end{align*}
where we treat $\mb \xi = \phi_{\vtheta'}'(\vx)$ as an optimization variable. Recent empirical evidence \cite[Figure 8]{papyan2020traces} suggests that shallower layers exhibit less severe variability collapse than deeper layers, alas in a progressive manner. Based on the analytical framework that we laid out in this work, it is of interest to investigate the patterns of learned features and the corresponding optimization landscape across shallower layers recursively.

\paragraph{Study of the Relationship Between \NC\;and Robustness.} 

It has been observed that the \NC\; phenomenon present during the terminal phase of training improves the adversarial robustness of the learned network \cite{papyan2020prevalence}. Based on our analytical framework, it would be interesting to theoretically investigate the optimization landscape and representations learned by adversarial training. Indeed, the study of the interplay of the learned features and the final classifier enables precise characterization of the adversarial robustness of the learned model \cite{sulam2020adversarial}. On the other hand, another line of recent work \cite{yu2020learning,chan2020deep} empirically showed and argued that mapping each class to a linearly separable subspace with maximum dimension (instead of collapsing them to a vertex of Simplex ETF) can improve robustness against random data corruptions such as label noise. Further empirical and theoretical investigations are needed to clarify our understandings of potential benefits and full implications of \NC\;for robustness.

%preventing \NC\; could improve . 
%It is known that modern neural networks trained in a standard procedure are vulnerable to adversarial attacks. To some extend, this can be explained by neural collapse which indicates standardtraining procedure learns simple representation which may be less robust than more complex ones.It is of interest to investigate the representations learned by adversarial training. Related work [30]Recent work [Yi Ma etc] shows that preventing nerual collapse could improve robustness againstrandom input corruptions. More empirical and theoretical investigation to clarify the understandingon this question.

\paragraph{Study of the Relationship Between \NC\;and Generalization.}
Our results on the benign landscape of unconstrained feature models imply that \NC\;is almost universal and agnostic to the choice of optimization algorithms. This seems to be delivering an exciting message that neural networks always converge to a perfect linear classifier in which the peeled-off layers learn maximally separated features and the final layer learns the corresponding maximal margin linear classifier. It is important to note, however, that our result concerns only the training data and does not apply to test data. Since the network is highly overparameterized, there can be infinitely many neural networks with \NC\;for a particular training dataset with very different generalization. This was also empirically demonstrated in our experiments, showing that different optimization algorithms lead to models that have notably different generalization performances, although all of them exhibit \NC. Therefore, the universality of \NC\; in fact implies that such a phenomenon cannot fully explain network generalization. A pertinent study of generalization will necessarily require scrutiny into the peeled-off layers, where different optimization methods impose different implicit biases on the learned network parameters.

%the behavior of which is obscured in the unconstrained feature model, as they determine the specific \NC\;solution that the network converges to. \cy{A preliminary writeup for your reference, may need to mention discussions in existing works of neural collapse as well}

% \paragraph{Study of the Relationship Between \NC\;and Neural Network Architectural Design.} \cy{I feel that neural collapse is at odds with residual learning and isometric learning in general. In particular from an optimization perspective it can be shown with both empirical and theoretical evidences that each layer should be close to an identity / isometric mapping, but such mappings are not very suited for producing neural collapse solutions. Could be another good place to mention MCR if we want to, if not having this discussion in this paper then could be good for the proposal.}

\paragraph{Study of the Relationship Between \NC\;and Network Training.} While our analysis assumes that the neural network can produce arbitrary features before the final layer due to over-parameterization, one should be reminded that this implicitly requires a good optimization landscape in training the peeled-off layers. On the other hand, it is well-known that deep models can be notoriously difficult to train, due to issues like vanishing and exploding gradients, which leads to complicated landscapes. Efficient optimization techniques for deep learning as of today often rely on residual learning \cite{he2016deep}, and isometric learning \cite{qi2020deep}, which are crucial for effective training of deep networks beyond a few layers.
Pieces of evidence come not only from extensive empirical work that improves network training via enforcing isometry \cite{xiao2018dynamical,liu2021orthogonal,liu2021convolutional} but also theoretical work that shows benign optimization landscape of residual and isometric networks \cite{hardt2016identity,hu2019provable}. %\js{i don't understand this last sentence; evidence of what?}
We note that the principles of residual learning and isometric learning, which employs an identity and isometric transformation that preserves distances (and hence, the structure of the data) at each layer, is at odds with \NC, which neglects intrinsic data structures and compresses each class into a vector of Simplex ETF. Resolving this conflict may require us to rethink the cross-entropy loss as a surrogate for the risk objective and design new objective functions that respect the intrinsic structure of the data \cite{yu2020learning}.
%\qq{maybe move some of this paragraph into the introduction instead?}
%\cy{This paragraph might be going too far; I don't know any existing evidence that backs up the argument here so may not be suited to put into intro for motivating the work. We could totally remove it here. Your call. Meanwhile, I do believe this points to an interesting future direction that worth some further investigation.}

\paragraph{Dealing with a Large Number of Classes $K\gg d$.}
Finally, it should be noted that our current analysis of the cross-entropy loss for classification focuses on cases where the number of classes, $K$, is smaller than the feature dimension $d$. This condition is crucial for showing negative curvatures of the critical saddle points. In many applications, such as recommendation systems \cite{covington2016deep} and document retrieval \cite{chang2019pre}, the number of classes can be huge and one cannot afford to design a model with feature dimension $d \ge K$. In such cases, our results do not apply; in fact, it is impossible for $K$ features in a $d$-dimensional space to form a Simplex ETF if $d < K$. This is also the case for contrastive learning \cite{chen2020simple,he2020momentum} in self-supervised feature learning \cite{jing2020self}, where augmented views of every sample are treated as a ``class'' so that the total number of classes grows with the size of the dataset. Nonetheless, recent studies in contrastive learning \cite{wang2020understanding,chen2020intriguing} showed that the features learned through the contrastive loss, a variant of cross-entropy loss \eqref{eqn:ce}, exhibit properties analogous to \NC. More precisely, samples from the same class are mapped to nearby features, referred to as \emph{alignment}, and feature vectors from different classes are as separated as mas as possible in feature space, referred to as \emph{uniformity}.
Hence, the \emph{alignment} and \emph{uniformity} are reminiscent of \emph{variability collapse} and \emph{convergence to Simplex ETF} properties of \NC, respectively, where Simplex ETF can be viewed as an instance of the uniformity property under the special case of $d \ge K$. A precise characterization of the \emph{uniformity} for the case $d \ll K$, in terms of landscape analysis and global optimality, could be an important question for further investigation.
%has been absent in the literature, leaving the global optimality and global landscape as open problems. 
%\cy{A preliminary writeup for your reference} \qq{very well written, little to change :)}

%============

%Could one envision similar tendencies in deep neural
%networks handling regression or synthesis tasks? Indeed, what is
%the parallel of the ideal classification to this breed of networks?
%These are important open questions that should be addressed, in
%our quest to demystify neural network solutions of inverse problems, generative adversarial networks, and more

%\zz{Not sure for language processing, but I don't see similar patterns in regression problems.}

%\paragraph{Optimization}

%\paragraph{Generalization}

%\paragraph{Robustness}

% \cy{Contrastive learning operates in the regime of $d \ll K$, where the features cannot form a Simplex ETF. Nonetheless, the features possess some ``uniformity'' properties that seem to generalize the notion of ETF \cite{wang2020understanding,chen2020intriguing}. May be of interest to generalize the study to those cases as future work. } 
% \qq{very good point}

\section*{Acknowledgment} ZZ acknowledges support from NSF grant CCF 2008460. 
%CY acknowledges support from Tsinghua-Berkeley Shenzhen Institute Research Fund. 
XL and QQ acknowledge support from NSF grant DMS 2009752. JS acknowledges support from NSF grant CCF 2007649. We would like to thank %\js{for consistency, I would remove these "Dr." and "Prof" from the acknowledges, and simply have their names.} Dr. 
Qinqing Zheng (Facebook AI Research), Vardan Papyan (U. Toronto), and Felix Yu (Google Research) for timely pointing us to some important references and valuable feedback on the final draft. We thank Christina Baek (UC Berkeley) and Sam Buchanan (Columbia U.) for fruitful discussions during various stages of the work. We also thank Zhexin Wu (U. Michigan) for proofreading and pointing out several typos in the draft.
%QQ also acknowledges support of Moore-Sloan fellowship, 

\newpage 

{\small 
\bibliographystyle{unsrt}
\bibliography{biblio/optimization,biblio/learning}
}
%\section*{Appendix}
%\newpage 

%\bibliographystyle{ieeetr}
%\pagenumbering{arabic}

\newpage 

\appendices

\paragraph{Notations and Organizations.}

To begin, we first briefly introduce some notations used throughout the appendix. For a scalar function $f(\mZ)$ with a variable $\mZ\in\R^{K\times N}$, its gradient is a $K\times N$ matrix whose $(i,j)$-th entry is $[\nabla f(\mZ)]_{ij} = \frac{\partial f(\mZ)}{z_{ij}}$ for all $i\in[K],j\in[N]$, where $z_{ij}$ represents the $(i,j)$-th entry of $\mZ$. The Hessian of $f(\mZ)$ can be viewed as an $KN\times KN$ matrix by vectorizing the matrix $\mZ$. An alternative way to present the Hessian is by a bilinear form defined via $[\nabla^2 f(\mZ)](\mA,\mB) = \sum_{i,j,k,\ell} \frac{\partial^2 f(\mZ)}{\partial z_{ij}z_{k\ell}} a_{ij}b_{k\ell}$ for any $\mA,\mB\in\R^{K\times N}$, which avoids the procedure of vectorizing the variable $\mZ$. We will use the bilinear form for the Hessian in \Cref{sec:appendix-prf-global-geometry}. 

The appendix is organized as follows. In Appendix \ref{app:basics}, we introduce the basic definitions and inequalities used throughout the appendices. In Appendix \ref{app:thm-global}, we provide a detailed proof for \Cref{thm:global-minima}, showing that the Simplex ETFs are the \emph{only} global minimizers to our regularized cross-entropy loss. Finally, in Appendix \ref{sec:appendix-prf-global-geometry}, we present the whole proof for \Cref{thm:global-geometry} that the function is a strict saddle function and no spurious local minimizers exist, which is one of the major contributions of the work.

\section{Basics}\label{app:basics}

%\subsection{Notations}
%\paragraph{Notations.} To begin, we first briefly introduce some notations used throughout the appendix. For a scalar function $f(\mZ)$ with a variable $\mZ\in\R^{K\times N}$, its gradient is a $K\times N$ matrix whose $(i,j)$-th entry is $[\nabla f(\mZ)]_{ij} = \frac{\partial f(\mZ)}{z_{ij}}$ for all $i\in[K],j\in[N]$, where $z_{ij}$ represents the $(i,j)$-th entry of $\mZ$. The Hessian of $f(\mZ)$ can be viewed as an $KN\times KN$ matrix by vectorizing the matrix $\mZ$. An alternative way to present the Hessian is by a bilinear form defined via $[\nabla^2 f(\mZ)](\mA,\mB) = \sum_{i,j,k,\ell} \frac{\partial^2 f(\mZ)}{\partial z_{ij}z_{k\ell}} a_{ij}b_{k\ell}$ for any $\mA,\mB\in\R^{K\times N}$, which avoids the procedure of vectorizing the variable $\mZ$. We will use the bilinear form for the Hessian in \Cref{sec:appendix-prf-global-geometry}. 

%\subsection{Basics}

\begin{definition}[$K$-Simplex ETF]\label{def:simplex-ETF}
A standard Simplex ETF is a collection of points in $\bb R^K$ specified by the columns of
\begin{align*}
    \mb M \;=\;  \sqrt{\frac{K}{K-1}}  \paren{ \mb I_K - \frac{1}{K} \mb 1_K \mb 1_K^\top },
\end{align*}
where $\mb I_K \in \bb R^{K \times K}$ is the identity matrix, and $\mb 1_K \in \bb R^K$ is the all ones vector. In the other words, we also have
\begin{align*}
    \mb M^\top \mb M \;=\; \mb M \mb M^\top \;=\; \frac{K}{K-1}  \paren{ \mb I_K - \frac{1}{K} \mb 1_K \mb 1_K^\top }.
\end{align*}

As in \cite{papyan2020prevalence,fang2021layer}, in this paper we consider general Simplex ETF as a collection of points in $\R^d$ specified by the columns of $\sqrt{\frac{K}{K-1}} \mP \paren{ \mb I_K - \frac{1}{K} \mb 1_K \mb 1_K^\top }$, where $\mP\in\R^{d\times K} (d\ge K)$ is an orthonormal matrix, i.e., $\mP^\top \mP = \mb I_K$. 

\end{definition}

\begin{lemma}[Young's Inequality]\label{lem:young-inequality}
Let $p,q$ be positive real numbers satisfying $\frac{1}{p}+\frac{1}{q} = 1$. Then for any $a,b \in \bb R$, we have
\begin{align*}
    \abs{ab} \; \leq\; \frac{ \abs{a}^p }{p}\;+\; \frac{ \abs{b}^q }{q},
\end{align*}
where the equality holds if and only if $\abs{a}^p=\abs{b}^q$. The case $p=q=2$ is just the AM-GM inequality for $a^2,\;b^2$: $\abs{ab}\leq \frac{1}{2}\paren{a^2+b^2}$, where the equality holds if and only if $\abs{a} = \abs{b}$.
\end{lemma}

The following Lemma extends the standard variational form of the nuclear norm.

\begin{lemma} \label{lem:nuclear-norm}
For any fixed $\mZ\in\R^{K\times N}$ and $\alpha>0$, we have
%\js{and for any $\alpha$?} we have
\begin{align}\label{eqn:nuclear-norm-eq}
    \norm{\mb Z}{*} \;=\; \min_{ \mb Z = \mb W \mb H } \; \frac{1}{ 2 \sqrt{ \alpha } } \paren{ \norm{\mb W}{F}^2 + \alpha \norm{\mb H}{F}^2  }.
\end{align}
Here, $\norm{\mZ}{*}$ denotes the nuclear norm of $\mb Z$:
\begin{align*}
    \norm{\mb Z}{*} \;:=\; \sum_{k=1}^{ \min\Brac{K,N}} \sigma_{k}(\mb Z) = \trace\paren{\mb \Sigma},\quad \text{with}\quad \mb Z \;=\; \mb U \mb \Sigma \mb V^\top,
\end{align*}
where $\Brac{\sigma_{k}}_{k=1}^{ \min\Brac{K,N} }$ denotes the singular values of $\mb Z$, and $\mb Z =\mb U \mb \Sigma \mb V^\top$ is the singular value decomposition (SVD) of $\mb Z$.
\end{lemma}

\begin{proof}[Proof of Lemma \ref{lem:nuclear-norm}] 
Let $\mb Z = \mb U \mb \Sigma \mb V^\top$ be the SVD of $\mb Z$. First of all, by the facts that $\mb U^\top \mb U = \mb I$, $\mb V^\top\mb V = \mb I$, $\trace{\mb A^\top \mb A} = \norm{\mb A}{F}^2$ for any $\mb A \in \bb R^{n_1 \times n_2} $, and  cyclic permutation invariance of $\trace\paren{\cdot}$, we have
\begin{align*}
    \norm{\mb Z}{*} \;=\; \trace\paren{ \mb \Sigma } \;&=\; \frac{1}{2\sqrt{\alpha}} \trace\paren{ \sqrt{\alpha}\mb U^\top \mb U \mb \Sigma   } + \frac{\sqrt{\alpha}}{2} \trace\paren{ \frac{1}{\sqrt{\alpha}} \mb \Sigma \mb V^\top\mb V  } \\
    \;&=\; \frac{1}{2\sqrt{\alpha}}  \paren{  \norm{ \alpha^{1/4} \mb U \mb \Sigma^{1/2}  }{F}^2 + \alpha \norm{ \alpha^{-1/4} \mb \Sigma^{1/2} \mb V^\top   }{F}^2 }.
\end{align*}
This implies that there exists some $\mb W =  \alpha^{1/4} \mb U \mb \Sigma^{1/2}$ and $\mb H = \alpha^{-1/4} \mb \Sigma^{1/2} \mb V^\top $, such  that $\norm{\mb Z}{*} = \frac{1}{ 2 \sqrt{ \alpha } } \paren{ \norm{\mb W}{F}^2 + \alpha \norm{\mb H}{F}^2  }$. This equality further implies that 
\begin{align}\label{eqn:lower-bound-Z}
    \norm{\mb Z}{*}\;\geq \; \min_{ \mb Z = \mb W \mb H } \; \frac{1}{ 2 \sqrt{ \alpha } } \paren{ \norm{\mb W}{F}^2 + \alpha \norm{\mb H}{F}^2  }.
\end{align}
On the other hand, for any $\mb W\mb H = \mb Z$, we have
\begin{align*}
    \norm{\mb Z}{*} \;&=\; \trace \paren{ \mb \Sigma } \;=\; \trace \paren{ \mb U^\top \mb Z \mb V } \;=\; \trace\paren{ \mb U^\top \mb W \mb H \mb V } \\
    \;&\leq\; \frac{1}{2\sqrt{\alpha}} \norm{\mb U^\top \mb W}{F}^2 + \frac{ \sqrt{\alpha} }{2} \norm{ \mb H \mb V }{F}^2 \;\leq\; \frac{1}{2\sqrt{\alpha}} \paren{ \norm{\mb W}{F}^2 + \alpha \norm{\mb H}{F}^2 }, 
\end{align*}
where the first inequality utilize the Young's inequality in Lemma \ref{lem:young-inequality} that $\abs{\trace(\mA\mB)}\le \frac{1}{2c}\norm{\mA}{F}^2 + \frac{c}{2}\norm{\mB}{F}^2$ for any $c>0$ and $\mA,\mB$ of appropriate dimensions, %\js{This is not CW. This is just expanding the binomial, isn't it?}\zz{Again, it's Young's inequality which I was always confused with the CW.} 
and the last inequality follows because $\norm{\mb U}{}= 1$ and $\norm{\mb V}{}= 1$. 
%\js{isn't $\|\mb U\|=\|\mb V\|=1$?}\zz{yes, corrected.}. 
Therefore, we have 
\begin{align}\label{eqn:upper-bound-Z}
    \norm{\mb Z}{*} \;\leq\; \min_{ \mb Z = \mb W \mb H } \; \frac{1}{ 2 \sqrt{ \alpha } } \paren{ \norm{\mb W}{F}^2 + \alpha \norm{\mb H}{F}^2  }.
\end{align}
Combining the results in \eqref{eqn:lower-bound-Z} and \eqref{eqn:upper-bound-Z}, we complete the proof.
\end{proof}

\begin{lemma}[Theorem 7.2.6 of \cite{horn2012matrix}]\label{lem:matrix-analysis-horn}
Let $\mb A \in \bb R^{n \times n}$ be a symmetric positive semidefinite matrix.
Then for any fixed $k \in \Brac{2,3,\cdots}$, there exists a unique real symmetric positive semidefinte matrix $\mb B$ such that $\mb B^k = \mb A$.
\end{lemma}

\section{Proof of \Cref{thm:global-minima}}\label{app:thm-global}

In this part of appendices, we prove \Cref{thm:global-minima} in \Cref{sec:main} that we restate as follows.

\begin{theorem}[Global Optimality Condition]\label{thm:global-minima-app}
	Assume that the feature dimension $d$ is larger than the number of classes $K$, i.e., $d> K$. Then any global minimizer $(\mW^\star, \mH^\star,\vb^\star)$ of 
	\begin{align}\label{eq:obj-app}
     \min_{\mb W , \mb H,\mb b  } \; f(\mb W,\mb H,\mb b) \;:=\; g(\mW\mH + \vb\vone^\top) \;+\; \frac{\lambda_{\mb W} }{2} \norm{\mb W}{F}^2 + \frac{\lambda_{\mb H} }{2} \norm{\mb H}{F}^2 + \frac{\lambda_{\mb b} }{2} \norm{\mb b}{2}^2
\end{align}
with 
\begin{align}\label{eqn:g-lce-app}
    g(\mW\mH + \vb\vone^\top) := \frac{1}{N}\sum_{k=1}^K\sum_{i=1}^n \Lce(\mW\vh_{k,i} + \vb,\vy_k),
\end{align}
obeys the following 
\begin{align*}
    & \norm{\mb w^\star}{2}\;=\; \norm{\mb w^{\star 1} }{2} \;=\; \norm{\mb w^{\star 2}}{2} \;=\; \cdots \;=\; \norm{\mb w^{\star K} }{2}, \quad \text{and}\quad \mb b^\star = b^\star \mb 1, \\ 
    & \vh_{k,i}^\star \;=\;  \sqrt{ \frac{ \lambda_{\mb W}  }{ \lambda_{\mb H} n } } \vw^{\star k} ,\quad \forall \; k\in[K],\; i\in[n],  \quad \text{and} \quad  \ol{\mb h}_{i}^\star \;:=\; \frac{1}{K} \sum_{j=1}^K \mb h_{j,i}^\star \;=\; \mb 0, \quad \forall \; i \in [n],
\end{align*}
where either $b^\star = 0$ or $\lambdab=0$, and the matrix $\frac{1}{\norm{\mb w^\star}{2}}  \mW^{\star\top} $ forms a $K$-simplex ETF defined in Definition \ref{def:simplex-ETF} in the sense that 
\begin{align*}
    \frac{1}{\norm{\mb w^\star}{2}^2}  \mW^{\star\top} \mW^{\star}\;=\; \frac{K}{K-1}  \paren{ \mb I_K - \frac{1}{K} \mb 1_K \mb 1_K^\top }.
\end{align*}
\end{theorem}

\subsection{Main Proof}

Similar to the proofs in \cite{lu2020neural,fang2021layer}, we prove the theorem by directly showing that $f(\mW,\mH,\vb)> f(\mW^\star,\mH^\star,\vb^\star)$ for any $(\mW,\mH, \vb)$ not in the form as shown in \Cref{thm:global-minima-app}.

\vspace{0.1in}

\begin{proof}[Proof of \Cref{thm:global-minima-app}] 
First note that the objective function $f$ is \emph{coercive}\footnote{A function $f:\bb R^n \mapsto \bb R$ is coercive if $f(\mb x) \rightarrow +\infty$ as $\norm{\mb x}{2} \rightarrow +\infty $.}  due to the weight decay regularizers and the fact that the CE loss is always non-negative. %\js{I don't think so: if data is separable, the minimizer diverges as $g\to 0$. $f$ is coercive because the regularizers ($\|\cdot\|_F$) are coercive.} \zz{Correct and add the weight decay regularizers.}. 
This implies that the global minimizer of $f(\mb W,\mb H,\mb b)$ in \eqref{eq:obj-app} is always finite. By Lemma \ref{lem:critical-balance}, we know that any critical point $(\mb W,\mb H,\mb b)$ of $f$ in \eqref{eq:obj-app} satisfies 
\begin{align*}
    \mW^\top\mW = \frac{\lambda_{\mH}}{\lambda_{\mW}}\mH\mH^\top. 
\end{align*}
For the rest of the proof, to simplify the notations, let $\norm{\mW}{F}^2 = \rho$, and thus $\norm{\mH}{F}^2 = \frac{\lambda_{\mH}}{\lambda_{\mW}}\rho$. 

%\paragraph{The global solutions can be attained by the $K$-Simplex ETF.} 
We will first provide a lower bound for the cross-entropy term $g(\mW\mH + \vb\vone^\top)$ for any $\mW$ with energy $\rho$, and then show that the lower bound is attained if and only if the parameters are in the form described in \Cref{thm:global-minima-app}. Now, by Lemma \ref{lem:lower-bound-g}, we know that for any $c_1,c_3 >0$,
\begin{align*}
   g(\mb W \mb H + \mb b \mb 1^\top) \;\geq\; - \frac{\rho }{(1+c_1)(K-1)} \sqrt{ \frac{ \lambda_{\mb W}  }{ \lambda_{\mb H} n } } \;+\; c_2
\end{align*}
with $c_2 = \frac{1}{1+c_1} \log\paren{ (1+c_1)(K-1) } + \frac{c_1}{1+c_1} \log\paren{ \frac{1+c_1}{c_1} } $. Therefore, we have
\begin{align*}
    f(\mb W,\mb H,\mb b)\;&=\;  g(\mb W \mb H + \mb b \mb 1^\top) \;+\; \frac{\lambda_{\mb W} }{2} \norm{\mb W}{F}^2 + \frac{\lambda_{\mb H} }{2} \norm{\mb H}{F}^2 + \frac{\lambda_{\mb b} }{2} \norm{\mb b}{2}^2 \\
%    \;&\geq\; \underbrace{ - \frac{ \rho }{ (1+c_1)(K-1) } \sqrt{ \frac{ \lambda_{\mb W} }{ \lambda_{\mb H}  n } } + c_2 + \frac{\lambda_{\mb W}}{2} \rho + \frac{\lambda_{\mb H}^2 }{2\lambda_{\mb W}} \rho }_{ \xi\paren{ \rho, \lambda_{\mb W},\lambda_{\mb H} } }  + \frac{\lambda_{\mb b} }{2} \norm{\mb b}{2}^2 \\
 \;&\geq\; \underbrace{ - \frac{ \rho }{ (1+c_1)(K-1) } \sqrt{ \frac{ \lambda_{\mb W} }{ \lambda_{\mb H}  n } } + c_2 + \lambda_{\mb W} \rho  }_{ \xi\paren{ \rho, \lambda_{\mb W},\lambda_{\mb H} } }  + \frac{\lambda_{\mb b} }{2} \norm{\mb b}{2}^2 \\
    \;&\geq\; \xi\paren{ \rho, \lambda_{\mb W},\lambda_{\mb H} },
\end{align*}
where the last inequality becomes an equality whenever either $\lambdab = 0$ or $\vb = \vzero$. %\js{By the way, typically one never regularizers the bias terms (i.e. $\lambda_b=0$). I'm wondering if we should include this term in $f$ at all.}\zz{Good point. I am not sure if the bias term is regularized or not in practice. Tianyu told me the default implementation in Pytorch does include the regularizer for the bias term. But I am also thinking about whether including the bias term in $f$. On one hand, our result covers the case without the regularizer for the bias term. But on the other hand, our result indicate the bias term in the last layer has no much effect, which maybe a little bit surprising. The results in CIFAR10 and MNIST indeed show that there is no difference in terms of both training and testing without the bias term}. 
Furthermore, by Lemma \ref{lem:lower-bound-equality-cond}, we know that the inequality $f(\mb W,\mb H,\mb b)\;\geq\; \xi\paren{ \rho, \lambda_{\mb W},\lambda_{\mb H} }$ becomes an equality \emph{if and only if} $(\mb W,\mb H,\mb b)$ satisfy the following
\begin{itemize}
    \item[(a)] $\norm{\mb w }{2} \;=\; \norm{\mb w^1}{2} \;=\; \norm{\mb w^2}{2} \;=\; \cdots \;=\; \norm{\mb w^K}{2}$;
    \item[(b)] $\mb b = b \mb 1$, where either $b = 0$ or $\lambdab = 0$;
    \item[(c)] $\ol{\mb h}_{i} \;:=\; \frac{1}{K} \sum_{j=1}^K \mb h_{j,i} \;=\; \mb 0, \quad \forall \; i \in [n]$, and $\sqrt{ \frac{ \lambda_{\mb W}  }{ \lambda_{\mb H} n } } \vw^k \;=\; \vh_{k,i},\quad \forall \; k\in[K],\; i\in[n]$;
    \item[(d)] $\mb W \mb W^\top \;=\; \frac{\rho }{K-1} \paren{ \mb I_K - \frac{1 }{K} \mb 1_K \mb 1_K^\top }$;
    \item[(e)] $c_1 \;=\; \brac{(K-1)\exp\bracket{-\frac{\rho}{K-1}\sqrt{\frac{\lambdaW}{\lambdaH n}}}}^{-1}$.
\end{itemize}
To finish the proof, we only need to show that $\rho = \norm{\mb W}{F}^2$ must be finite for any fixed $\lambdaW,\lambdaH>0$. From (e), we know that $c_1 \;=\; \brac{(K-1)\exp\bracket{-\frac{\rho}{K-1}\sqrt{\frac{\lambdaW}{\lambdaH n}}}}^{-1}$ is an \emph{increasing} function in terms of $\rho$, and $c_2 = \frac{1}{1+c_1}\log\bracket{(1+c_1)(K-1)} +  \frac{c_1}{1+c_1}\log\bracket{\frac{1+c_1}{c_1}}$ is a \emph{decreasing} function in terms of $c_1$. Therefore, we observe the following:
\begin{itemize}
    \item When $\rho\rightarrow 0$, we have $c_1 \rightarrow \frac{1}{K-1}$ and $c_2 \rightarrow \log K$, so that 
    \begin{align*}
        \lim_{\rho \rightarrow 0} \xi (\rho;\lambdaW,\lambdaH) \;=\;  \lim_{\rho \rightarrow 0} c_2(\rho) \;=\; \log K. 
    \end{align*}
    \item On the other hand, when $\rho\rightarrow +\infty$, $c_1 \rightarrow +\infty$ and $c_2 \rightarrow 0$, so that $ \xi(\rho;\lambdaW,\lambdaH) \rightarrow +\infty$ as $\rho\rightarrow +\infty$.
\end{itemize}
\begin{figure}[htp!]
	\centering
    \includegraphics[width = 0.5\textwidth]{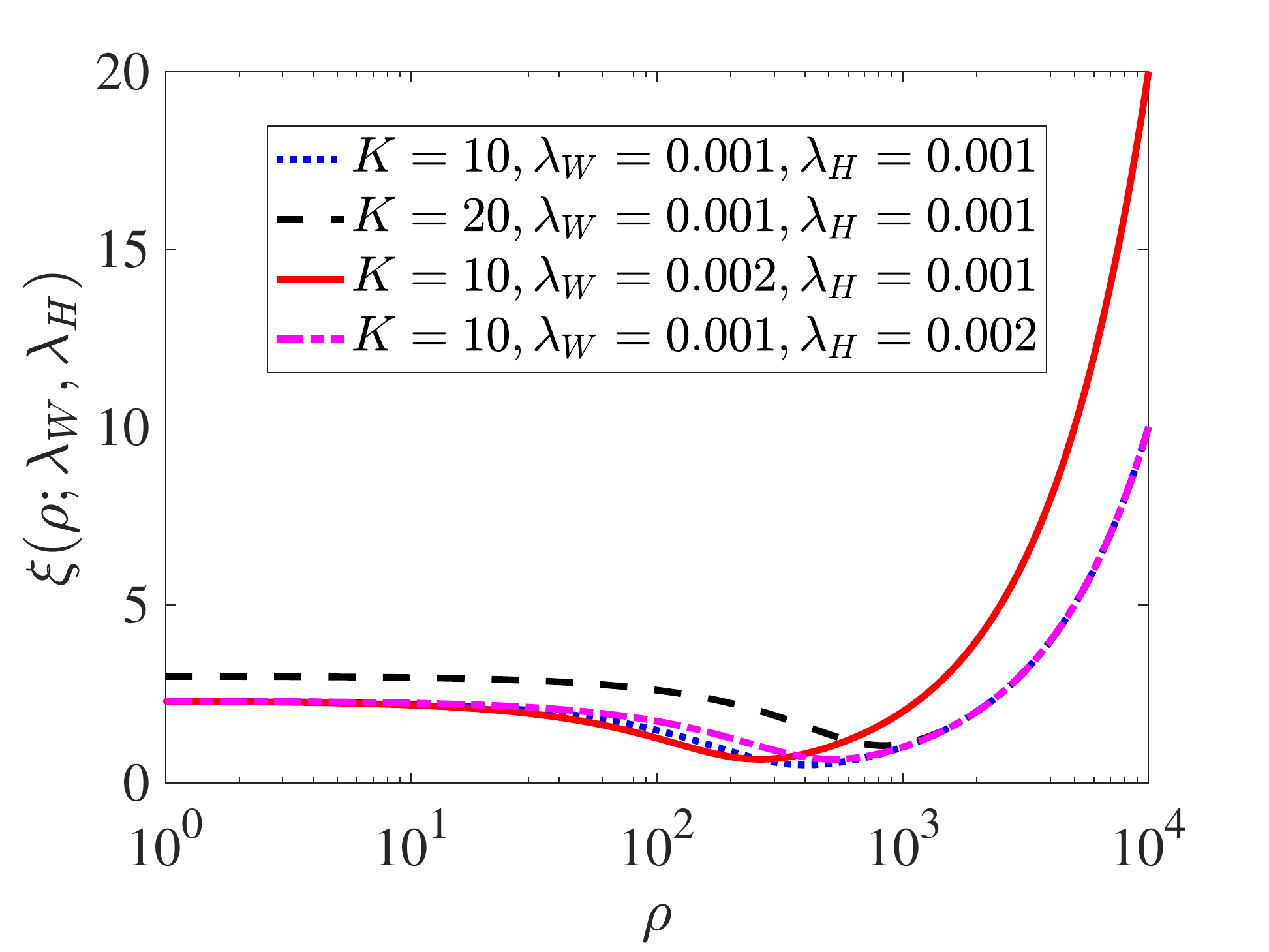}
	\caption{Plot of $\xi(\rho;\lambdaW,\lambdaH)$ in terms of $\rho$ for $n = 100$ and different $K,\lambdaW,\lambdaH$.}
	\label{fig:obj-lower-bound}
\end{figure}
Since $\xi(\rho;\lambdaW,\lambdaH)$ %\js{why do we need the $\to +\infty$ here?}\zz{rephrased this sentence.} 
is a continuous function of $\rho \in [0,+\infty)$ and $\xi(\rho;\lambdaW,\lambdaH)\rightarrow +\infty$ and $\rho \rightarrow +\infty$, these further imply that $\xi(\rho;\lambdaW,\lambdaH)$ achieves its minimum at a finite $\rho$ (see \Cref{fig:obj-lower-bound} for an example. %\qq{it could be better to plot more instances in one plot)}). 
This finishes the proof.
\end{proof}
%$(ii)$ $c_1  \rightarrow \frac{1}{K-1}$ and $c_2 \rightarrow \log K$ when $\rho\rightarrow 0$, and $(iii)$ $c_1  \rightarrow \infty$ and $c_2 \rightarrow 0$ when $\rho\rightarrow +\infty$. Thus, as plotted in \Cref{fig:obj-lower-bound}, $\xi(\rho;\lambdaW,\lambdaH) \rightarrow \log K$ when $\rho\rightarrow 0$ and $\xi(\rho;\lambdaW,\lambdaH) \rightarrow \infty$ when $\rho\rightarrow \infty$, and $\xi(\rho;\lambdaW,\lambdaH)$ achieves its minimum at a finite $\rho$.  One can verify that all the inequalities in \eqref{eq:cross-entropy-lower-bound} and \eqref{eq:all-cross-entropy-lower-bound} (and hence \eqref{eq:obj-lower-boound}) achieve equality when $(\mW,\mH,\vb)$ have the form as in \Cref{thm:global-minima}. 

\subsection{Supporting Lemmas}

We first characterize the following balance property between $\mW$ and $\mH$ for any critical point  $(\mW,\mH,\vb)$ of our loss function:

\begin{lemma}\label{lem:critical-balance} 
Let $\rho = \norm{\mb W }{F}^2$. Any critical point $(\mW,\mH,\vb)$ of \eqref{eq:obj-app} obeys
\begin{align}\label{eq:critical-balance}
    \mW^\top\mW \;=\; \frac{\lambda_{\mH}}{\lambda_{\mW}}\mH\mH^\top\quad \text{and}\quad \rho \;=\; \norm{\mb W}{F}^2 \;=\; \frac{\lambda_{\mH}}{\lambda_{\mW}} \norm{\mb H}{F}^2.
\end{align}

\end{lemma}

\begin{proof}[Proof of Lemma \ref{lem:critical-balance}]
By definition, any critical point $(\mW,\mH,\vb)$ of \eqref{eq:obj-app} satisfies the following:
\begin{align}
    	\nabla_{\mW}f(\mW,\mH,\vb) \;&=\; \nabla_{\mb Z = \mb W\mb H}\; g(\mW\mH + \vb \vone^\top)\mH^\top + \lambdaW \mW \;=\; \vzero,\label{eqn:W-crtical} \\
	\nabla_{\mH}f(\mW,\mH,\vb) \;&=\;  \mW^\top \nabla_{\mb Z = \mb W\mb H}\; g(\mW\mH + \vb\vone^\top) + \lambdaH \mH \;=\; \vzero. \label{eqn:H-crtical}
\end{align}
Left multiply the first equation by $\mb W^\top$ on both sides and then right multiply second equation by $\mH^\top$ on both sides, it gives
\begin{align*}
    \mb W^\top \nabla_{\mb Z = \mb W\mb H}\; g(\mW\mH + \vb\vone^\top)\mH^\top \;&=\; - \lambdaW \mb W^\top \mW , \\
     \mb W^\top \nabla_{\mb Z = \mb W\mb H}\; g(\mW\mH + \vb\vone^\top)\mH^\top  \;&=\; - \lambdaH \mb H^\top \mH.
\end{align*}
Therefore, combining the equations above, we obtain
\begin{align*}
    \lambdaW \mW^\top \mW \;=\; \lambdaH \mH \mH^\top.
\end{align*}
Moreover, we have
\begin{align*}
    \rho\;=\; \norm{\mb W}{F}^2 \;=\; \trace \paren{ \mb W^\top \mb W } \;=\;  \frac{ \lambda_{\mb H} }{ \lambda_{\mb W} } \trace \paren{ \mb H \mb H^\top } \;=\; \frac{ \lambda_{\mb H} }{ \lambda_{\mb W} } \trace \paren{ \mb H^\top\mb H  } \;=\; \frac{ \lambda_{\mb H} }{ \lambda_{\mb W} } \norm{\mb H}{F}^2,
\end{align*}
as desired.
\end{proof}

\begin{lemma}\label{lem:lower-bound-g}
Let $\mb W = \begin{bmatrix} (\mb w^1)^\top \\ \vdots \\ (\mb w^K)^\top 
\end{bmatrix}\in \bb R^{K \times d}$, $\mH = \begin{bmatrix}\vh_{1,1} \cdots \vh_{K,n} \end{bmatrix}\in \bb R^{d \times N}$, $N= nK$, and $\rho = \norm{\mb W}{F}^2$. Given  $g(\mW\mH + \vb\vone^\top)$ defined in \eqref{eqn:g-lce-app}, for any critical point $(\mW,\mH,\vb)$ of \eqref{eq:obj-app} and $c_1>0$, we have
\begin{align}\label{eqn:lower-bound-g}
   g(\mb W \mb H + \mb b \mb 1^\top) \;\geq\; - \frac{\rho }{(1+c_1)(K-1)} \sqrt{ \frac{ \lambda_{\mb W}  }{ \lambda_{\mb H} n } } \;+\; c_2,
\end{align}
with $c_2 = \frac{1}{1+c_1} \log\paren{ (1+c_1)(K-1) } + \frac{c_1}{1+c_1} \log\paren{ \frac{1+c_1}{c_1} } $. 
\end{lemma}

\begin{proof}[Proof of Lemma \ref{lem:lower-bound-g}]
By Lemma \ref{lem:lce-lower-bound} with $\mb z_{k,i} = \mW\vh_{k,i} + \vb$, since the scalar $c_1>0$ can be arbitrary, we choose the same $c_1$ and $c_2$ for all $i\in[n]$ and $k \in [K]$, we have the following lower bound for $ g(\mb W \mb H + \mb b \mb 1^\top)$ as 
\begin{align}
    & (1+c_1)(K-1) \brac{g(\mb W \mb H + \mb b \mb 1^\top) -c_2  }  \nonumber  \\
    \;=\;&(1+c_1)(K-1)\brac{ \frac{1}{N}\sum_{k=1}^K\sum_{i=1}^n \Lce(\mW\vh_{k,i} + \vb,\vy_k) - c_2  }  \nonumber \\
     \;\geq\;& \frac{1}{ N} \sum_{k=1}^K\sum_{i=1}^n \Bigg[ \sum_{j=1}^K \paren{  \mb h_{k,i}^\top \mb w^j +b_j} - K \paren{  \mb h_{k,i}^\top \mb w^k + b_k } \Bigg]  \label{eqn:g-lower-1} \\
     \;=\;&  \frac{1}{ N} \sum_{i=1}^n\Bigg[  \paren{\sum_{k=1}^K \sum_{j=1}^K \mb h_{k,i}^\top \mb w^j - K \sum_{k=1}^K  \mb h_{k,i}^\top \mb w^k } \;+\; \underbrace{\sum_{k=1}^K \sum_{j=1}^K\paren{  b_j- b_k } }_{=0}\Bigg]  \nonumber \\
     \;=\;& \frac{1}{ N} \sum_{i=1}^n  \paren{\sum_{k=1}^K \sum_{j=1}^K  \mb h_{j,i}^\top  \mb w^k - K \sum_{k=1}^K  \mb h_{k,i}^\top \mb w^k }  \nonumber \\
     \;=\;& \frac{K}{ N} \sum_{i=1}^n \sum_{k=1}^K  \brac{ \bigg( \frac{1}{K} \sum_{j=1}^K (\mb h_{j,i} - \mb h_{k,i}) \bigg)^\top \mb w^k } \;=\; \frac{1}{ n} \sum_{i=1}^n \sum_{k=1}^K  \paren{ \ol{\mb h}_{i} - \mb h_{k,i} }^\top \mb w^k,  \nonumber
\end{align}
where for the last equality we let $\ol{\mb h}_{i} = \frac{1}{K} \sum_{j=1}^K \mb h_{j,i}$. Furthermore, from the AM-GM inequality in Lemma \ref{lem:young-inequality}, we know that for any $\mb u,\mb v \in \bb R^K$ and any $c_3>0$,
\begin{align*}
    \mb u^\top \mb v \;\leq\; \frac{c_3}{2} \norm{\mb u}{2}^2 \;+\; \frac{1}{2c_3} \norm{\mb v}{2}^2,
\end{align*}
where the inequality becomes an equality when $c_3\vu= \vv$. Thus, we further have
\begin{align*}
     &(1+c_1)(K-1) \brac{g(\mb W \mb H + \mb b \mb 1^\top) -c_2  } \\
     \;\geq\;& - \frac{c_3}{ 2} \sum_{k=1}^K \norm{\mb w^k}{2}^2 \;-\; \frac{1}{2c_3n} \sum_{i=1}^n \sum_{k=1}^K \norm{ \ol{\mb h}_i - \mb h_{k,i} }{2}^2 \\
     \;=\;& - \frac{c_3}{ 2} \sum_{k=1}^K \norm{\mb w^k}{2}^2 \;-\ \frac{1}{2c_3n} \sum_{i=1}^n \brac{ \paren{\sum_{k=1}^K \norm{ \mb h_{k,i} }{2}^2} - K \norm{ \ol{\mb h}_i }{2}^2 } \\
     \;=\;& - \frac{c_3}{ 2} \norm{\mb W}{F}^2 \;-\ \frac{1}{2c_3n}  \paren{ \norm{\mb H}{F}^2 - K  \sum_{i=1}^n \norm{ \ol{\mb h}_i }{2}^2 },  
\end{align*}
where the first inequality becomes an equality if and only if 
\begin{align}\label{eqn:lower-bound-attained-1}
    c_3\vw^k = (\vh_{k,i} - \overline\vh_i),\quad \forall\; k\in[K],\;i\in[n].
\end{align}
Let $\rho = \norm{\mb W}{F}^2$. Now, by using Lemma \ref{lem:critical-balance}, we have $\mb W^\top \mb W = \frac{ \lambda_{\mb H} }{ \lambda_{\mb W} } \mb H \mb H^\top  $, so that $\norm{\mb H}{F}^2 = \trace\paren{ \mb H \mb H^\top } =  \frac{ \lambda_{\mb W} }{ \lambda_{\mb H} } \trace \paren {\mb W^\top \mb W} =  \frac{ \lambda_{\mb W} }{ \lambda_{\mb H} } \rho $. Therefore, we have
\begin{align}\label{eqn:g-lower-2}
   g(\mb W \mb H + \mb b \mb 1^\top) \geq - \frac{\rho }{2(1+c_1)(K-1)} \paren { c_3  + \frac{ \lambda_{\mb W} }{ \lambda_{\mb H} }  \frac{1}{c_3 n }   } + c_2,
\end{align}
as desired. The last inequality achieves its equality if and only if 
\begin{align}\label{eqn:lower-bound-attained-2}
    \overline \vh_i = \vzero,\quad \forall \; i\in[n].
\end{align}

Plugging this into \eqref{eqn:lower-bound-attained-1}, we have
\begin{align*}
    c_3\vw^k \;=\; \vh_{k,i} \Rightarrow c_3^2 \;=\; \frac{ \sum_{i=1}^n\sum_{k=1}^K\norm{\vh_{k,i}}{2}^2 }{ n \sum_{k=1}^K\norm{\vw^k}{2}^2 } = \frac{\norm{\mH}{F}^2}{n\norm{\mW}{F}^2} = \frac{\lambdaW}{n\lambdaH}.
\end{align*}
This together with the lower bound in \eqref{eqn:g-lower-2} gives
\begin{align*}
   g(\mb W \mb H + \mb b \mb 1^\top) \geq - \frac{\rho }{(1+c_1)(K-1)} \sqrt{ \frac{ \lambda_{\mb W}  }{ \lambda_{\mb H} n } }  + c_2,
\end{align*}
as suggested in \eqref{eqn:lower-bound-g}.
\end{proof}

Next, we show that the lower bound in \eqref{eqn:lower-bound-g} is attained if and only if $(\mb W, \mb H,\mb b)$ satisfies the following conditions.  
\begin{lemma}\label{lem:lower-bound-equality-cond}
Under the same assumptions of Lemma \ref{lem:lower-bound-g},
the lower bound in \eqref{eqn:lower-bound-g} is attained for any critical point $(\mb W,\mb H,\mb b)$ of \eqref{eq:obj-app}
if and only if the following hold
\begin{align*}
    & \norm{\mb w^1}{2} \;=\; \norm{\mb w^2}{2} \;=\; \cdots \;=\; \norm{\mb w^K}{2}, \quad \text{and}\quad \mb b = b \mb 1, \\ 
    &\ol{\mb h}_{i} \;:=\; \frac{1}{K} \sum_{j=1}^K \mb h_{j,i} \;=\; \mb 0, \quad \forall \; i \in [n], \quad \text{and} \quad \sqrt{ \frac{ \lambda_{\mb W}  }{ \lambda_{\mb H} n } } \vw^k \;=\; \vh_{k,i},\quad \forall \; k\in[K],\; i\in[n], \\
    &\mb W \mb W^\top \;=\; \frac{\rho }{K-1} \paren{ \mb I_K - \frac{1 }{K} \mb 1_K \mb 1_K^\top }, \;\text{and}\; c_1 \;=\; \brac{(K-1)\exp\bracket{-\frac{\rho}{K-1}\sqrt{\frac{\lambdaW}{\lambdaH n}}}}^{-1}.
\end{align*}
\end{lemma}
The proof of Lemma \ref{lem:lower-bound-equality-cond} utilizes the conditions in Lemma \ref{lem:lce-lower-bound}, and the conditions \eqref{eqn:lower-bound-attained-1} and \eqref{eqn:lower-bound-attained-2} during the proof of Lemma \ref{lem:lower-bound-g}. 

\begin{proof}[Proof of Lemma \ref{lem:lower-bound-equality-cond}]
From the proof of \ref{lem:lower-bound-g}, if we want to attain the lower bound, we know that we need at least \eqref{eqn:lower-bound-attained-1} and \eqref{eqn:lower-bound-attained-2} to hold, which is equivalent to the following:
\begin{align}\label{eqn:lower-bound-attained-3}
    \ol{\vh}_i \;=\;\frac{1}{K} \sum_{j=1}^K \mb h_{j,i} \;=\; \vzero,\quad \forall i\in [n], \quad \text{and}\quad 
    \sqrt{ \frac{ \lambda_{\mb W}  }{ \lambda_{\mb H} n } }  \vw^k \;=\;  \vh_{k,i},\quad \forall \; k\in [K],\;i\in [n],
\end{align}
which further implies that
\begin{align}\label{eqn:w-sum-zero}
    \sum_{k=1}^K \mb w^k \;=\; \mb 0.
\end{align}
Next, under the condition \eqref{eqn:lower-bound-attained-3}, if we want \eqref{eqn:lower-bound-g} to become an equality, we only need \eqref{eqn:g-lower-1} to become an equality, which is true if and only if the condition \eqref{eqn:lce-lower-bound-equality} in Lemma \ref{lem:lce-lower-bound} holds for $\mb z_{k,i} = \mW\vh_{k,i} + \vb$ for all $i\in[n]$ and $k \in [K]$. First, let $[\mb z_{k,i}]_j = \mb h_{k,i}^\top \mb w^j + b_j $, where we have
\begin{align}
    \sum_{j=1}^K [\mb z_{k,i}]_j \;=\;  \mb h_{k,i}^\top \sum_{j=1}^K \mb w^j  + \sum_{j=1}^K b_j \;&=\; \sqrt{ \frac{ \lambda_{\mb H} n }{ \lambda_{\mb W} } }  \mb h_{k,i}^\top  \sum_{j=1}^K \mb h_{j,i}  + \sum_{j=1}^K b_j \nonumber  \\
    \;&=\;  \sqrt{ \frac{ \lambda_{\mb H} n }{ \lambda_{\mb W} } } K  \mb h_{k,i}^\top  \ol{\mb h}_i + \sum_{j=1}^K b_j \;=\; K \ol{b} \label{eqn:g-lower-3}
\end{align}
with $\ol{b} = \frac{1}{K} \sum_{i=1}^K b_i$, and 
\begin{align}
    K[\mb z_{k,i}]_k \;=\; K \mb h_{k,i}^\top  \mb w^k  + Kb_k \;=\; \sqrt{ \frac{ \lambda_{\mb W}  }{ \lambda_{\mb H} n } } \paren{K\norm{\mb w^k}{2}^2} + Kb_k. \label{eqn:g-lower-4}
\end{align}
Based on \eqref{eqn:g-lower-3}, \eqref{eqn:g-lower-4}, and \eqref{eqn:lce-lower-bound-equality} from Lemma \ref{lem:lce-lower-bound}, we have 
\begin{align}
    c_1 \;&=\; \brac{(K-1) \exp \paren{  \frac{ \paren{\sum_{j=1}^K [\mb z_{k,i}]_j} - K [\mb z_{k,i}]_k}{K-1} }  }^{-1}  \nonumber   \\
    \;&=\; \brac{(K-1)\exp\bracket{\frac{K}{K-1} \paren{ \ol{b} -  \sqrt{ \frac{ \lambda_{\mb W}  }{ \lambda_{\mb H} n } } \norm{\mb w^k}{2}^2 - b_k  } }}^{-1}. \label{eqn:c-form-1}
\end{align}
Since the scalar $c_1>0$ is chosen to be the same for all $k \in [K]$, we have
\begin{align}\label{eqn:g-lower-5}
    \sqrt{ \frac{ \lambda_{\mb W}  }{ \lambda_{\mb H} n } } \norm{\mb w^k}{2}^2 + b_k \;=\; \sqrt{ \frac{ \lambda_{\mb W}  }{ \lambda_{\mb H} n } } \norm{\mb w^\ell}{2}^2 + b_\ell, \quad \forall \; \ell \not =k. 
\end{align}
Second, since $[\mb z_{k,i}]_j = [\mb z_{k,i}]_{\ell}$ for all $\forall j,\ell \neq k,\; k\in[K]$, from \eqref{eqn:lower-bound-attained-3} we have
\begin{align}
     &\mb h_{k,i}^\top \mb w^j  + b_j \;=\;  \mb h_{k,i}^\top \mb w^\ell+ b_\ell, \quad \; \forall j,\ell \neq k,\; k\in[K] \nonumber  \\
     \Longleftrightarrow \quad & \sqrt{ \frac{ \lambda_{\mb W}  }{ \lambda_{\mb H} n } } (\mb w^k)^\top \mb w^j + b_j \;=\; \sqrt{ \frac{ \lambda_{\mb W}  }{ \lambda_{\mb H} n } }  (\mb w^k)^\top \mb w^\ell + b_\ell,\quad  \forall j,\ell \neq k,\; k\in[K]. \label{eqn:g-lower-7}
\end{align}
Based on this and \eqref{eqn:w-sum-zero}, we have
\begin{align}
    \sqrt{ \frac{ \lambda_{\mb W}  }{ \lambda_{\mb H} n } }\norm{\mb w^k}{2}^2 + b_k \;&=\; - \sqrt{ \frac{ \lambda_{\mb W}  }{ \lambda_{\mb H} n } } \sum_{j\not = k} (\mb w^j)^\top \mb w^k +   b_k \nonumber  \\ 
    \;&=\; -(K-1)\sqrt{ \frac{ \lambda_{\mb W}  }{ \lambda_{\mb H} n } }\underbrace{ (\mb w^{\ell})^\top \mb w^k}_{ \ell \not= k, \ell \in [K] } +  \paren{ b_k + \sum_{j\not= \ell, k} \paren{b_\ell - b_j } }  \nonumber  \\
    \;&=\; -(K-1) \sqrt{ \frac{ \lambda_{\mb W}  }{ \lambda_{\mb H} n } } (\mb w^{\ell})^\top \mb w^k + \brac{ 2b_k + (K-1) b_\ell - K \ol{b} } \label{eqn:g-lower-6}
\end{align}
for all $\ell \neq k$. Combining \eqref{eqn:g-lower-5} and \eqref{eqn:g-lower-6}, for all $k,\ell \in [K]$ with $k \not = \ell $ we have 
\begin{align*}
     2b_k + (K-1) b_\ell - K \ol{b}  \;=\; 2b_\ell  + (K-1) b_k - K \ol{b}\quad \Longleftrightarrow \quad b_k \;=\; b_\ell, \; \forall \; k\not = \ell.  
\end{align*}
Therefore, we can write $\mb b = b \mb 1_K$ for some $b>0$. Moreover, since $b_k = b_\ell$ for all $k \not = \ell $, \eqref{eqn:g-lower-5} and \eqref{eqn:g-lower-7}, and \eqref{eqn:g-lower-6} further imply that
\begin{align}
    &\norm{\mb w^1}{2}\;=\; \norm{\mb w^2}{2}\;=\; \cdots \;=\; \norm{\mb w^K}{2},\quad \text{and}\quad \norm{\mb w^k}{2}^2 \;=\; \frac{1}{K} \norm{\mb W}{F}^2 \;=\; \frac{\rho}{K}, \label{eqn:g-lower-8} \\
    & (\mb w^j)^\top \mb w^k \;=\; (\mb w^\ell)^\top \mb w^k \;=\; - \frac{1}{K-1} \norm{\mb w^k}{2}^2\;=\; - \frac{\rho}{K(K-1)}, \quad  \forall j,\ell \neq k,\; k\in[K],\label{eqn:g-lower-9}
\end{align}
where \eqref{eqn:g-lower-9} is equivalent to
\begin{align*}
    \mb W \mb W^\top \;=\; \frac{\rho }{K-1} \paren{ \mb I_K - \frac{1 }{K} \mb 1_K \mb 1_K^\top }.
\end{align*}
Finally, plugging the results in \eqref{eqn:g-lower-8} and \eqref{eqn:g-lower-9} into \eqref{eqn:c-form-1}, we have
\begin{align*}
    c_1  \;=\; \brac{(K-1)\exp\paren{  -  \frac{\rho}{K-1} \sqrt{ \frac{ \lambda_{\mb W}  }{ \lambda_{\mb H} n } }  } }^{-1},
\end{align*}
as desired.
\end{proof}

\begin{lemma}\label{lem:lce-lower-bound}
Let $\mb y_k\in \bb R^K$ be an one-hot vector with the $k$th entry equalling $1$ for some $k\in [K]$. For any vector $\vz\in\R^{K}$ and $c_1>0$, the cross-entropy loss $\Lce (\vz,\vy_k)$ with $\vy_k$ can be lower bounded by
\begin{align}\label{eqn:lce-lower-bound}
     \Lce (\vz,\vy_k) \;\geq \;  \frac{1}{1+c_1}\frac{ \paren{\sum_{i=1}^K z_i} - K z_k }{K-1} + c_2,
\end{align}
where $c_2 = \frac{1}{1+c_1} \log\paren{ (1+c_1)(K-1) } + \frac{c_1}{1+c_1} \log\paren{ \frac{1+c_1}{c_1} } $. The inequality becomes an equality when
\begin{align}\label{eqn:lce-lower-bound-equality}
    z_i \;=\; z_j,\quad \forall i,j \not = k, \quad \text{and}\quad c_1 \;=\; \brac{(K-1) \exp \paren{  \frac{ \paren{\sum_{i=1}^K z_i} - K z_k}{K-1} }  }^{-1}.
\end{align}
\end{lemma}

\begin{proof}[Proof of Lemma \ref{lem:lce-lower-bound}]
Notice that for any vector $\vz\in\R^{K}$, the cross-entropy loss with $\vy_k$ can be lower bounded by
\begin{align}
    \Lce (\vz,\vy_k) \;=\; \log \paren{  \frac{ \sum_{i=1}^K \exp(z_i) }{ \exp(z_k)  } } \;&=\; \log \paren{ 1+ \sum_{i \not= k}\exp \paren{ z_i - z_k } } \nonumber  \\
    \;&\geq\; \log \paren{  1+ (K-1) \exp \paren{ \sum_{i \not = k} \frac{z_i - z_k}{ K-1 }  } }  \nonumber \\
    \;&=\; \log \paren{  1+ (K-1) \exp \paren{ \frac{ \sum_{i=1}^K z_i - K z_k }{K-1} } } \label{eqn:lce-jensen-lower}
\end{align}
where the inequality follows from the Jensen's inequality that
\begin{align*}
    \sum_{i \not= k}\exp \paren{ z_i - z_k } \;=\; (K-1) \sum_{i \not= k} \frac{1}{K-1} \exp \paren{ z_i - z_k } \;\geq\; (K-1)   \exp \paren{\sum_{i \not= k} \frac{z_i - z_k}{K-1}  },
\end{align*}
which achieves the equality only when $z_i = z_j$ for all $i,j \not= k$. Second, by the concavity of the $\log(\cdot)$ function (i.e., $\log \paren{ tx+(1-t)x'} \geq t \log x + (1-t) \log x' $ for any $x,x' \in \bb R$ and $t \in [0,1]$, which becomes an equality iff $x = x'$, or $t=0 $, or $t=1$), from \eqref{eqn:lce-jensen-lower}, for any $c_1>0$ we have 
\begin{align*}
     &\Lce (\vz,\vy_k) \\
     \;=\;& \log \paren{  \frac{c_1}{1+c_1} \frac{1+c_1}{c_1}  + \frac{1}{1+c_1} (1+c_1)  (K-1) \exp \paren{ \frac{ \sum_{i=1}^K z_i - K z_k }{K-1} } } \\
     \;\geq \;&  \frac{1}{1+c_1}  \log \paren{  (1+c_1)(K-1) \exp \paren{ \frac{ \sum_{i=1}^K z_i - K z_k }{K-1} } } + \frac{c_1}{1+c_1} \log\paren{ \frac{1+c_1}{c_1} }    \\
     \;=\;& \frac{1}{1+c_1}\frac{ \paren{\sum_{i=1}^K z_i} - K z_k }{K-1} + \underbrace{  \frac{1}{1+c_1} \log\paren{ (1+c_1)(K-1) } + \frac{c_1}{1+c_1} \log\paren{ \frac{1+c_1}{c_1} } }_{c_2},
\end{align*}
as desired. The last inequality becomes an equality if any only if
\begin{align*}
    \frac{1+c_1}{c_1} \;=\;   (1+c_1)  (K-1) \exp \paren{ \frac{ \sum_{i=1}^K z_i - K z_k }{K-1} } \quad \text{or} \quad  c_1\;=\; 0, \quad \text{or} \quad  c_1\;=\; +\infty . 
\end{align*}
However, when $c_1=0$ or $ c_1= +\infty $, the equality is trivial. Therefore, we have
\begin{align*}
    c_1 \;=\; \brac{(K-1) \exp \paren{  \frac{ \paren{\sum_{i=1}^K z_i} - K z_k}{K-1} }  }^{-1},
\end{align*}
as desired.
\end{proof}

\section{Proof of \Cref{thm:global-geometry}}
\label{sec:appendix-prf-global-geometry}

In this part of appendices, we prove \Cref{thm:global-geometry} in \Cref{sec:main}. In particular, we analyze the global optimization landscape of 
\begin{align*}
     \min_{\mb W , \mb H,\mb b  } \;f(\mb W,\mb H,\mb b) \;=\;  \frac{1}{Kn} \sum_{k=1}^K \sum_{i=1}^{n} \Lce \paren{ \mb W \mb h_{k,i} + \mb b, \mb y_k } \;+\; \frac{\lambda_{\mb W} }{2} \norm{\mb W}{F}^2 + \frac{\lambda_{\mb H} }{2} \norm{\mb H}{F}^2 + \frac{\lambda_{\mb b} }{2} \norm{\mb b}{2}^2 ,
\end{align*}
with respect to $\mb W \in \bb R^{K \times d}$, $\mb H = \begin{bmatrix}\vh_{1,1} \cdots \vh_{K,n} \end{bmatrix} \in \bb R^{d \times N} $, and $\mb b \in \bb R^K$.
We show that the function is a strict saddle function \cite{ge2015escaping,sun2015nonconvex,zhang2020symmetry} in the Euclidean space, that there is no spurious local minimizer and all global minima are corresponding to the form showed in \Cref{thm:global-minima-app}.

\begin{theorem}[No Spurious Local Minima and Strict Saddle Property]\label{thm:global-geometry-app}
	Assume that the feature dimension $d$ is larger than the number of classes $K$, i.e., $d> K$. Then the function $f(\mb W,\mb H,\mb b)$ in \eqref{eq:obj-app} is a strict saddle function with no spurious local minimum:
	\begin{itemize}
	    \item Any local minimizer of \eqref{eq:obj-app} is a global minimum of the form shown in \Cref{thm:global-minima-app}.
	    \item Any critical point of \eqref{eq:obj-app} that is not a local minimum has at least one negative curvature direction, i.e., the Hessian $\nabla^2 f(\mb W,\mb H,\mb b)$ at this point has at least one negative eigenvalue 
	    \begin{align*}
	        \lambda_i\paren{\nabla^2 f(\mb W,\mb H,\mb b)}\;<\;0.
	    \end{align*}
	\end{itemize}
\end{theorem}

\subsection{Main Proof}

\begin{proof}[Proof of \Cref{thm:global-geometry-app}]
The main idea of proving \Cref{thm:global-geometry} is to first connect the problem \eqref{eq:obj-app} to its corresponding convex counterpart,
so that we can obtain the global optimality conditions for the original problem \eqref{eq:obj-app} based on the convex counterpart. Finally, we characterize the properties of all the critical points based on the optimality condition. We describe this in more detail as follows.

\paragraph{Connection of \eqref{eq:obj-app} to a Convex Program.} Let $\mb Z = \mb H \mb W \in \bb R^{K \times N}$ with $N=nK$ and $\alpha = \frac{\lambda_{\mH}}{\lambda_{\mW}}$. By Lemma \ref{lem:nuclear-norm}, we know that
\begin{align*}
   \min_{\mb H\mb W = \mb Z} \; \frac{ \lambda_{\mb W} }{2} \norm{\mb W}{F}^2 + \frac{ \lambda_{\mb H} }{2} \norm{\mb H}{F}^2 \;&=\; \sqrt{\lambda_{\mb W} \lambda_{\mb H} }  \min_{\mb H\mb W = \mb Z}  \;\frac{1}{2\sqrt{\alpha} } \paren{ \norm{\mb W}{F}^2 + \alpha \norm{\mb H}{F}^2 } \\
   \;&=\; \sqrt{\lambda_{\mb W} \lambda_{\mb H} }  \norm{\mb Z}{*}.
\end{align*}
Thus, we can relate the bilinear factorized problem \eqref{eq:obj-app} to a convex problem in terms of $\mb Z$ and $\mb b$ as follows
\begin{align}\label{eq:convex-prob}
    \min_{\mb Z \in \bb R^{K \times N},\; \mb b \in \bb R^{K}}\; \wt{f}(\mb Z,\mb b) \;:=\; g\paren{\mb Z + \mb b \mb 1^\top} +  \sqrt{\lambda_{\mb W} \lambda_{\mb H} }  \norm{\mb Z}{*} + \frac{\lambda_{\mb b}}{2} \norm{\mb b}{2}^2.
\end{align}
Similar to the idea in \cite{haeffele2015global,ciliberto2017reexamining,zhu2018global,li2019non}, we will characeterize the critical points of \eqref{eq:obj-app} by establishing a connection to the optimality condition of the convex problem \eqref{eq:convex-prob}. Towards this goal, we first show the global minimum of the convex program \eqref{eq:convex-prob} provides a lower bound for the original problem \eqref{eq:obj}. More specifically, in Lemma \ref{lem:equivalence} we can show that for any global minimizer $(\mb Z_\star, \mb b_\star)$ of $\wt{f}$ satisfies 
\begin{align}\label{eqn:convex-lower-bound}
    \wt{f}(\mZ^\star,\vb^\star) \;\leq\; f(\mW,\mH,\vb),\quad \forall \; \mW\in\R^{K\times d},\;\mH\in\R^{d\times N},\;\vb\in\R^K.
\end{align}

\paragraph{Characterizing the Optimality Condition of \eqref{eq:obj-app} Based on the Convex Program \eqref{eq:convex-prob}.}
Next, we characterize the optimality condition of our nonconvex problem \eqref{eq:obj-app}, based on the relationship to its convex counterpart
\eqref{eq:convex-prob}. Specifically, Lemma \ref{lem:conv-global-cond} showed that any critical point $(\mZ,\vb)$ of \eqref{eq:convex-prob} is characterized by the following necessary and sufficient condition
\e\begin{split}
    &\nabla g(\mZ + \vb\vone^\top) \;\in\; - \sqrt{\lambdaW\lambdaH}\partial\norm{ \mZ}{*},\\
    &\sum_{i=1}^N [\nabla g(\mZ + \vb\vone^\top)]_i + \lambdab\vb\; =\; \vzero,
\end{split}
\label{eq:KKT-convex-app}\ee
where $[\nabla g(\mZ + \vb\vone^\top)]_i$ represents the $i$-th column in $\nabla g(\mZ + \vb\vone^\top)$. By Lemma \ref{lem:global-optimality-nonconvex}, we can transfer the optimality condition from convex to the nonconvex problem \eqref{eq:obj-app}. More specifically, any critical point $(\mW,\mH,\vb)$ of \eqref{eq:obj-app} satisfies
\begin{align*} 
    \norm{ \nabla g(\mb W \mb H + \mb b \mb 1^\top) }{} \;\leq \; \sqrt{ \lambda_{\mb W} \lambda_{\mb H} },
\end{align*}
then $(\mb Z,\mb b)$ with $\mb Z = \mb W\mb H$ satisfies all the conditions in \eqref{eq:KKT-convex-app}. Combining with \eqref{eqn:convex-lower-bound}, Lemma \ref{lem:global-optimality-nonconvex} showed that $(\mW,\mH,\vb)$ is a global solution of the nonconvex problem \eqref{eq:obj-app}.

\paragraph{Proving No Spurious Local Minima and Strict Saddle Property.} Now we turn to prove the strict saddle property and that there are no spurious local minima by characterizing the properties for all the critical points of \eqref{eq:obj-app}. Denote the set of critical points of $f(\mW,\mH,\vb)$ by
\begin{align*}
    \calC \;:=\; \Brac{ \; \mW\in\R^{K\times d},\mH\in\R^{d\times N}, \vb\in\R^K \;\mid\;  \nabla f(\mW,\mH,\vb) \;=\; \vzero\; }.
\end{align*}
To proceed, we separate the set $\calC$ into two disjoint subsets
\begin{align*}
    \calC_1 \;&:=\;  \calC\cap \Brac{\; \mW\in\R^{K\times d},\mH\in\R^{d\times N}, \vb\in\R^K\;\mid\; \norm{\nabla g(\mW\mH + \vb\vone^\top)}{} \;\leq\; \sqrt{\lambdaW\lambdaH} \;} ,\\
       \calC_2 \;&:=\; \calC\cap \Brac{\; \mW\in\R^{K\times d},\mH\in\R^{d\times N}, \vb\in\R^K\;\mid\; \norm{\nabla g(\mW\mH + \vb\vone^\top)}{}  \;>\; \sqrt{\lambdaW\lambdaH} \;},
\end{align*}
satisfying $\calC = \calC_1\cup\calC_2$. By Lemma \ref{lem:global-optimality-nonconvex}, we already know that any $(\mW,\mH,\vb)\in \calC_1$ is a global optimal solution of $f(\mb W,\mb H,\mb b)$ in \eqref{eq:obj-app}. If we can show that any critical point in $\calC_2$ possesses negative curvatures (i.e., the Hessian at $(\mb W,\mb H,\mb b)$ has at least one negative eigenvalue), then we prove that there is no spurious local minima as well as strict saddle property.

%If we show that $\calC_2$ is the set of strict saddles, then we prove that there is no spurious local minima as well as strict saddle property. By Lemma \ref{lem:global-optimality-nonconvex}, we conclude that $(\mW,\mH,\vb)$ is the globally optimal solution of \eqref{eq:obj} for any $(\mW,\mH,\vb)\in \calC_1$. If we show that $\calC_2$ is the set of strict saddles, then we prove that there is no spurious local minima as well as strict saddle property. 
Thus, the remaining part is to show any point in $\calC_2$ possesses negative curvatures. We will find a direction $\mDelta$ along which the Hessian has a strictly negative curvature for this point. Towards that goal, for any $\mDelta = (\mDelta_{\mW},\mDelta_{\mH},\mDelta_{\vb})$, we first compute the Hessian bilinear form of $f(\mb W,\mb H,\mb b)$ along the direction $\mDelta$ by
\begin{align}\label{eqn:hessian-f}
\begin{split}
        &\nabla^2 f(\mW,\mH,\vb)[\mDelta,\mDelta] \\
    \;=\;& \nabla^2g(\mW\mH + \vb\vone^\top) \brac{  \mW\mDelta_{\mH} + \mDelta_{\mW}\mH + \mDelta_{\vb}\vone^\top,\mW\mDelta_{\mH} + \mDelta_{\mW}\mH + \mDelta_{\vb}\vone^\top } \\
    &+ 2\innerprod{\nabla g(\mW\mH+ \vb\vone^\top)}{\mDelta_{\mW}\mDelta_{\mH}} +\lambdaW\norm{\mDelta_{\mW}}{F}^2 + \lambdaH \norm{\mDelta_{\mH}}{F}^2 + \lambdab\norm{\mDelta_{\vb}}{2}^2.
\end{split}
\end{align}
We now utilize the property that $\norm{ \nabla g(\mW\mH + \vb\vone^\top)}{} > \sqrt{\lambdaW\lambdaH}$ for any $(\mW,\mH,\vb)\in \calC_2$ to construct a negative curvature direction. Let $\mb u$ and $\mb v$ be the left and right singular vectors corresponding to the largest singular value $\sigma_1(\nabla^2 g(\mW\mH + \vb\vone^\top))$ of $\nabla^2 g(\mW\mH + \vb\vone^\top)$, which is larger than $\sqrt{\lambdaW\lambdaH}$ by our assumption.

Since $d>K$ and $\mW \in \bb R^{K \times d}$, there exists a nonzero $\va\in\R^{d}$ in the null space of $\mW$, i.e., $\mW\va = \vzero$.
%$\mathrm{row}(\mW)= \mathrm{span}\Brac{\mb w_1, \cdots, \mb w_K}$, i.e., $\mW\va = \vzero$. 
%\zz{I think it is more common to simply say the null space of the matrix $\mW$. Not common to say the null space of a subspace.} 
Furthermore, by Lemma  \ref{lem:critical-balance}, we know that 
\begin{align*}
    \mW^\top\mW = \sqrt{ \frac{ \lambda_{\mb H} }{ \lambda_{\mb W} }  } \mH\mH^\top \quad \Longrightarrow\quad \mH^\top\va = \vzero.
\end{align*}
We now construct the negative curvature direction as 
 \begin{align*}
   \mb \Delta \;=\;  \paren{ \mb \Delta_{\mb W},\mDelta_{\mH}, \mDelta_{\vb}  } \;=\; \paren{  \paren{ \frac{ \lambda_{\mb H} }{ \lambda_{\mb W} } }^{1/4} \mb u \mb a^\top ,  - \paren{ \frac{ \lambda_{\mb H} }{ \lambda_{\mb W} } }^{-1/4} \mb a \mb v^\top, \mb 0 }
 \end{align*}
so that the term $\innerprod{\nabla g(\mW\mH+ \vb\vone^\top)}{\mDelta_{\mW}\mDelta_{\mH}}$ is small enough to create a negative curvature. Since $\mW\va = \vzero$, $\va^\top \mH = \vzero$, and $\mDelta_{\vb} = \mb 0$, we have
\begin{align*}
    \mW\mDelta_{\mH} + \mDelta_{\mW}\mH + \mDelta_{\vb}\vone^\top \;=\; - \paren{ \frac{ \lambda_{\mb H} }{ \lambda_{\mb W} } }^{-1/4}  \mW  \va \mb v^\top +  \paren{ \frac{ \lambda_{\mb H} }{ \lambda_{\mb W} } }^{1/4} \mb u \mb a^\top \mH \;=\; \vzero,
\end{align*}
so that $\nabla^2g(\mW\mH + \vb\vone^\top) \brac{  \mW\mDelta_{\mH} + \mDelta_{\mW}\mH + \mDelta_{\vb}\vone^\top,\mW\mDelta_{\mH} + \mDelta_{\mW}\mH + \mDelta_{\vb}\vone^\top } = 0$. Thus, combining the results above with \eqref{eqn:hessian-f}, we obtain the following
\begin{align*}
&\nabla f(\mW,\mH,\vb)[\mDelta,\mDelta] \\
\;=\;& 2\innerprod{\nabla g(\mW\mH + \vb\vone^\top)}{\mDelta_{\mW}\mDelta_{\mH}} \;+\; \lambdaW\norm{\mDelta_{\mW}}{F}^2 + \lambdaH\norm{\mDelta_{\mH}}{F}^2 \\
\;=\;& -2\norm{\va}{2}^2\innerprod{\nabla g(\mW\mH + \vb\vone^\top)}{\mb u \mb v^\top} + 2 \sqrt{ \lambda_{\mb W} \lambda_{\mb H} }  \norm{\va}{2}^2 \\
\;=\;& -2\norm{\va}{2}^2 \paren{\norm{\nabla g(\mW\mH+ \vb\vone^\top)}{} - \sqrt{\lambdaW\lambdaH}} \;<\;0,
\end{align*}
where the last inequality is based on the fact that $(\mb W,\mb H,\mb b) \in \mc C_2 $ so that $\norm{ \nabla g(\mW\mH + \vb\vone^\top) }{} > \sqrt{\lambdaW\lambdaH}$. Therefore, any $(\mW,\mH,\vb)\in \calC_2$ possesses at least one negative curvature direction. This completes our proof of \Cref{thm:global-geometry-app}.
\end{proof}

\subsection{Supporting Lemmas}

\begin{lemma} \label{lem:equivalence}
If $(\mZ^\star,\vb^\star)$ is a global minimizer of
\begin{align*} 
    \min_{\mb Z \in \bb R^{K \times N},\; \mb b \in \bb R^{K}}\; \wt{f}(\mb Z,\mb b) \;:=\; g\paren{\mb Z + \mb b \mb 1^\top} +  \sqrt{\lambda_{\mb W} \lambda_{\mb H} }  \norm{\mb Z}{*} + \frac{\lambda_{\mb b}}{2} \norm{\mb b}{2}^2.
\end{align*}
introduced in \eqref{eq:convex-prob}, then $\wt{f}(\mZ^\star,\vb^\star) \le f(\mW,\mH,\vb)$ for all $\mW\in\R^{K\times d},\mH\in\R^{d\times N}, \vb\in\R^K$.
\end{lemma}

\begin{proof}[Proof of Lemma \ref{lem:equivalence}] Suppose $(\mZ^\star,\vb^\star)$ is a global minimum of $\wt{f}(\mb Z,\mb b)$. Then by \Cref{lem:nuclear-norm}, we have %\xiao{mark}
\begin{align*}
    \wt{f}(\mZ^\star,\vb^\star) \;&= \; \min_{\mZ,\vb} \; g(\mZ + \vb\vone^\top) + \sqrt{\lambdaW\lambdaH} \norm{\mZ}{*} + \frac{\lambda_{\mb b} }{2}\norm{\vb}{2}^2\\
    \;&= \; \min_{\mZ,\vb} \; g(\mb Z + \vb\vone^\top)+ \min_{\mW\mH=\mZ}  \paren{ \frac{\lambdaW}{2}\norm{\mW}{F}^2 + \frac{\lambdaH}{2}\norm{\mH}{F}^2 } + \frac{\lambdab}{2}\norm{\vb}{2}^2 \\
    \;&\leq\;\min_{\mb W,\mb H, \mb Z,\vb, \mb Z = \mb W\mb H} \; g(\mb W \mb H + \vb\vone^\top)+  \frac{\lambdaW}{2}\norm{\mW}{F}^2 + \frac{\lambdaH}{2}\norm{\mH}{F}^2  + \frac{\lambdab}{2}\norm{\vb}{2}^2.
\end{align*}
Thus, we have 
\begin{align*}
    \wt{f}(\mZ^\star,\vb^\star)  \;\leq\; \min_{\mW\in\R^{K\times d},\mH\in\R^{d\times N}, \vb\in\R^K} \; f(\mW,\mH,\vb)
\end{align*}
as desired.
\end{proof}

\begin{lemma}[Optimality Condition for the Convex Program \eqref{eq:convex-prob}]\label{lem:conv-global-cond}
Consider the following convex program in \eqref{eq:convex-prob} that we rewrite as follows
\begin{align*} 
    \min_{\mb Z \in \bb R^{K \times N},\; \mb b \in \bb R^{K}}\; \wt{f}(\mb Z,\mb b) \;:=\; g\paren{\mb Z + \mb b \mb 1^\top} +  \sqrt{\lambda_{\mb W} \lambda_{\mb H} }  \norm{\mb Z}{*} + \frac{\lambda_{\mb b}}{2} \norm{\mb b}{2}^2.
\end{align*}
Then the necessary and sufficient condition for $(\mb Z,\mb b)$ being the global solution of \eqref{eq:convex-prob} is 
\e\begin{split}
    &\nabla g(\mZ + \vb\vone^\top)\mV \;=\; - \sqrt{\lambdaW\lambdaH} \mU,\quad \nabla g(\mZ + \vb\vone^\top)^\top\mU \;=\; - \sqrt{\lambdaW\lambdaH} \mV,\\
     & \norm{\nabla g(\mZ + \vb\vone^\top)}{} \;\leq\;  \sqrt{\lambdaW\lambdaH}, \quad \text{and}\quad \sum_{i=1}^N [\nabla g(\mZ + \vb\vone^\top)]_i + \lambdab\vb\; =\; \vzero,
\end{split}
\label{eq:KKT-convex-equivalence}\ee
where $\mb U$ and $\mb V$ are the left and right singular value matrices of $\mb Z$, i.e., $\mb Z = \mb U \mb \Sigma \mb V^\top$.
\end{lemma}

\begin{proof}[Proof of Lemma \ref{lem:conv-global-cond}] 
Standard convex optimization theory asserts that any critical point $(\mZ,\vb)$ of \eqref{eq:convex-prob} is global, where the optimality is characterized by the following necessary and sufficient condition
\e\begin{split}
    &\nabla g(\mZ + \vb\vone^\top) \;\in\; - \sqrt{\lambdaW\lambdaH}\partial\norm{ \mZ}{*},\\
    &\sum_{i=1}^N [\nabla g(\mZ + \vb\vone^\top)]_i + \lambdab\vb\; =\; \vzero,
\end{split}
\label{eq:KKT-convex}\ee
where $[\nabla g(\mZ + \vb\vone^\top)]_i$ represents the $i$-th column in $\nabla g(\mZ + \vb\vone^\top)$, and $\partial\norm{\mZ}{*}$ denotes the subdifferential of the convex nuclear norm $\norm{\mZ}{*}$ evaluated at $\mZ$. By its definition, we have
\begin{align*}
    \partial \norm{\mb Z}{*} \;:=\; \Brac{ \mb D \in \bb R^{K \times N} \;\mid\; \norm{\mb G}{*} \geq \norm{\mb Z}{*} + \innerprod{ \mb G - \mb Z }{ \mb D },\; \mb G \in \bb R^{K \times N}  },
\end{align*}
where for nuclear norm, the previous work \cite{watson1992characterization,recht2010guaranteed} showed that based on the projection onto the column space and row space via the SVD of $\mb Z = \mb U \mb \Sigma \mb V^\top$, this is equivalent to  
\begin{align*}
   \partial \norm{\mb Z}{*} \;=\; \Brac{ \mb U \mb V^\top + \mb W, \mb W \in \bb R^{K \times N} \;\mid\; \mb U^\top \mb W = \mb 0,\; \mb W \mb V = \mb 0,\; \norm{\mb W}{}\leq 1 },
\end{align*}
where $\mb U$ and $\mb V$ are the left and right singular value matrices of $\mb Z$. Using the result above, we can now rewrite the optimality condition in \eqref{eq:KKT-convex} as suggested in Lemma \ref{lem:conv-global-cond}.
\end{proof}

\begin{lemma}[Optimality Condition for the Nonconvex Program \eqref{eq:obj-app}] \label{lem:global-optimality-nonconvex}
If a critical point $(\mW,\mH,\vb)$ of \eqref{eq:obj-app} satisfies
\begin{align}\label{eqn:global-cond-WH}
    \norm{ \nabla g(\mb W \mb H + \mb b \mb 1^\top) }{} \;\leq \; \sqrt{ \lambda_{\mb W} \lambda_{\mb H} },
\end{align}
then it is a global minimum of \eqref{eq:obj-app}.	

\end{lemma}

\begin{proof}[Proof of Lemma \ref{lem:global-optimality-nonconvex}] 
Suppose $(\mW^\star,\mH^\star,\vb^\star)$ is a critical point of \eqref{eq:obj-app} satisfying \eqref{eqn:global-cond-WH}, we will show that $(\mb Z^\star,\vb^\star)$ with $\mb Z^\star=\mW^\star\mH^\star$ is a global minimizer of \eqref{eq:convex-prob} by showing that $(\mb W^\star\mb H^\star,\vb^\star)$ satisfies the optimality condition for the convex program in \eqref{eq:KKT-convex-equivalence} (Lemma \ref{lem:conv-global-cond}). First of all, it is easy to check that
\begin{align*}
    \nabla_{\vb}f(\mW^\star,\mH^\star,\vb^\star) \;&=\;  \sum_{i=1}^N [\nabla g(\mW^\star\mH^\star + \vb^\star\vone^\top)]_i + \lambdab\vb^\star \;=\; \mb 0 \\
      \;&\Longrightarrow \; \sum_{i=1}^N [\nabla g(\mb Z^\star + \vb^\star\vone^\top)]_i + \lambdab\vb^\star\;=\; \vzero.
\end{align*}
Second, let $\mb Z^\star = \mW^\star\mH^\star = \mU\mSigma\mV^\top$ be the compact SVD of $\mb Z^\star= \mW^\star\mH^\star$. By Lemma \ref{lem:critical-balance}, we have
\begin{align*}
\mW^{\star\top}\mW^\star \;=\; \frac{\lambda_{\mH}}{\lambda_{\mW}}\mH^\star\mH^{\star\top}\;  \Longrightarrow\;  \mH^{\star\top} \mH^\star \mH^{\star\top} \mH^\star \;=\; \frac{\lambdaW}{\lambdaH} \mH^{\star\top}\mW^{\star\top}\mW^\star\mH^\star \;=\; \frac{\lambdaW}{\lambdaH} \mV\mSigma^2\mV^\top.
\end{align*}	
Now by utilizing Lemma \ref{lem:matrix-analysis-horn}, from the above equation we obtain the following
\begin{align*}
    \mb H^{\star\top} \mb H^\star \;=\; \sqrt{\frac{\lambdaW}{\lambdaH}} \mV\mSigma\mV^\top.
\end{align*}
This together with the first equation in \eqref{eqn:W-crtical} gives 
\begin{align*}
    \nabla g(\mW^\star\mH^\star + \vb^\star\vone^\top)\mH^{\star\top} \;=\; - \lambdaW \mW^\star\;\; &\Longrightarrow\;\; \nabla g(\mW^\star\mH^\star + \vb^\star\vone^\top)\mH^{\star\top}\mH^\star = - \lambdaW \mW^\star\mH^\star\\
    &\Longrightarrow \;\;\nabla g(\mW^\star\mH^\star + \vb\vone^\top)\sqrt{\frac{\lambdaW}{\lambdaH}} \mV\mSigma\mV^\top = - \lambdaW \mU\mSigma\mV^\top\\
    &\Longrightarrow \;\;  \nabla g(\mZ^\star + \vb^\star\vone^\top)\mV = - \sqrt{\lambdaW\lambdaH} \mU.
\end{align*}
Similarly, we can also get 
\begin{align*}
    \nabla g(\mZ^\star + \vb^\star\vone^\top)^\top\mU \;=\; - \sqrt{\lambdaW\lambdaH} \mV.
\end{align*}
 Thus, together with \eqref{eqn:global-cond-WH}, $(\mb Z^\star,\vb^\star)$ with $\mb Z^\star=\mb W^\star\mb H^\star$ satisfies the optimality condition \eqref{eq:KKT-convex-equivalence}, and hence is a global minimizer of $\wt{f}(\mb Z^\star,\mb b^\star)$ in \eqref{eq:convex-prob}. 
 
Finally, we complete the proof by invoking Lemma \ref{lem:equivalence}. By Lemma \ref{lem:equivalence} with $\mb Z^\star= \mb W^\star\mb H^\star$, we know that $f(\mb W^\star,\mb H^\star,\mb b^\star) = \wt{f}(\mb Z^\star,\mb b^\star) \leq f(\mb W,\mb H,\mb b)$ for all $\mW\in\R^{K\times d},\mH\in\R^{d\times N}, \vb\in\R^K$. Therefore, we must have $(\mb W^\star,\mb H^\star,\mb b^\star)$ to be the global solution of $f(\mb W,\mb H,\mb b)$ in \eqref{eq:obj-app}.
\end{proof}

%\eqref{eq:KKT-convex} and 

%By Lemma \ref{lem:equivalence}, we showed that $(\mW,\mH,\vb)$ is a global minimizer of \eqref{eq:obj-app} if $(\mb Z,\vb)$ with $\mb Z = \mb W \mb H$ is a global minimizer of the convex program \eqref{eq:convex-prob}.

%First note that any critical point $(\mW,\mH,\vb)$ of \eqref{eq:obj-app} satisfies \eqref{eq:critical-balance} (from Lemma \ref{lem:critical-balance}) and
%By assumption \eqref{eqn:global-cond-WH}, . To prove this,

\end{document}